\relax
\documentclass[letterpaper]{article} 
\usepackage{arXiv_aaai_modified}  
\usepackage{times}  
\usepackage{helvet} 
\usepackage{courier}  
\usepackage[hyphens]{url}  
\usepackage{graphicx} 
\usepackage{graphicx}  
\usepackage{graphicx} 
\usepackage[caption=false]{subfig}

\usepackage{algorithm}
\usepackage{algorithmic}
\usepackage{nccmath}
\usepackage{booktabs}

\usepackage{url}

\usepackage{bm}

\newcommand{\req}[1]{eq.~(\ref{#1})}
\newcommand{\rfig}[1]{Fig.~\ref{#1}}
\newcommand{\rtab}[1]{Tab.~\ref{#1}}
\usepackage{amsmath,cases}
\usepackage{amsthm}

\newtheorem{lemma}{Lemma}

\newcommand{\citet}[1]{\citeauthor{#1} \shortcite{#1}}

\title{Absum: Simple Regularization Method for Reducing Structural Sensitivity of Convolutional Neural Networks}
\author{ \Large \textbf{Sekitoshi Kanai\textsuperscript{\rm 1}\textsuperscript{\rm 2}, Yasutoshi Ida\textsuperscript{\rm 1}, Yasuhiro Fujiwara\textsuperscript{\rm 3}, Masanori Yamada\textsuperscript{\rm 4}, Shuichi Adachi\textsuperscript{\rm 2}}\\ 
\textsuperscript{\rm 1}NTT Software Innovation Center
\textsuperscript{\rm 2}Keio University
\textsuperscript{\rm 3}NTT Communication Science Laboratories
\textsuperscript{\rm 4}NTT Secure Platform Laboratories\\ 
\{sekitoshi.kanai.fu, yasutoshi.ida.yc, yasuhiro.fujiwara.kh, masanori.yamada.cm\}@hco.ntt.co.jp,
adachi@appi.keio.ac.jp
}
 
\begin{document}

\maketitle

\begin{abstract}
We propose Absum, which is a regularization method for improving
adversarial robustness of convolutional neural networks (CNNs).
Although CNNs can accurately recognize images, recent studies
have shown that the convolution operations in CNNs
commonly have structural sensitivity to specific noise composed of Fourier basis functions.
By exploiting this sensitivity, they proposed a simple black-box adversarial attack: Single Fourier attack.
To reduce structural sensitivity, we can use regularization of convolution filter weights
since the sensitivity of linear transform can be assessed by the norm of the weights. 
However, standard regularization methods can prevent minimization of the loss function
because they impose a tight constraint for obtaining high robustness. 
To solve this problem,
Absum imposes a loose constraint; it penalizes the absolute values of
the summation of the parameters in the convolution layers. 
Absum can improve robustness against single Fourier attack
while being as simple and efficient as
standard regularization methods (e.g., weight decay and $L_1$ regularization).
Our experiments demonstrate that Absum improves robustness against single Fourier attack more than
standard regularization methods.
Furthermore, we reveal that robust CNNs with Absum are more robust against transferred
attacks due to decreasing the common  sensitivity and against high-frequency noise
than standard regularization methods.
We also reveal that Absum can improve robustness
against gradient-based attacks (projected gradient descent) when used with adversarial training.
\end{abstract}
\frenchspacing
\section{Introduction}
Deep neural networks have achieved great success in many applications, e.g., image recognition \cite{resnet} and 
machine translation \cite{vaswani2017attention}. 
Specifically, CNNs and rectified linear units (ReLUs) have resulted in breakthroughs in image recognition \cite{lecun1989backpropagation,relu} and
 are de facto standards for image recognition and other applications \cite{resnet,dcgan}.
Though CNNs can classify image data as accurately as humans, 
they are sensitive to small perturbations of inputs, i.e.,
injecting imperceptible perturbations can make deep models misclassify image data.
 Such attacks are called adversarial attacks and the perturbed inputs are called adversarial examples \cite{szegedy2013intriguing}.

We can roughly divide adversarial attacks into two types; white-box attacks, which use the information of target models \cite{fgsm,pgd2,deepfool}, 
and black-box attacks, which do not require the information of target models
\cite{papernot2016transferability,chen2017zoo,papernot2017practical}. 
Black-box attacks, rather than white-box attacks, can threaten online deep-learning services
since it is difficult to access the target models in online deep-learning applications 
\cite{papernot2017practical,yuan2019adversarial}. 

Most black-box attacks are transferred attacks, 
which are generated as white-box attacks for substitute models
instead of the target model \cite{papernot2016transferability}.
This implies that deep models have common sensitivity against specific perturbations.
In fact, \citet{tsuzuku2018structural} have recently
shown that CNNs have the structural sensitivity from the perspective that convolution
can be regarded as the 
product of the circulant matrix and proposed single Fourier attack (SFA).\footnote{\citet{Fp} concurrently proposed the same attack.}
Fourier basis functions create singular vectors of circulant matrices,
 and SFA uses these singular vectors since the dominant singular vector can be the worst noise
for a matrix-vector product.
Although SFA is a very simple attack composed of a single-frequency component,
it is universal adversarial perturbations for CNNs, i.e.,
it can decrease the classification accuracy of
various CNN-based models without using the information about the model parameters
and without depending on input images. 
To the best of our knowledge, an effective defense method against SFA has not been proposed.
Therefore, such a method is necessary.

To defend CNNs against SFA, we first reveal that the spectral norm constraint \cite{sedghi2018the}
(hereinafter, we call it SNC) can reduce the structural sensitivity.
While SNC was proposed to improve generalization performance,
it can improve robustness in the Fourier domain since singular values of convolution
layers correspond to
the magnitude of the frequency response.
However, SNC is not so practical since it requires high computational cost
to compute the spectral norm (the largest singular value).
We then develop \textit{Absum}; an efficient regularization method for reducing the structural
 sensitivity of CNNs. 
Instead of the spectral norm, 
we use the induced $\infty$-norm ($L_\infty$ operator norm) since it is  the
upper bound of the spectral norm for convolution.
However, 
a constraint of the induced $\infty$-norm, which is equivalent to
$L_1$ regularization, requires a tight constraint for robustness,
which prevents minimization of the loss function.
This is because the induced $\infty$-norm is a conservative 
measure; it handles the effects of negative inputs even though inputs always 
have positive values after ReLU activations.
To improve robustness without preventing the loss minimization, 
Absum relaxes the induced $\infty$-norm by penalizing
the absolute values of the
summations of weights instead of elements on the basis that input vectors always have positive elements.
Absum is as simple as standard regularization methods such as weight decay,
but it can reduce sensitivity to SFA.
We provide the proximal operator to minimize loss functions with Absum.

Image recognition experiments on MNIST, Fashion-MNIST (FMNIST), CIFAR10, CIFAR100, and SVHN demonstrate
 that Absum and SNC outperform $L_1$ and $L_2$ regularization methods
in terms of improving robustness against SFA, and 
the computation time of Absum is about one-tenth that of SNC. 
In the additional empirical evaluation,
we reveal that 
robust CNNs against SFA can be robust against transferred attacks
by using white-box attacks (projected gradient descent: PGD \cite{pgd,pgd2}). This implies that
sensitivity to SFA is one of the causes of the transferability of adversarial attacks. 
As a further investigation of Absum and SNC, we reveal that adversarial perturbations 
for CNNs trained with Absum and SNC have little high-frequency components, i.e., these CNNs
 are robust against high-frequency noise. 
Furthermore, our experiments show that Absum is effective against
PGD when using adversarial training.

The following are main contributions of this paper:
\begin{itemize}
\item We show that SNC improves robustness against SFA. SNC was
proposed to improve generalization performance, but effectiveness in robustness against SFA
had not been evaluated.
\item We propose Absum and its proximal operator.
Absum improves robustness against SFA as well as SNC while its computational cost
 is lower than that of SNC. 
\item In the futher empirical evaluation, Absum and SNC can also
improve robustness against other black-box attacks
(transferred attacks and High-Frequency attacks \cite{highF}). 
In addition, Absum can improve robustness against PGD when used with adversarial training.
\end{itemize}
\section{Preliminaries}
\label{pre}
\subsection{CNNs, ReLUs and Circulant Matrix}
\label{convrelu}
In this section, we outline CNNs, ReLUs, and a circulant matrix
for convolution operation.
Let $\bm{X}\!\in\!\bm{R}^{n\times n}$ be an input map, 
$\bm{Y}\!\in\!\bm{R}^{n\times n}$ be an output map,
 and $\bm{K}\!\in\!\bm{R}^{n\times n}$ be a filter 
matrix such that
$\bm{K}\!=\![\bm{k}_1, \bm{k}_2, \dots, \bm{k}_{n}]^T$, 
where $\bm{k}_i\!=\![k_{i,1},k_{i,2}\dots,k_{i,n}]^T\!\in\!\bm{R}^{n}$.
The output of the convolution operation $\bm{Y}\!=\!\bm{K}*\bm{X}$ becomes
\begin{align}
\textstyle Y_{l,m}=\sum_{p=1}^n\sum_{q=1}^n k_{p,q}X_{l+p-1,m+q-1}.
\label{conv}
\end{align}
Note that when the filter size is $h\times h$ and $h<n$, 
we can embed it in the $n\times n$ matrix $\bm{K}$ by padding with zeros \cite{sedghi2018the}.
After the convolution, we usually use ReLU activations as the following function:
\begin{align}
\textstyle ReLU(x)=\max (x,0).
\end{align}
Typical model architectures use a combination of convolution and ReLU.
For example, a standard block of ResNet \cite{resnet} is composed as
\begin{align}
\scalebox{0.906}{$
\textstyle h(\bm{X})\!=\!ReLU(\bm{X}\!+\!BN(\bm{K}^{(2)}\!*\!ReLU(BN(\bm{K}^{(1)}\!*\!\bm{X})))),$}
\label{rescomp}
\end{align}
where $BN$ is batch normalization \cite{ioffe2015batch}.

Since SFA and Absum are based on a circulant matrix for convolution operation,
we show that the convolution can be expressed as a product of 
a vector and doubly block circulant matrix. 
Let $\bm{x}\!=\!\mathrm{vec}(\bm{X})$ and $\bm{y}\!=\!\mathrm{vec}(\bm{Y})$
be vectors obtained by stacking the columns of $\bm{X}$ and 
$\bm{Y}$, respectively. Convolution $\bm{K}*\bm{X}$ can be written as
\begin{align}
\textstyle \bm{y}=\bm{A}\bm{x},
\end{align}
where $\bm{A}\in \bm{R}^{n^2\times n^2}$ is the following matrix:
\begin{align}
\scalebox{0.8}{$
\textstyle\!\! \bm{A}$}
\scalebox{0.8}{$
=\!
\left[\!\begin{array}{c@{\hskip5pt}c@{\hskip5pt}c@{\hskip4pt}c}
\mathrm{c}(\bm{k}_1)&\mathrm{c}(\bm{k}_2)&\dots&\mathrm{c}(\bm{k}_{n})\\
\mathrm{c}(\bm{k}_{n})&\mathrm{c}(\bm{k}_1)&\dots&\mathrm{c}(\bm{k}_{n-1})\\
\vdots&&&\vdots\\
\mathrm{c}(\bm{k}_2)&\mathrm{c}(\bm{k}_3)&\dots&\mathrm{c}(\bm{k}_{1})
\end{array}\!\!\right]$}\!,
\scalebox{0.8}{$
\textstyle\mathrm{c}(\bm{k}_i)\!$}=
\!\scalebox{0.8}{$
\!\left[\!\begin{array}{c@{\hskip5pt}c@{\hskip5pt}c@{\hskip4pt}c}
k_{i,1},&k_{i,2},&\dots,&k_{i,n}\\
k_{i,n},&k_{i,1},&\dots,&k_{i,n-1}\\
\vdots&&&\vdots\\
k_{i,2},&k_{i,3},&\dots,&k_{i,1}\!
\end{array}\!\!\right]$}\!.
\label{dcm}
\end{align}
The coefficients $k_{i,j}$ are cyclically shifted in $\mathrm{c}(\bm{k}_i)\!\in\!\bm{R}^{n\times n}$, 
and block matrices $\mathrm{c}(\bm{k}_i)$ are cyclically shifted in $\bm{A}$.
Therefore, $\bm{A}$ is called a doubly block circulant matrix.
\subsection{Single Fourier Attack}
As mentioned above, convolution can be written by a doubly block circulant matrix.
Such matrices always have eigenvectors $\bm{Q}\!=\!\frac{1}{n}\bm{F}\!\otimes\!\bm{F}$,
where elements of $\bm{F}$ are composed of the Fourier basis 
$F_{l,m}\!=\!\mathrm{exp}(\!-j\frac{2\pi}{n}lm)$,
where $j\!=\!\sqrt{-1}$ \cite{Jain1989,sedghi2018the,tsuzuku2018structural}, and
singular vectors are also composed of $\bm{F}\!\otimes\!\bm{F}$
even if we stack convolution layers \cite{tsuzuku2018structural,karner2003spectral}.  
From these characteristics,
\citet{tsuzuku2018structural} proposed SFA.
The perturbed input image $\hat{\bm{X}}$ by SFA is
\begin{align}
\scalebox{0.89}{$
\hat{\bm{X}}\!=\!\bm{X}\!+\!\varepsilon((1\!+\!j)(\bm{F})_l\!\otimes \!(\bm{F})_m\!+\!(1\!-\! j)(\bm{F})_{n-l}\!\otimes\! (\bm{F})_{n-m})$},
\end{align}
where $(\bm{F})_{l}\!\in\!\bm{R}^n$ is the $l$-th column vector of $\bm{F}$,
$\bm{X}$ is an input image, and $\varepsilon$ is magnitude of the attack.
SFA is composed of $(\bm{F})_l\!\otimes\!(\bm{F})_m$ and its complex conjugate
$(\bm{F})_{n-l}\otimes (\bm{F})_{n-m}$ to create a perturbation that has real values
since inputs of CNNs are assumed to be real values.
The $l$ and $m$ are hyperparameters such that $l\!=\!0,1,\dots,n-1, m\!=\!0,1,\dots,n-1$. 
Figure~\ref{sfa} shows examples of CIFAR10 perturbed by SFA.
We can see that $(l,m)$ determines a space-frequency of the noise.
Note that stacked convolution layers without activation functions
(e.g., $\bm{A}^{(2)}\!\bm{A}^{(1)}\bm{x}$)
also have singular vectors composed of Fourier basis functions.
Even though we use nonlinear activation functions,
many model architectures (e.g., WideResNet, DenseNet-BC, and GoogLeNet) are sensitive to
SFA \cite{tsuzuku2018structural}.
\begin{figure}[tb]
\centering
\includegraphics[width=\hsize]{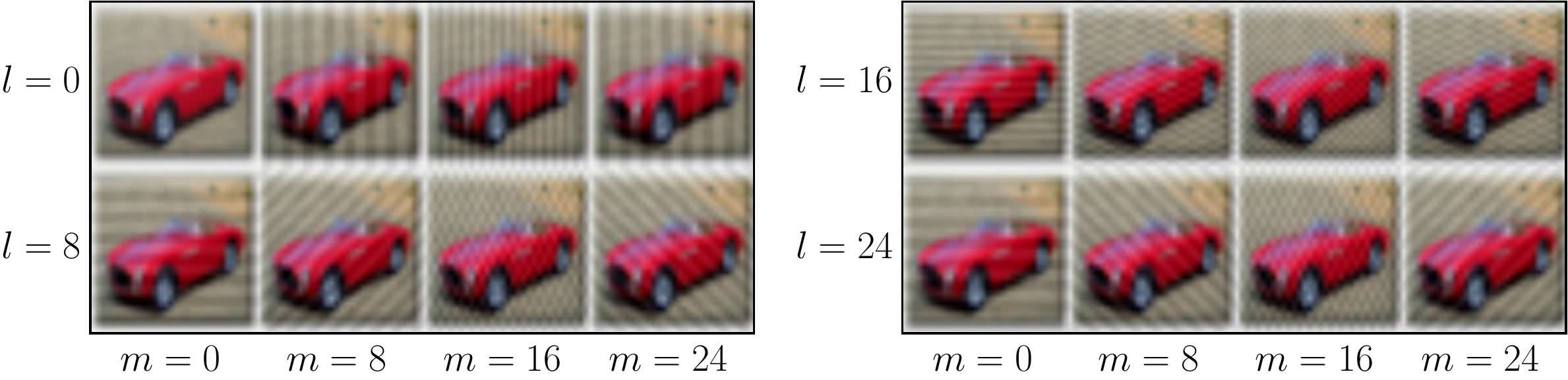} 
\caption{Examples perturbed by SFA of $(l,m)=(0,0), (0,8),\dots, (24,16), (24,24)$}
\label{sfa}
\end{figure}
\subsection{Vulnerability of CNNs in Frequency Domain}
Sensitivity to SFA can be regarded as sensitivity to a single-frequency noise \cite{Fp}. 
To understand the vulnerability of CNNs, several studies focused on
sensitivity of CNNs in the frequency domain \cite{Fp,highF,das2018shield,liu2019feature}.
These studies point out that sensitivity to high-frequency components
in images is one of the causes of adversarial attacks
since human visual systems are not sensitive to high-frequency components unlike CNNs.
In fact, several studies show that CNNs are sensitive to high-frequency noise 
\cite{jo2017measuring,highF,Fp,das2018shield}.
\citet{jo2017measuring} and \citet{highF} show that CNNs misclassify images processed by low-pass filters and \citet{highF} call this a High-Frequency attack, which is a simple black-box adversarial
attack. 
There is a hypothesis that
robust CNNs against high-frequency noise are also robust against adversarial attacks \cite{highF,Fp}.
Note that \citet{highF} claimed that sensitivity in the high-frequency domain
contributes to high performance on clean data; thus, there is a trade-off.
\subsection{Related Work}
Adversarial attacks can be transferred to other models
and transferred white-box attacks become adversarial 
black-box attacks \cite{papernot2017practical}.
These attacks can be defended against by adversarial training,
which is a promising defense method \cite{papernot2017practical,pgd2}.
However, the computational cost of adversarial training is
larger than naive training. 
Note that Absum can be used with adversarial training.
Several studies proposed black-box attacks using queries to ask the target model about predicted labels of given data,
but these attacks might still be impractical since they require 
many queries 
\cite{chen2017zoo,brendel2018decisionbased,pmlr-v80-ilyas18a}.
On the other hand, SFA only uses the information that the target 
model is composed of CNNs and is more practical.
%

Our method simply penalizes parameters in a similar manner compared to standard regularization methods.
As standard regularization methods, $L_2$ regularization (weight decay) is
commonly used for improving generalization performance due to its simplicity.
$L_1$ regularization is also used since it induces sparsity \cite{goodfellow2016deep}.
In addition, spectral norm (induced 2-norm) regularization can also improve generalization performance 
\cite{yoshida2017spectral,sedghi2018the}.
Due to space limitations, 
we outline other studies less relevant than the above studies in the appendix.
\section{Defense Methods against SFA}
In this section, we first show that SNC can improve robustness against SFA.
Since SNC has a large time complexity,
we next discuss whether standard regularizations can be alternatives.
Finally, we discuss Absum and its proximal operator, 
which is an efficient defense method against SFA.
\subsection{Spectral Norm Constraint}
SFA is based on the following properties of linear transform:
\begin{align}
\scalebox{0.92}{$
\textstyle
\sigma(\bm{A})\!=\!\max_{||\bm{x}||_2=1} ||\bm{A}\bm{x}||_2,~\bm{v}\!=\!\mathrm{arg}\max_{||\bm{x}||_2=1} ||\bm{A}\bm{x}||_2,
\textstyle
$}
\label{sn}
\end{align}
where $\sigma$ is the largest singular value (spectral norm or induced 2-norm),
and $\bm{v}$ is the right singular vector corresponding to $\sigma$.
Equation~(\ref{sn}) shows that the singular vector can be the worst noise for linear transform, and
SFA uses the singular vectors for convolutional layers.
Since the spectral norm determines the impact of SFA, 
we can reduce sensitivity to SFA by constraining the spectral norm.
The constraint of the spectral norm for CNNs (i.e., SNC)
 \cite{sedghi2018the,gouk2018regularisation} was proposed in the context of 
 improving generalization performance.
 SNC clips $\sigma$ if it exceeds a preset threshold; thus, it can directly
 control sensitivity to a single-frequency perturbation.
However, the constraints of the exact
spectral norm\footnote{The spectral norm in spectral norm regularization 
\cite{yoshida2017spectral} is often quite different from that of $\bm{A}$ \cite{sedghi2018the,gouk2018regularisation}.} of $\bm{A}$
 incurs large computation cost;
the $O(n^2c^2(c+\mathrm{log}(n)))$ time for each convolution when input size is $n\!\times\! n$, and
the numbers of input and output channels are $c$ even if we use the efficient spectral norm constraints \cite{sedghi2018the}. SNC can be infeasible when the size of inputs increases.
\subsection{Standard Regularizations fail to Defend} 
Instead of using the spectral norm,
we can assess the effect of the perturbation for linear transform by using
\begin{align}
\textstyle
\max_{||\bm{x}||_\infty=1} ||\bm{A}\bm{x}||_\infty.
\label{LOpt}
\end{align}
Equation (\ref{LOpt}) is the induced $\infty$-norm $||\bm{A}||_{\infty}$,
and we have $||\bm{A}||_{2}\!\leq\!||\bm{A}||_{\infty}$ for convolution 
(it is proved in the appendix). 
This norm is calculated as:
\begin{align}
\textstyle
||\bm{A}||_{\infty}=\max_{l} \sum_m |A_{l,m}|.
\label{solv}
\end{align}
 Substituting \req{dcm} for \req{solv}, we have
\begin{align}
\textstyle
\max_{l} \sum_m |A_{l,m}|=\sum_m\sum_l |k_{l,m}|.
\label{l1}
\end{align}
Thus, the penalty of the induced $\infty$-norm can be $L_1$ regularization \cite{gouk2018regularisation}.
Therefore, $L_1$ regularization can improve robustness.
However, the induced $\infty$-norm is a conservative measure of robustness 
\cite{szegedy2013intriguing}; the highly weighted $L_1$ regularization
for robustness can prevent minimization of the loss function.
Figure~\ref{LSSFA} shows the test accuracy of models, which is trained with $L_1$ regularization, on data perturbed by SFA against the regularization weight $\lambda$.
In this figure, the robust accuracy against SFA increases along with the regularization weight, i.e.,
the robustness increases according to the regularization weight.
However, the accuracy significantly decreases when the weight exceeds a certain point.
This is because training with high weighted $L_1$ regularization does not have sufficient search space
to minimize the loss function.
Note that weight decay can also penalize the 
spectral norm (in the appendix) and imposes tight
regularization, as discussed in the experiments section.
Therefore, we need a weak regularization method such that models
become both highly robust and accurate.
\begin{figure}[tb]
\centering
\centering
\includegraphics[width=0.8\linewidth ,clip]{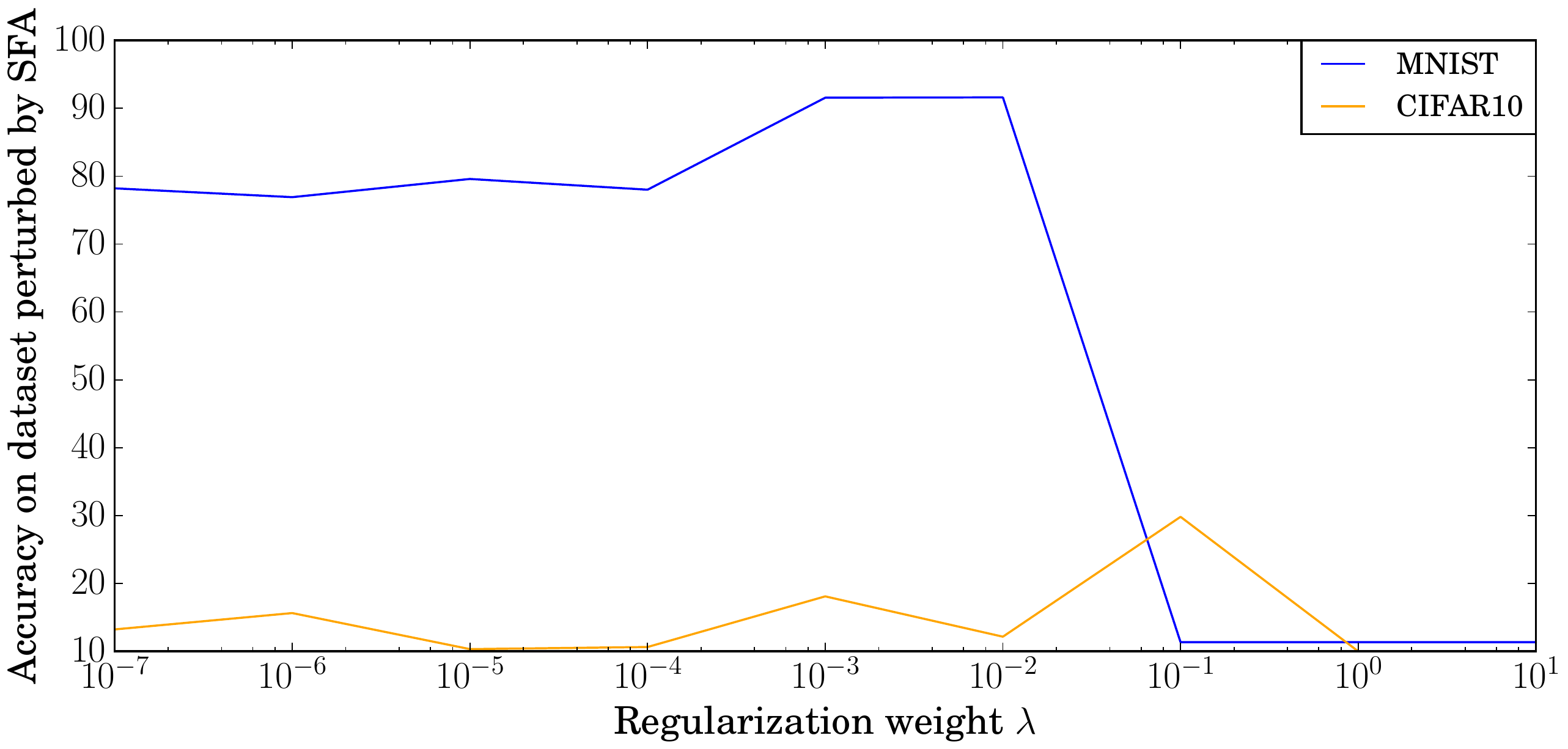}
\caption{Accuracy of models trained with $L_1$ regularization on test dataset perturbed by SFA vs regularization weight.
$l,m$ of SFA are tuned to minimize accuracy for each $\lambda$.}
\label{LSSFA}
\end{figure}
\begin{figure}[t]
\centering
\includegraphics[width=0.7\linewidth ,clip]{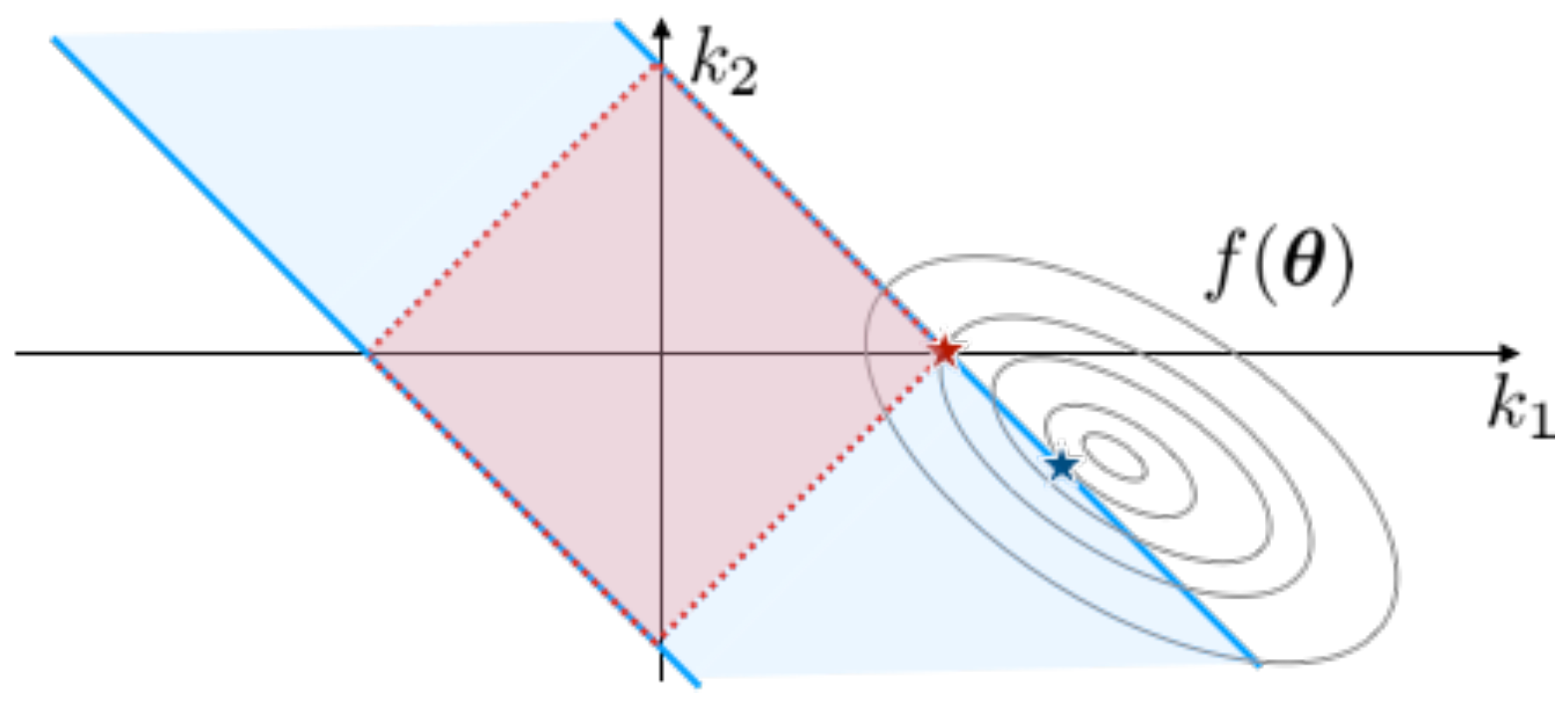}
\caption{Comparison of search spaces of Absum: $|k_1+k_2|$ (blue) and $L_1$ regularization: $|k1|+|k_2|$ (red) where $f(\bm{\theta})$ is loss function. We have $\{\bm{k}|\sum_i|k_i|\leq c\}\subseteq \{\bm{k}|~|\sum_ik_i|\leq c\}$ for any constant $c\geq 0$ from triangle inequality.}
\label{SSAbsum}
\end{figure}
\subsection{Absum: Simple and Weak Regularization}
To develop a weak regularization method,
we reconsider the optimization problem of \req{LOpt}.
The maximum point (\req{solv}) is achieved by $x_m\!=\!\mathrm{sign}(A_{l',m}\!)$, where
$l'\!=\!\mathrm{arg} \max_{l}\! \sum_m\! |A_{l,m}|$, i.e., $x_m\!=\!1$ if
 $A_{l',m}\!>\!0$ and $x_m\!=\!-1$ if $A_{l',m}\!<\!0$.
However, we should consider the sign of input in practice because we usually use ReLUs
 as activation functions.
As described in \req{rescomp}, ReLUs are used before convolution as $\bm{K}*ReLU(\cdot)$.
Thus, $\bm{x}$ cannot have negative elements, i.e., $x_m$ cannot be $\mathrm{sign}(A_{l',m})$ when $\mathrm{sign}(A_{l',m})\!=\!-1$. 
Therefore, the induced $\infty$-norm can overestimate sensitivity to the perturbation.
From this insight,
we consider the norm of $\bm{A}\bm{x}$ when $\bm{x}\!=\!\bm{1}$ instead of \req{LOpt}
\begin{align}
\textstyle
||\bm{A}\bm{1}||_\infty=\max_{l} |\sum_m A_{l,m}|=|\sum_m\sum_l k_{l,m}|.
\end{align}
For robustness, we use this value as the regularization term.
We call our method \textit{Absum} since this value is the absolute value of the summation of the filter coefficients.

The objective function of training with Absum is 
\begin{align}
\textstyle
\min_{\bm{\theta}} \frac{1}{N}\sum_{p=1}^N f(\bm{\theta},\bm{X}_p,\bm{Y}_p)&\textstyle+\lambda \sum_{i=1}^{L}g(\bm{K}^{(i)}),\label{obj}\\ \textstyle
g(\bm{K}^{(i)})&\textstyle
=|\sum_{m=1}^n\sum_{l=1}^n k^{(i)}_{l,m}|,\nonumber
\end{align}
where $f(\cdot)$ is a loss function,
$\bm{X}_p$ and $\bm{Y}_p$ are the $p$-th training image and label, respectively,
 $\bm{\theta}$ is the parameter vector including $\bm{K}^{(i)}$ in the model, and $\lambda$ is a regularization weight.
The $\bm{K}^{(i)}$ is the filter matrix of the $i$-th convolution, and $L$ is the number of
convolution filters.\footnote{We penalize the filter matrix for each channel. If one convolution layer has $c_1$ output channels and $c_2$ input channels, the regularization term becomes $\lambda\sum_{l=1}^{c_1}\sum_{m=1}^{c_2}g(\bm{K}^{(l+(m-1)c_1)})$.}
Figure~\ref{SSAbsum} shows search spaces of Absum (blue) and
$L_1$ regularization (red) when we have two parameters.
The constraint of Absum is looser than $L_1$ regularization
because a large element $k^{(i)}_{l,m}\!\gg\! 0$ is allowed if a small element $k^{(i)}_{l',m'}\!\ll\! 0$ satisfies $|k^{(i)}_{l,m }|\!=\!|k^{(i)}_{l',m'}|$. Even if $|\sum_l\!\sum_m\! k_{l,m}|\!=\!0$, the search space of Absum is a $n^2\!-\! 1$ dimensional space $\{\bm{K}|\bm{K}\!\in\! \bm{R}^{n\times n},\!\sum_l\! \sum_m\! k_{l,m}\!=\!0\}$
while that of $L_1$ regularization is a point $\bm{K}=\bm{O}$ if $\sum_l\!\sum_m\! |k_{l,m}|\!=\!0$. Note that
the search space of weight decay is also the point 
$\bm{K}=\bm{O}$ when $||\bm{K}||_F=\!0$.
Therefore, the loss function with Absum can be lower than that with $L_1$ regularization and weight decay
if we use a large $\lambda$. 

Note that when the filter size is $h\!\times\! h$ and $h\!<\!n$,
we only need to compute $|\sum_{m=1}^h\!\sum_{l=1}^h \!k_{l,m}|$ since zeros padded in $\bm{K}$ do not affect \req{obj} (hereafter, we use $h$ instead of $n$). 
\subsection{Proximal Operator for Absum}
Since $g(\bm{K})$ is not differentiable at $\sum_l\! \sum_m\! k_{l,m}\!=\!0$,
the gradient method might not be effective for minimizing \req{obj}.
To minimize \req{obj}, we use a proximal gradient method,
which can minimize a differentiable loss function with a non-differentiable 
regularization term \cite{proximal}.
We now introduce proximal operator for Absum.
For clarity, let $\bar{\bm{k}}$ be $\bar{\bm{k}}\!=\!\mathrm{vec}(\bm{K})\!=\!
[\bm{k}_{0}^T,\dots,\bm{k}_{h-1}^{T}]^T\!\in\!\bm{R}^{h^2}$.
The proximal operator for $\lambda g(\bar{\bm{k}})$ is
\begin{align}
\scalebox{0.81}{$
\mathrm{prox}_{\lambda g}(\bm{ \bar{\bm{k}}})\!=\!
\begin{cases}
\bm{ \bar{\bm{k}}}\!+\!\lambda\bm{1}&\!\!\mathrm{if}\sum_l \sum_m k_{l,m}\!<\!-h^2\lambda,\\ 
\bm{ \bar{\bm{k}}}\!-\!\frac{\sum_l\!\sum_m\!k_{l,m}}{h^2}\bm{1}&\!\!\mathrm{if}-h^2\lambda
\!\leq\! \sum_l \sum_m k_{l,m}\!\leq\!h^2\lambda,\\
\bm{ \bar{\bm{k}}}\!-\!\lambda\bm{1}&\!\!\mathrm{if}\sum_l \sum_m k_{l,m}\!>\!h^2\lambda.
\end{cases}$}
\label{prox}
\end{align} 
The following lemmas show that \req{prox} is the proximal operator for Absum:
\begin{lemma}
\label{convlem}
If $\bar{\bm{k}}=[\bar{k}_1,\dots,\bar{k}_{\bar{n}}]^T \in \bm{R}^{\bar{n}}$,
 $g(\bar{\bm{k}})=|\sum_i \bar{k}_i|$ is a convex function. 
\end{lemma}
\begin{lemma}
\label{prlem}
If $\bar{\bm{k}}=[\bar{k}_1,\dots,\bar{k}_{\bar{n}}]^T \in \bm{R}^{\bar{n}}$, $\bm{u}\in \bm{R}^{\bar{n}}$
 and $g(\bar{\bm{k}})=|\sum_i \bar{k}_i|$, 
we have 
\begin{align}
\textstyle
\!\!\mathrm{prox}_{\lambda g}(\bm{ \bar{\bm{k}}})\!&=\!
\textstyle
\mathrm{arg}\min_{\bm{u}}\frac{1}{2}||\bm{u}-\bar{\bm{k}}  ||_2^2+\lambda|\sum_i u_i|\\
\textstyle 
\!&=\!
\textstyle
\begin{cases}
\bm{ \bar{\bm{k}}}+\lambda\bm{1}\!\!&\mathrm{if}~\sum_i \bar{k}_i<-\bar{n}\lambda,\\
\textstyle 
\bm{ \bar{\bm{k}}}-\frac{\sum_i \bar{k}_i}{\bar{n}}\bm{1}\!\!&\mathrm{if}~-{\bar{n}}\lambda
\leq \sum_i \bar{k}_i\leq {\bar{n}}\lambda,\\
\textstyle 
\bm{ \bar{\bm{k}}}- \lambda\bm{1}\!\!&\mathrm{if}~\sum_i \bar{k}_i>{\bar{n}}\lambda.
\end{cases}
\label{prox2}
\end{align}
\end{lemma}
The proofs of lemmas are provided in the appendix.
Lemma \ref{convlem} shows that we can use the proximal gradient method,
and Lemma \ref{prlem} shows that the proximal operator of Absum can be obtained as the closed-form of \req{prox2}. 
By using the proximal operator after stochastic gradient descent (SGD), we update the $i$-th convolution filter:
\begin{align}
\textstyle
\bar{\bm{k}}^{(i)}\!\leftarrow\!\mathrm{prox}_{\eta \lambda g}(\bar{\bm{k}}^{(i)}\!-\!\eta\nabla_{\bar{\bm{k}}^{(i)}}\frac{1}{B}\sum_{b=1}^{B}f(\bm{\theta},\bm{X}_b,\bm{Y}_b ))
\end{align}
where $\eta$ is a learning rate, and $B$ is a minibatch size. 
We provide the pseudocode of the whole training in the appendix.
We can compute the proximal operator in $O(h^2)$ time for each convolution when the filter size is $h\times h$ because we only need to compute the summation of parameters and elementwise operations. 
We can also compute weight decay and $L_1$ regularization in $O(h^2)$ since the number of
parameters in each convolution is $h^2$.
Therefore, the order of computational complexity of Absum is 
the same as those of weight decay
and $L_1$ regularization.
When we have $c$ input channels and $c$ output channels,
the computational costs of Absum, weight decay, and $L_1$ regularization are $O(c^2h^2)$ and less than
that of SNC $O(c^2n^2(c+\mathrm{log}(n)))$ where $n\geq h$.

Note that the loss function $f$ for training deep neural networks
 is usually non-convex while $g(\bm{K})$ is convex.
Several studies investigate the proximal gradient method when $f$ is non-convex
\cite{li2015accelerated},
and \citet{NIPS2016_6504} use
the proximal gradient method for inducing sparse structures
in deep learning.
We observed that the algorithm of Absum can find a good parameter point during the experiments.
\begin{table*}[tb]
\centering
\caption{Accuracies on datasets perturbed by SFA.}
\label{sfatab}
\scalebox{0.79}{
\begin{tabular}{lrrrrrrrrrrrrrrrrr}
\toprule
& \multicolumn{4}{c}{Avg.}&\multicolumn{4}{c}{Min.} &\multicolumn{4}{c}{CLN}&\multicolumn{4}{c}{$\lambda$ and $\sigma$}\\
\cmidrule(l){2-5}\cmidrule(l){6-9}\cmidrule(l){10-13}\cmidrule(l){14-17}
{} &      Absum &         WD &        L1 &        SNC &      Absum &         WD &         L1 &SNC &      Absum &         WD &         L1 &        SNC &      Absum &         WD &         L1 &SNC \\
\midrule
MNIST &  \bf{98.64} &  98.59 &  98.48 &  98.55 &  \bf{94.76} &  86.84 &  78.01 &  91.79 &  99.14 &  99.10 &  \bf{99.18} &  99.10&$10^{-2}$&$10^{-3}$&$10^{-4}$&$10$\\
FMNIST &  \bf{83.11} &  83.09 &  82.49 &  82.60 &  \bf{60.12} &  47.57 &  58.38 &  55.36 &  \bf{88.46} &  86.99 &  87.05 &  87.50&$10^{-3}$&$10^{-2}$&$10^{-3}$&$10$\\
CIFAR10 &  79.05 &  69.09 &  66.44 &  \bf{85.57} &   53.90 &  11.44 &  15.64 &  \bf{73.99} &  89.69 &  \bf{94.73} &  \bf{93.41} &  88.37 &$10^{-1}$&$10^{-4}$&$10^{-6}$&$0.5$\\
CIFAR100 &  48.69 &  42.97 &  38.99 &  \bf{60.42} &  16.32 &   5.23 &  9.84 &  \bf{45.05} &  68.72 &  67.05 &  \bf{71.68} &  62.76&$10^{-3}$&$10^{-6}$&$10^{-7}$&$1$ \\
SVHN &  \bf{93.34} &  91.74&  91.14 &  93.20 &  \bf{73.69}&  60.36 &  57.52&  62.90 &  95.93 &  \bf{96.37} &  96.20 &  95.42&$10^{-3}$&$10^{-3}$&$10^{-7}$&$0.1$\\
\bottomrule
\end{tabular}
}
\end{table*}
\section{Experiments}\label{exp}
We discuss the evaluation of the effectiveness of SNC
and Absum in improving robustness against SFA.
Next, we show that Absum is more efficient than SNC especially
when the size of input images and models are large.
Finally, as the further investigation,
we discuss the evaluation of the performance of Absum and SNC in
terms of robustness against transferred attacks,
vulnerability in frequency domain, and robustness against PGD
when used with adversarial training.
To evaluate effectiveness, we conducted experiments of image recognition on MNIST \cite{mnist}, FMNIST \cite{fmnist}, CIFAR10, CIFAR100 \cite{cifar}, and SVHN \cite{svhn}.
We compared Absum and SNC with standard regularizations (weight decay (WD) and $L_1$ regularization). 
\subsection{Experimental Conditions}
We provide details of the experimental conditions in the appendix.
In all experiments, we selected the best regularization weight from among $[10^{1}, 10^{0},\dots, 10^{-7}]$ for Absum and standard regularization methods, 
and the best spectral norm $\sigma$ from among $[0.01, 0.1, 0.5, 1.0, 10]$
for SNC.
In SNC, we clipped $\sigma$ once in 100 iterations due to the large computational cost.
For MNIST and FMNIST, we stacked two convolutional layers 
and two fully connected layers and used ReLUs as activation functions.
For CIFAR10, CIFAR100, and SVHN, the model architecture was ResNet-18 \cite{resnet}.
We used SFA with $l,m\!\in\!\{0,1,\dots,27\}$ and $\varepsilon=80/255$ on MNIST and FMNIST,
and $l, m\!\in\! \{0,1,\dots,31\}$ and $\varepsilon\!=\!10/255$ on CIFAR10, CIFAR100, and SVHN.

In addition, we used PGD to evaluate robustness against transferred attacks
and white-box attacks since PGD is a sophisticated white-box attack.
In addition to naive training, we evaluated robustness against PGD
when we used adversarial training \cite{pgd,pgd2}
 with each method because Absum can be used with it due to its simplicity.
Model architectures were the same as in the experiments involving SFA.
The hyperparameter settings for PGD were based on \cite{pgd2}.
The $L_\infty$ norm of the perturbation $\varepsilon$ was set to $\varepsilon\!=\!0.3$ for MNIST and FMNIST
and $\varepsilon\!=\!8/255$ 
for CIFAR10, CIFAR100, and SVHN at training time.
For PGD, we updated the perturbation for 40 iterations with a step size of 0.01 on MNIST and FMNIST
at training and evaluation times,
and on CIFAR10, CIFAR100, and SVHN, for 7 iterations with a step size of 2/255
at training time and 100 iterations with the same step size at evaluation time.
\subsection{Effectiveness and Efficiency}
\subsubsection{Robustness against SFA}
Table~\ref{sfatab} lists the accuracies of each method on test data perturbed by SFA
and selected $\lambda$ and $\sigma$.
In this table, Avg. denote robust accuracies against SFA averaged over $(l,m)$, and
Min. denotes minimum accuracies among hyperparameters $(l, m)$,
i.e., robust accuracies against optimized SFA. 
CLN denotes accuracies on clean data.
The $\lambda$ and $\sigma$ are selected so that Avg. would become the 
highest.
In \rtab{sfatab}, Absum and SNC are more robust against SFA compared with WD and $L_1$.
Although SNC is more robust than Absum on CIFAR10 and CIFAR100,
clean accuracies of SNC are less than those of Absum and
the computation time of SNC is larger than that of Absum as discussed below.
In the appendix, we provide
accuracies against each $(l,m)$ and 
the results in which $\lambda$ and $\sigma$ are selected so that each of CLN 
and Min. would become highest.

Figure~\ref{AccVsLam} shows the test accuracies of the methods on MNIST and CIFAR10 perturbed by SFA against regularization weights.
In this figure, min and max denote the minimum and maximum test accuracies among $(l,m)$, respectively,
and avg. denotes test accuracies  averaged over $(l,m)$.
All methods tend to increase their minimum accuracy
(results of SFA with optimized $(l, m)$) according to the regularization weight.
However, $L_1$ and WD significantly decrease in accuracy when the regularization weight
is higher than $10^{-1}$.
On the other hand, Absum with the high regularization weight does not decrease in accuracy.
Figure~\ref{trainLoss} shows the lowest training loss $\frac{1}{N}\sum f$ in training on CIFAR10 against $\lambda$.
WD and $L_1$ with a large $\lambda$ prevent minimization of the training loss.
On the other hand, Absum with a large $\lambda$ can decrease the training loss because the search space of $\bm{K}^{(i)}\!\in\!\bm{R}^{h\times h}$ has $h^2\!-\!1$ dimensional space even if $g(\bm{K}^{(i)})\!=\!0$.
In conclusion, standard regularization methods might not be 
effective in improving robustness against SFA
because the high regularization weight imposes too tight of constraints to minimize the loss function.
On the other hand, Absum imposes looser constraints; thus,
we can improve robustness while maintaining classification performance.
The results of other datasets 
are almost the same as \rfig{AccVsLam} (included in the appendix). We also provide figures 
showing the accuracy and the training loss of SNC against $\sigma$ in the appendix.
\begin{figure}[tbp]
\centering
\subfloat[MNIST]{\includegraphics[width=\linewidth ,clip]{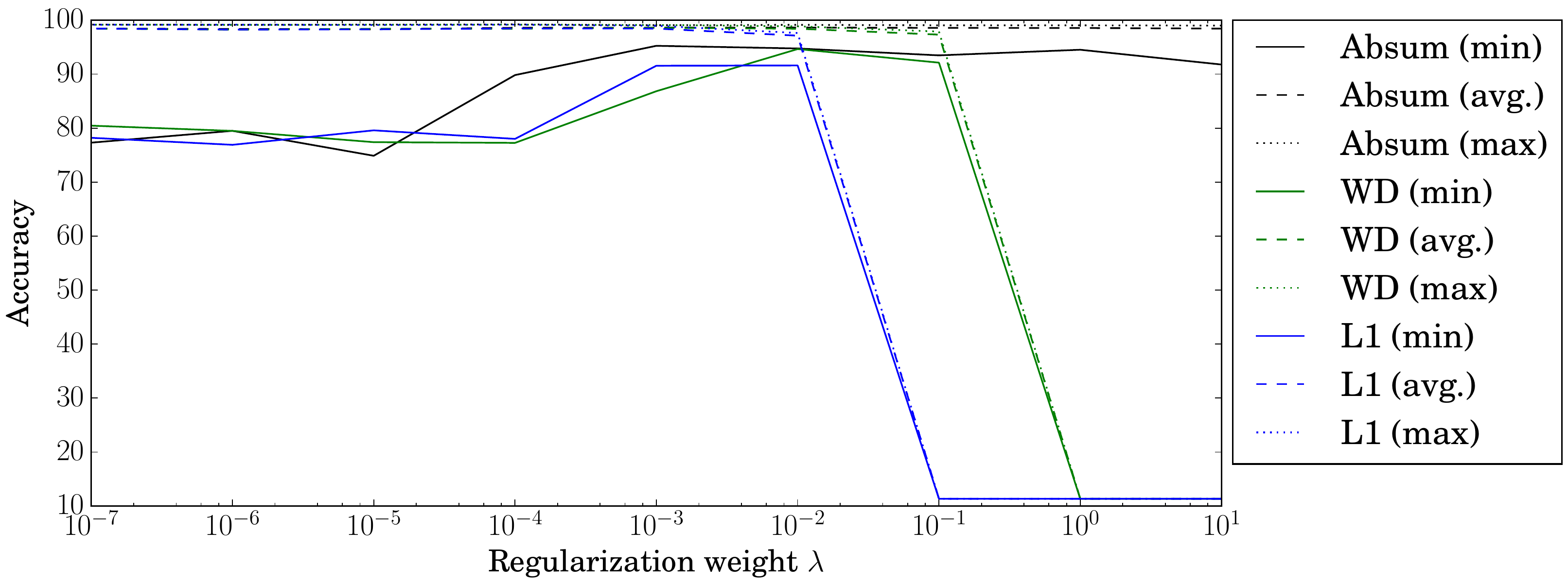}}\\
\subfloat[CIFAR10]{\includegraphics[width=\linewidth ,clip]{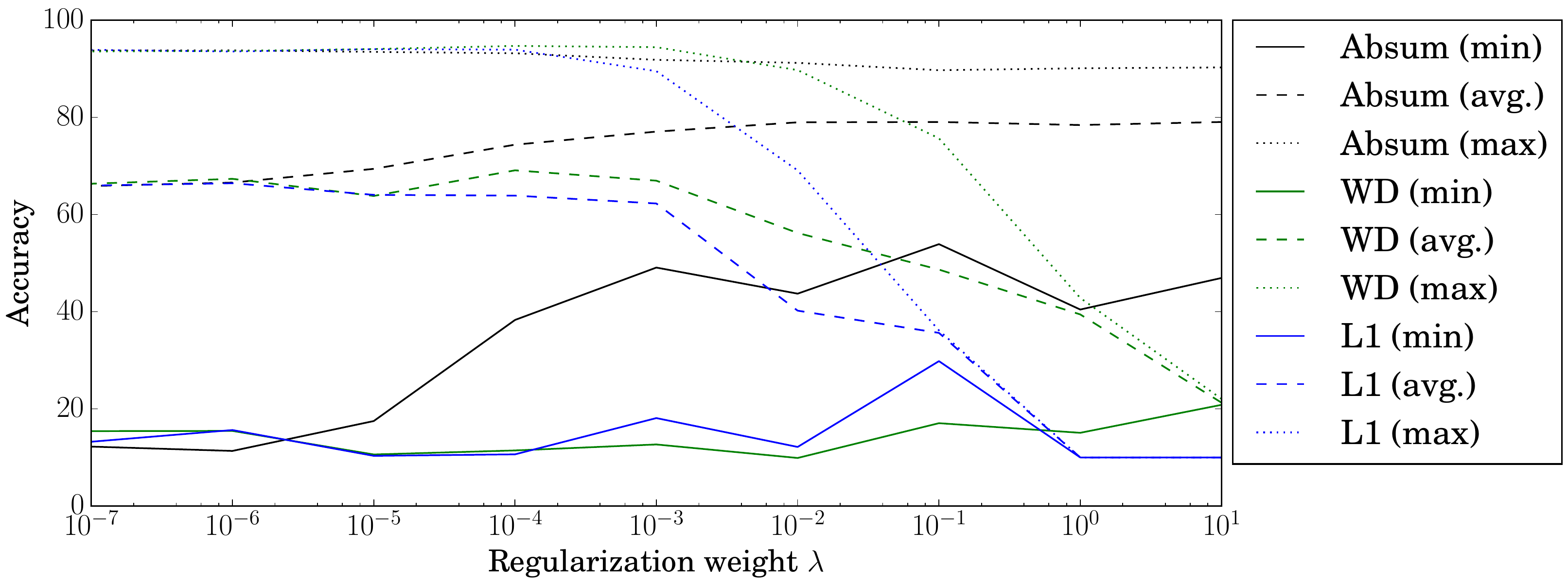}}
\caption{Accuracy on test datasets perturbed by SFA vs $\lambda$}
\label{AccVsLam}
\end{figure}
\begin{figure}[t]
\centering
\includegraphics[width=0.8\linewidth ,clip]{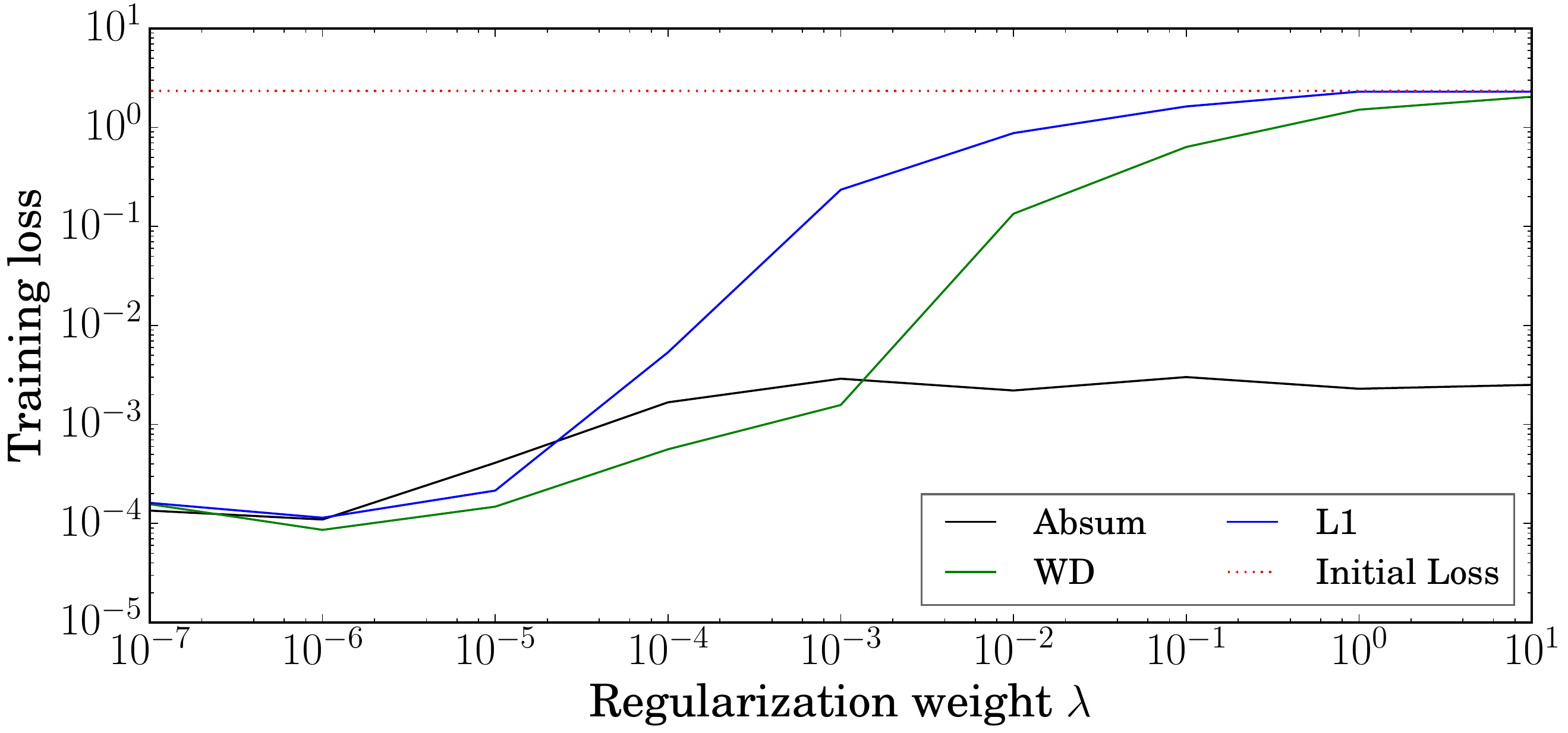}
\caption{Training loss vs $\lambda$}
\label{trainLoss}
\end{figure}
\subsubsection{Computational Cost}
To confirm the efficiency of Absum, we evaluated the runtime for one epoch.
We also evaluated the runtime of the forward and backward
 processes of ResNet-18 for one image when input size increases
by using random synthetic three channels images whose sizes
were 32$\times$32, 64$\times$64, 128$\times$128, 256$\times$256, 512$\times$512, and 1024$\times$1024 with 10 random labels.
The results are shown in \rfig{CompTime}.
As shown in \rfig{CompTime}~(a), Absum is about ten times faster than SNC on $32\times 32$ image datasets
with ResNet18. 
The runtime of SNC is comparable to those of other methods on MNIST and FMNIST
because we use only two convolution layers, and image sizes of these datasets are smaller than other
 datasets.
In \rfig{CompTime}~(b), the runtime of Absum does not increase significantly
compared with SNC and the increase in the runtime of Absum is similar to
those of standard regularization methods.
This is because the computational cost of Absum does not depend on the size of input images.
Since SNC incurs large computational cost and depends on the input size, we could not evaluate
the runtime when the image width is larger than 256.
\begin{figure}[tb]
\centering
\subfloat[Runtime for one epoch]{\includegraphics[width=0.488\linewidth]{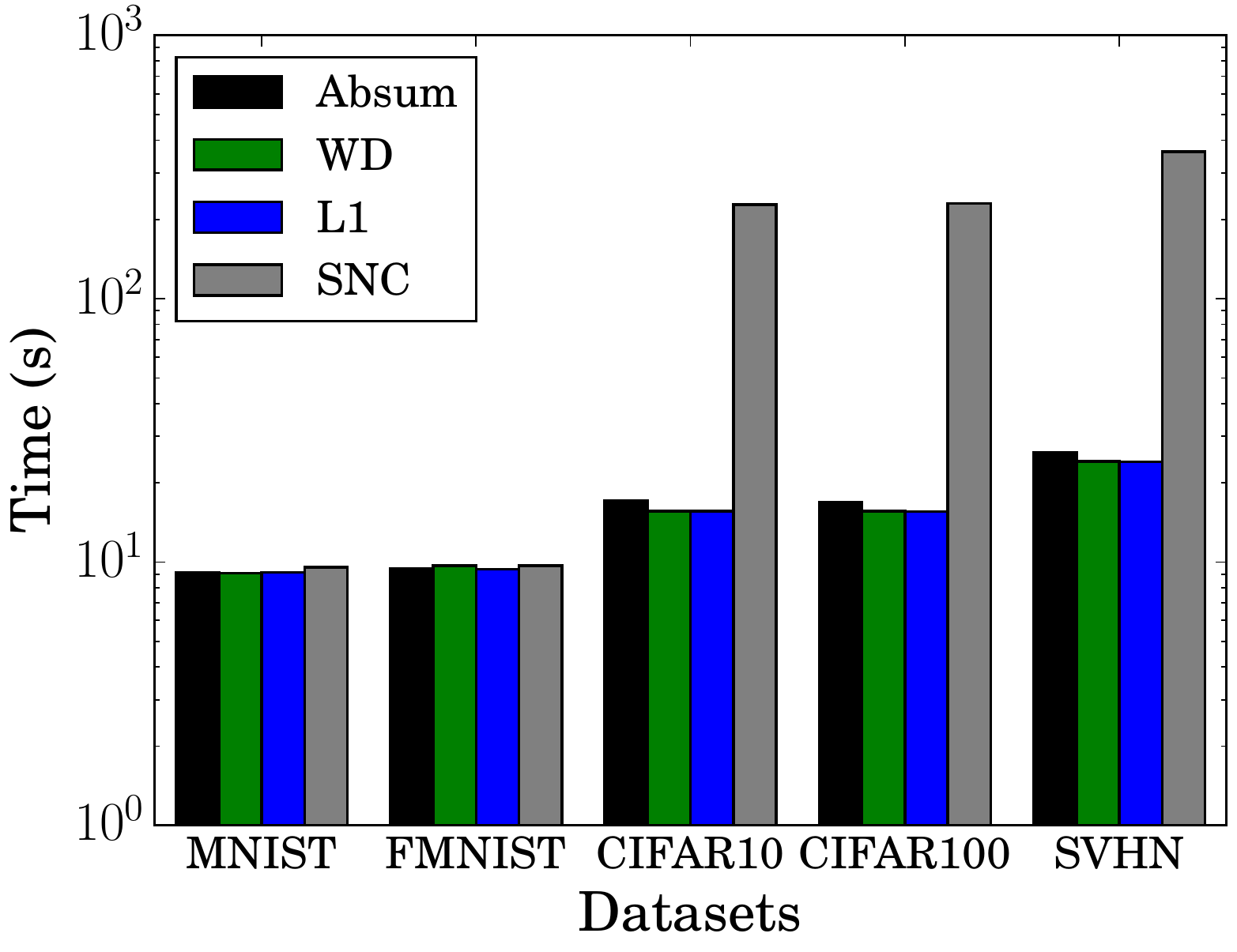}}\centering\hfill
\subfloat[Runtime vs. input size]{\includegraphics[width=0.488\linewidth]{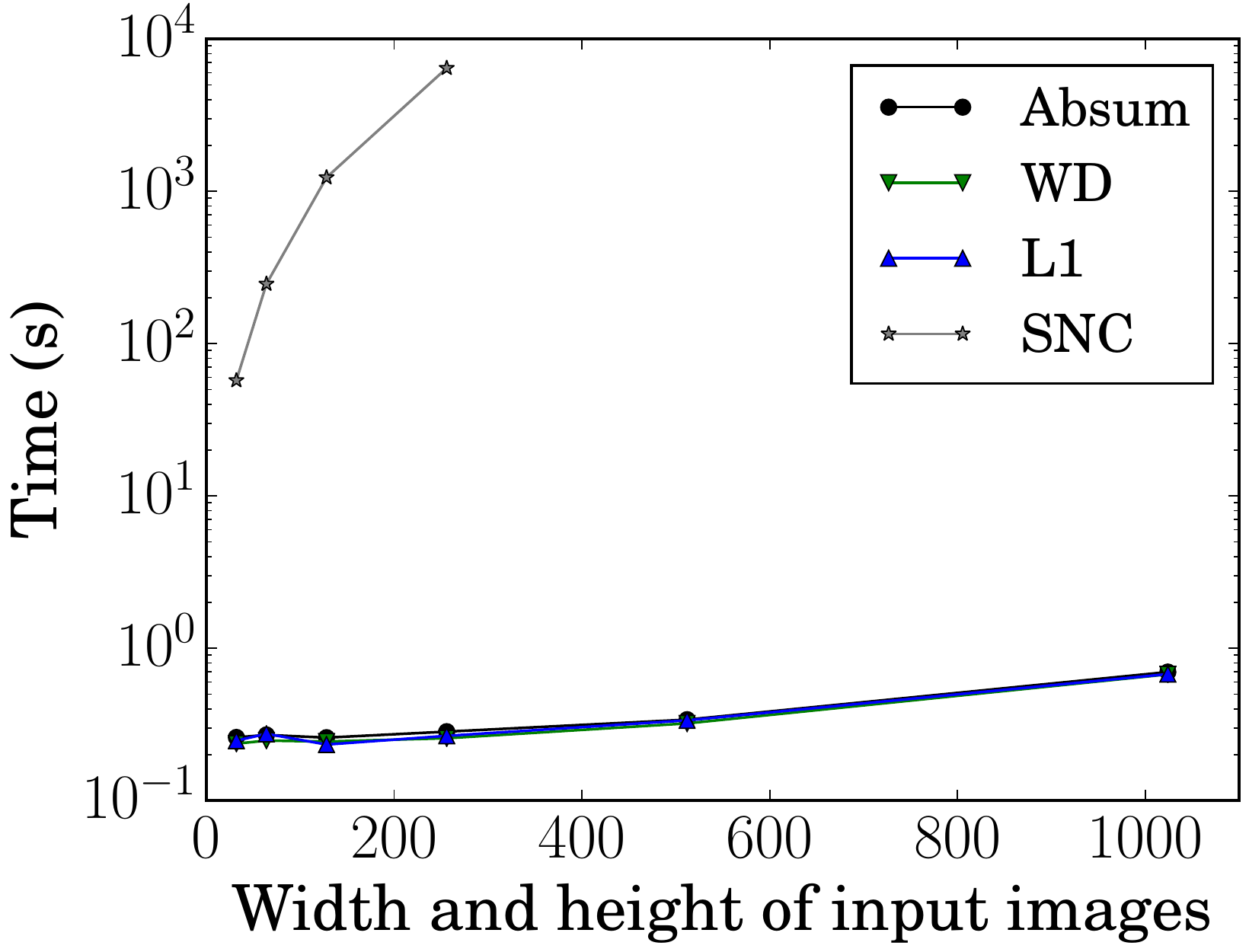}}
%
\caption{Computation time}
\label{CompTime}
\end{figure}
\subsection{Extensive Empirical Investigation}
\subsubsection{Robustness against Transferred Attacks}
Sensitivity to SFA is caused by convolution operation and is universal for CNNs.
This sensitivity might be a cause of transferability of adversarial attacks, and robust
CNNs against SFA can be robust against transferred attacks.
To confirm this hypothesis, we investigate sensitivity to transferred PGD.
We generate adversarial examples by using the substitute models that were trained under 
the same setting as that presented in the previous section but with different random initializations.
We used these substitute models rather than completely different models
because they can be regarded as one of the worst-case instances for transferred attacks \cite{pgd2}.
The accuracies on these adversarial examples are listed in \rtab{transfer}.
Absum and SNC improve robustness compared to WD and $L_1$.
Tables~\ref{sfatab} and \ref{transfer} imply that the method of
improving robustness against SFA can also improve robustness against the transferred attacks.
This is the first study that shows the relation between
robustness against SFA and against transferred white-box attacks.
\begin{table}[t]
\centering
\caption{Robust accuracy against transferred PGD attacks. Results of other $\varepsilon$
are shown in the appendix. 
w/o Reg. denotes results of training without regularization.}
\label{transfer}
\scalebox{0.79}{
\begin{tabular}{ccccccc}
\toprule
&Absum&WD&L1&SNC&w/o Reg.\\
\midrule
MNIST ($\varepsilon\!=\!0.2$)&{\bf 76.34}&48.94&66.48&71.30&65.87\\
FMNIST ($\varepsilon\!=\!0.2$)&{\bf 30.08}&3.46&18.35&21.31&19.74\\
CIFAR10 ($\varepsilon\!=\!4/255$)&26.29&18.48&15.66&{\bf 48.85}&15.85\\
CIFAR100 ($\varepsilon\!=\!4/255$)&18.57&17.40&16.68&{\bf 36.57}&16.68\\
SVHN ($\varepsilon\!=\!4/255$)&49.11&40.49&52.79&46.36&\bf{54.39}
\\
\bottomrule
\end{tabular}
}
\end{table}
\subsubsection{Sensitivity in Frequency Domain}
\begin{figure}[tb]
\centering
\includegraphics[width=\linewidth ,clip]{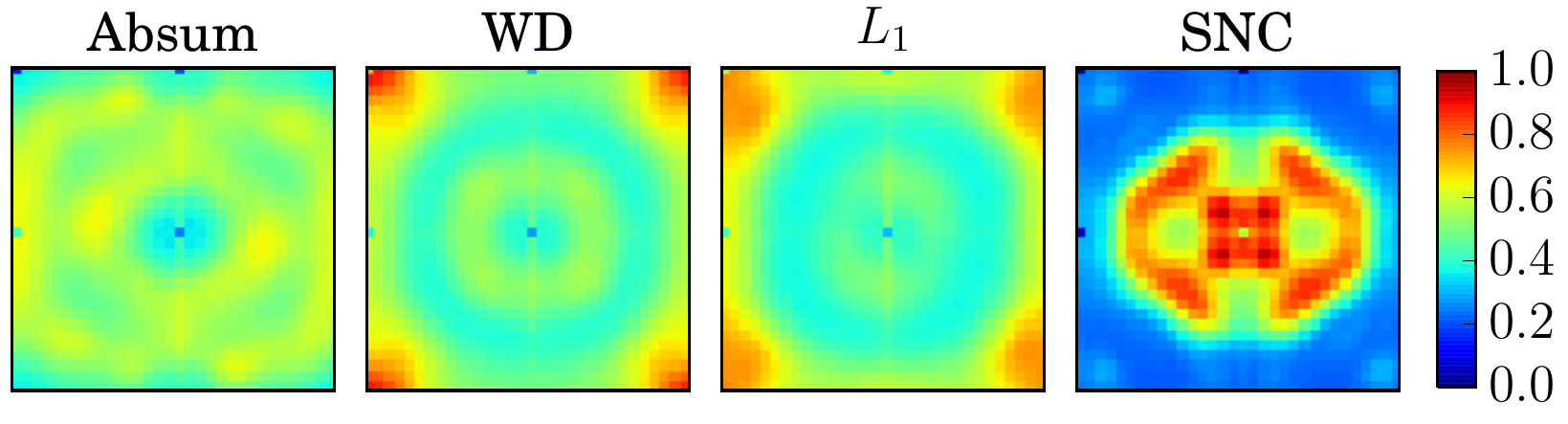}
\caption{Power spectra of PGD perturbations on CIFAR10. 
Magnitudes in low-frequency and high-frequency domains
 are located near center and edge of each figure, respectively.
They are normalized as in (0, 1) after logarithmic transform.}
\label{PGDPower}
\end{figure}
\begin{table}[t]
\centering
\caption{Robust accuracy against High-Frequency attacks.} 
\label{hfa}
\scalebox{0.8}{
\begin{tabular}{ccccccc}
\toprule
&Absum&WD&L1&SNC&w/o Reg.\\
\midrule
MNIST&99.00&98.98&\bf{99.10}&98.97&99.01\\
FMNIST&84.15&83.91&82.56&\bf{84.30}&84.03\\
CIFAR10&64.51&52.82&47.01&{\bf 82.11}&47.46\\
CIFAR100&41.44&36.15&31.53&{\bf 61.22}&31.80\\
SVHN&{\bf 52.95}&28.11&17.03&18.75&11.13
\\
\bottomrule
\end{tabular}
}
\end{table}
Several studies show that CNNs are sensitive to high-frequency noise
unlike human visual systems since CNNs are biased towards
high-frequency information \cite{highF,Fp}.
From the robustness against SFA, which is regarded as single-frequency noise,
 Absum and SNC can be expected not to bias CNNs towards
high-frequency information. To confirm this hypothesis, we first investigate
the power spectra of adversarial perturbations of models trained using each method.
Next, we investigate robustness against High-Frequency attacks, which remove
high-frequency components of image data. High-Frequency attacks have a hyperparameter of radius
that determines the cutoff frequency, and we set it as half the image width. 
In these experiments, $\lambda$ and $\sigma$ are the same as those in \rtab{sfatab}.

Figure~\ref{PGDPower} shows the power spectra of PGD perturbations on CIFAR10 and 
\rtab{hfa} lists the accuracies on the test data processed by High-Frequency attacks.
In \rfig{PGDPower}, we shift low frequency components to the center of the spectrum and
power spectra are averaged over test data and RGB channels.
This figure shows that vulnerabilities of WD and $L_1$ are biased
in the high-frequency domain, while vulnerability of SNC is highly biased in the low-frequency domain.
The power spectrum of Absum is not biased towards a specific frequency domain.
Due to these characteristics, SNC and Absum are more robust against High-Frequency attacks
than WD and $L_1$ (\rtab{hfa}).
Since human visual systems can perceive low-frequency noise better than high-frequency noise,
attacks for Absum and SNC might be more perceptible than attacks for WD and $L_1$.
Note that Absum is more robust against
high-pass filtering than SNC, which is presented in the appendix.
This result supports that Absum does not bias CNNs
towards a specific frequency domain while SNC biases CNNs
towards the low-frequency domain. 
\subsubsection{Robustness against PGD with Adversarial Training}
Table~\ref{data ad} 
lists the accuracies of models trained by 
adversarial training on data perturbed by PGD.
When using adversarial training, Absum improves robustness against PGD, 
the highest among regularization methods, on almost all datasets.
This implies that 
sensitivity to SFA is one of the causes of vulnerabilities of CNNs.
The $\lambda$ of Absum tends to be higher than the $\lambda$ of WD and $L_1$;
thus, Absum can also improve robustness against PGD
without deteriorating classification performance due to its looseness.
Note that Absum does not improve robustness against PGD
whithout adversarial training
since the structural sensitivity of CNNs does not necessarily cause all vulnerabilities of
CNN-based models (we discussed this in the appendix).
Even so, Absum is more effective than other standard regularizations since
it can efficiently improve robustness against black-box attacks (SFA, transferred attacks, and High-Frequency attacks)
and enhance adversarial training, as mentioned above.
\begin{table}[tb]
\centering
\caption{Accuracies (\%) on test datasets perturbed by PGD.}
\label{data ad}
\scalebox{0.74}{
\begin{tabular}{lrrrrrr}
\toprule
MNIST&\multicolumn{6}{c}{Adversarial training}\\
\midrule
$\varepsilon$&   0.05 &   0.10 &   0.15 &   0.20 &   0.25 &   0.30 \\
\midrule
Absum $\lambda=10^{-3}$& \bf{96.01}&  \bf{94.92} & \bf{93.75}&  \bf{92.73}& \bf{91.59}&  \bf{90.78}\\
WD $\lambda=10^{-4}$&  92.97 &  91.34 &  89.69 &  88.02 &  87.05 &  85.96\\
L1 $\lambda=10^{-4}$&  93.12 &  91.86 &  90.60 &  89.28 &  88.25 &  87.06 \\
SNC $\sigma=10$&  91.92 &  89.43 &  86.77 &  83.89 &  80.24 &  76.92 \\
w/o Reg. &  91.57 &  89.85 &  88.43 &  86.87 &  85.76 &  84.86 \\
\bottomrule
\toprule
FMNIST&\multicolumn{6}{c}{Adversarial training}\\
\midrule
%
Absum $\lambda=10^{-3}$&  \bf{66.94} &  \bf{65.92} &  \bf{65.77} &  \bf{65.52} &  \bf{65.24} &  \bf{64.95} \\
WD $\lambda=10^{-5}$&  65.38 &  63.64 &  62.91 &  62.60 &  62.11 &  61.96 \\
L1 $\lambda=10^{-6}$&  66.13 &  64.16 &  62.95 &  62.23 &  61.64 &  61.66 \\
SNC $\sigma=10$&  51.58 &  49.33 &  47.31 &  45.85 &  44.86 &  44.04 \\
w/o Reg. &  63.36 &  61.66 &  61.15 &  60.97 &  60.46 &  60.26 \\
\bottomrule
\end{tabular}
}
\centering
\scalebox{0.79}{
\begin{tabular}{lrrrrr}
\toprule
CIFAR10&\multicolumn{5}{c}{Adversarial training}\\
\midrule
$\varepsilon$ &   4/255  &   8/255  &   12/255 &   16/255 &   20/255 \\
\midrule
Absum $\lambda=10^{-5}$&  69.42 &  49.39 &  \bf{30.22} &  \bf{15.03} &  \bf{6.54} \\
WD $\lambda=10^{-5}$&   \bf{69.48} &  49.38 &  29.37 &  14.45 &  6.06 \\
L1 $\lambda=10^{-5}$& 68.99 &  \bf{49.45} &  29.51 &  14.68 &  6.31 \\
SNC $\sigma=10$&  68.47 &  48.74 &  29.07 &  14.32 &  6.04 \\
w/o Reg. &  68.46 &  48.77 &  29.20 &  14.50 &  6.08 \\
\bottomrule
\toprule
CIFAR100&\multicolumn{5}{c}{Adversarial training}\\
\midrule
Absum $\lambda=10^{-4}$&  \bf{42.19} &  \bf{27.25} &  15.89 &  \bf{8.47} &  4.14 \\
WD $\lambda=10^{-7}$&  41.14 &  27.05 &  \bf{15.90} &  8.26 &  \bf{4.28} \\
L1 $\lambda=10^{-4}$&  40.75 &  26.14 &  14.45 &  7.61 &  3.67 \\
SNC $\sigma=10$&  40.90 &  26.61 &  15.53 &  8.32 &  4.13 \\
w/o Reg. &  40.70 &  26.24 &  14.85 &  7.94 &  3.86 \\
\bottomrule
\toprule
SVHN&\multicolumn{5}{c}{Adversarial training}\\
\midrule
Absum $\lambda=10^{-5}$&  \bf{77.78} &  \bf{52.74} &  \bf{27.39} &  11.97 &  5.50 \\
WD $\lambda=10^{-7}$&  76.66 &  50.40 &  25.05 &  10.86 &  5.04 \\
L1 $\lambda=10^{-6}$&  76.50 &  51.49 &  27.10 &  \bf{12.12} &  \bf{5.63} \\
SNC $\sigma=1.0$&  77.23 &  50.80 &  25.24&  11.04 &  5.03 \\
w/o Reg. & N/A& N/A& N/A& N/A& N/A\\
\bottomrule
\end{tabular}
}
\end{table}
\section{Conclusion}
We proposed Absum; an efficient defense
method against SFA that can reduce the structural sensitivity of CNNs with ReLUs while
its computational cost remains comparable to standard regularizations.
By reducing the structural sensitivity, 
Absum can improve robustness against not only SFA, but also
transferred PGD, and High-Frequency attacks.
Due to its simplicity, Absum can be used with other methods, 
and Absum can enhance adversarial training of PGD.
\bibliographystyle{aaai}
\bibliography{Absum.bib}
\appendix
\section*{Appendix}
\section{Proofs of Lemmas}
In this section, we provide the proofs of the lemmas.
\begin{lemma}
If $\bar{\bm{k}}=[\bar{k}_1,\dots,\bar{k}_{\bar{n}}]^T \in \bm{R}^{\bar{n}}$,
 $g(\bar{\bm{k}})=|\sum_i^{\bar{n}} \bar{k}_i|$ is a convex function. 
\end{lemma}
\begin{proof}
If $g(\cdot)$ is a convex function, 
we have $g(t\bm{x}+(1-t)\bm{y})\leq tg(\bm{x})+(1-t)g(\bm{y})$, where $t\in [0,1]$ and $\forall \bm{x},\bm{y}\in \bm{R}^{\bar{n}}$.
Therefore, we investigate $J=tg(\bm{x})+(1-t)g(\bm{y})-g(t\bm{x}+(1-t)\bm{y})$, and
if $J\geq 0$, we prove the lemma. We have
\begin{align}
\scalebox{0.9}{$
\textstyle
J=$}&\scalebox{0.9}{$
\textstyle
t|\sum_i x_i|+(1-t)|\sum_i y_i|-|\sum_i (tx_i +(1-t)y_i)|,\nonumber $} \\
\scalebox{0.9}{$=$}&\scalebox{0.9}{$
\textstyle
|t\sum_i x_i|+|(1-t)\sum_i y_i|-|t\sum_i x_i +(1-t)\sum_iy_i|,$}
\end{align}
since $t\geq 0$ and $1-t\geq 0$.
Let $\alpha=t\sum_i x_i$ and $\beta=(1-t)\sum_i y_i$; thus, we have
\begin{align}
J=&|\alpha|+|\beta|-|\alpha+\beta|.
\end{align}
From the triangle inequality, we have $J\geq 0$; thus, this completes the proof.
\end{proof}
\begin{lemma}
If $\bar{\bm{k}}=[\bar{k}_1,\dots,\bar{k}_{\bar{n}}]^T \in \bm{R}^{\bar{n}}$, $\bm{u}\in \bm{R}^{\bar{n}}$
 and $g(\bar{\bm{k}})=|\sum_i^{\bar{n}} \bar{k}_i|$, 
we have 
\begin{align}
\textstyle
\mathrm{prox}_{\lambda g}(\bm{ \bar{\bm{k}}})=&
\textstyle
\mathrm{arg}\min_{\bm{u}}\frac{1}{2}||\bm{u}-\bar{\bm{k}}  ||_2^2+\lambda|\sum_i^{\bar{n}}u_i|\label{oppr}\\
=&
\textstyle
\begin{cases}
\bm{ \bar{\bm{k}}}+\lambda\bm{1}&\mathrm{if}~~~\sum_i^{\bar{n}} \bar{k}_i<-{\bar{n}}\lambda,\\
\bm{ \bar{\bm{k}}}-\frac{\sum_i^{\bar{n}} k_i}{\bar{n}}\bm{1}&\mathrm{if}~~~-{\bar{n}}\lambda
\leq \sum_i^{\bar{n}} \bar{k}_i\leq{\bar{n}}\lambda,\\
\bm{ \bar{\bm{k}}}- \lambda\bm{1}&\mathrm{if}~~~\sum_i^{\bar{n}} \bar{k}_i>{\bar{n}}\lambda.
\end{cases}
\end{align}
\end{lemma}
\begin{proof}
For clarity, let $J=\frac{1}{2}||\bm{u}-\bar{\bm{k}}  ||_2^2+\lambda|\sum_i^{\bar{n}}u_i|$.
We have three cases; 
(a) $\sum_i^{\bar{n}}u_i>0$, (b) $\sum_i^{\bar{n}}u_i<0$, and (c) $\sum_i^{\bar{n}}u_i=0$.
In (a), we have $|\sum_i^{\bar{n}}u_i|=\sum_i^{\bar{n}}u_i$, and 
$\frac{\partial J}{\partial u_i}=u_i-\bar{k}_i+\lambda=0$ at the optimal point.
Therefore, $u_i=\bar{k}_i-\lambda$, and the solution becomes $\bm{u}=\bar{\bm{k}}-\lambda\bm{1}$.
The condition is $\sum_i^{\bar{n}}u_i=\sum_i^{\bar{n}} \bar{k}_i-{\bar{n}}\lambda>0$, i.e., $\sum_i^{\bar{n}} \bar{k}_i>{\bar{n}}\lambda$.
In (b), we have $|\sum_i^{\bar{n}}u_i|=-\sum_i^{\bar{n}}u_i$, and we can optimize $J=\frac{1}{2}||\bm{u}-\bar{\bm{k}}  ||_2^2-\lambda\sum_i^{\bar{n}}u_i$ in the same manner as (a).
As a result, $\bm{u}=\bar{\bm{k}}+\lambda\bm{1}$ if $\sum_i^{\bar{n}} \bar{k}_i<-{\bar{n}}\lambda$.
In (c), $|\sum_i^{\bar{n}}u_i|$ is non-differentiable, but we can use subgradient $\bm{v}$ such as
$|\sum_i^{\bar{n}}z_i|\geq |\sum_i^{\bar{n}}u_i|+\bm{v}^T(\bm{z}-\bm{u})$. Let $B$ be $B=\{\bm{u}\pm r\bm{e}_i|i=1,\dots,n\}$ where small $r>0$ and $\bm{e}_k$ be the standard basis; thus, we have $M=\max_{\bm{z}\in\bm{B}}|\sum_i z_i|=|\sum_i u_i\pm r|=r$ when $|\sum_i u_i|=0$.
As a result, $\bm{v}$ is bounded as $||\bm{v}||_\infty\leq \frac{M-|\sum_i u_i|}{r}= 1$.
We then have $\frac{\partial J}{\partial u_i}=u_i-\bar{k}_i+\lambda v_i=0$; thus, $\bm{u}=\bar{\bm{k}}-\lambda \bm{v}$.
Since $||\bm{v}||_\infty\leq 1$, we have $-{\bar{n}}\lambda\leq\lambda\sum v_i\leq {\bar{n}}\lambda$.
Thus, the condition becomes $-{\bar{n}}\lambda\leq\sum_i \bar{k}_i\leq {\bar{n}}\lambda$ since $\sum_i u_i=\sum_i \bar{k}_i+\lambda\sum_i v_i=0$. By substituting $\bm{u}=\bar{\bm{k}}-\lambda \bm{v}$ into $J$, we have 
$J=\frac{1}{2}||\lambda\bm{v}  ||_2^2$ subject to $\sum_i v_i=\frac{\sum_i k_i}{\lambda}$
 and $||\bm{v}||_\infty\leq 1$. 
Thus, the minimum point is $v_1=v_2=\dots=v_{\bar{n}}=\frac{\sum_i k_i}{{\bar{n}}\lambda}$, i.e.,
$\bm{v}=\frac{\sum_i k_i}{{\bar{n}}\lambda}\bm{1}$. Therefore, $\bm{u}=\bar{\bm{k}}-\frac{\sum_i k_i}{\bar{n}}\bm{1}$
is the minimum point when $-{\bar{n}}\lambda\leq\sum_i \bar{k}_i\leq {\bar{n}}\lambda$.
This completes the proof.
\end{proof}
\section{Inequality of Induced Norms for Convolution}
The $u+n(v-1)$-th singular value of a doubly circulant matrix $\bm{A}$ can be written as 
$\sigma_{u,v}\!=\!
|\sum_{l,m}k_{l,m}\!\exp(j\frac{2\pi}{n}(ul+vm))|$ (not arranged in descending order),
 and we have $||\bm{A}||_2\!=\!\max_{u,v}\sigma_{u,v}\!\leq\!\sum_{l,m}|k_{l,m}|\!\leq\!||\bm{A}||_{\infty}$. 
Therefore, the spectral norm of $\bm{A}$ is bounded above
by the induced $\infty$-norm as $||\bm{A}||_2\!\leq\!||\bm{A}||_{\infty}$.
\section{$L_2$ Regularization and Induced Norm}
In this section, we explain that $L_2$ regularization
(weight decay: WD) can 
constrain the induced norm of a convolutional layer.
The $L_2$ regularization term of the convolution filter $\bm{K}\in \bm{R}^{n\times n}$ is
\begin{align}
\textstyle
\sum_{l}^n\sum_{m}^n k_{l,m}^2.
\end{align}
On the other hand, the square of the Frobenius norm of $\bm{A}$ becomes
\begin{align}
\textstyle ||\bm{A}||^2_{F}=\sum_{l}^{n^2}\sum_{m}^{n^2}A_{l,m}^2=n^2\sum_{l}^{n}\sum_{m}^{n}k_{l,m}^2.
\end{align}
Therefore, if we use the $L_2$ regularization,
we constrain the Frobenius norm of $\bm{A}$.
In addition, let $\bm{M}$ be $m\times m$ matrices, we have the following inequalities:
\begin{align}
\textstyle
||\bm{M}||_2\leq||\bm{M}||_F,\\
\textstyle
\frac{||\bm{M}||_\infty}{\sqrt{m}}\leq||\bm{M}||_2\leq\sqrt{m}||\bm{M}||_\infty,
\end{align}
where $||\cdot||_2$ is the induced 2-norm, which is the largest singular value.
From the above inequalities, we have $\frac{||\bm{A}||_\infty}{n}\leq||\bm{A}||_2\leq||\bm{A}||_F$,
and thus, 
 if we decrease the Frobenius norm of $\bm{A}$,
the induced 2-norm and $\infty$-norms are also decreased.
\section{Algorithm of Absum}
Algorithm \ref{alg} shows the whole training algorithm of Absum.
First, we update parameters by SGD (lines 3 and 4).
Next, we apply the proximal operator to each convolution filter (lines 5-13).
These processes are iteratively performed.
\begin{algorithm}[tb]
\caption{Training with Absum}
\label{alg}
\begin{algorithmic}[1]
\STATE Initialize parameters $\bm{\theta}$
\WHILE{$e\leq E$} 
\STATE Sample a minibatch $\{(\bm{X}_i,\bm{Y}_i)\}^B $
\STATE$\bm{\theta}\leftarrow \bm{\theta}-\eta \nabla_{\bm{\theta}}\frac{1}{B}\sum_b^{B} f(\bm{\theta}, \bm{X}_b,\bm{Y}_b)$\\
\FOR{$i\in{1,\dots, L}$}
\IF{$\sum_l^h\sum_m^h k_{l,m}^{(i)}<-h^2\eta\lambda$}
\STATE $\bar{\bm{k}}^{(i)}+\eta \lambda \bm{1}$
\ELSIF{$-h^2\eta\lambda\leq\sum_l^h\sum_m^h k_{l,m}^{(i)}\leq h^2\eta\lambda$}
\STATE $\bar{\bm{k}}^{(i)}-\frac{\sum_l^h\sum_m^h k_{l,m}^{(i)}}{h^2} \bm{1}$
\ELSE[$\sum_l^h\sum_m^h k_{l,m}^{(i)}>h^2\eta\lambda$]
\STATE $\bar{\bm{k}}^{(i)}-\eta \lambda \bm{1}$
\ENDIF
\ENDFOR
\STATE $e=e+ 1$
\ENDWHILE
\end{algorithmic}
\end{algorithm}
\section{Related Work}
Adversarial attacks are divided into two types; white-box and black-box attacks.
The fast gradient sign method (FGSM) and PGD are 
popular as simple and sophisticated white-box attacks, respectively \cite{fgsm,pgd,pgd2}.
Though many defense methods against white-box attacks have been proposed, e.g.,
defensive distillation \cite{distillation} and stochastic defense \cite{SAP},
several methods have been toppled by strong attacks \cite{best,candw}.
A promising method is adversarial training \cite{fgsm,pgd,pgd2},
which uses adversarial examples as training data.
However, its computational cost is larger than naive training. 
Note that Absum can be used with adversarial training and
enhances it, as discussed in experiments. 
Black-box attacks are more practical than white-box attacks
since it is difficult to access the target models in online applications 
\cite{papernot2017practical,yuan2019adversarial}. 
Most black-box attacks are transferred white-box attacks
and can be defended against by adversarial training \cite{papernot2017practical}.
Several black-box attacks use queries that ask the target model about the predicted labels of given input data,
but these attacks might still be impractical since they require a large amount of queries \cite{brendel2018decisionbased,chen2017zoo,pmlr-v80-ilyas18a}.
On the other hand, SFA only uses the information that the target model is composed of CNNs and is more practical.

%
An early study \cite{szegedy2013intriguing} showed that the induced norm can be a 
measure of robustness, and Parseval networks constrain
the induced norm of linear layers to improve robustness \cite{parseval}.
Parseval networks are more robust against FGSM than naive models
and can enhance adversarial training. 
However, the computational cost of Parseval networks is larger 
than standard regularization methods.
In addition, its robustness might be less than that of the spectral norm
regularization \cite{lmt} though Parseval networks penalize the spectral norm like the spectral norm
constraint. 
The spectral norm regularization can improve generalization performance \cite{yoshida2017spectral}.
However, the spectral norm in spectral norm regularization 
is often quite different from that of $\bm{A}$ \cite{gouk2018regularisation,sedghi2018the}
for convolution.

As simple regularization methods, \citet{srivastava2014dropout} shows that maxnorm regularization
can improve generalization performance of deep learning.
The maxnorm regularization in \cite{srivastava2014dropout} restricts the $L_2$ norm of weight vectors to 
be strictly less than or equal to a threshold $c$ as $||\bm{A}_i||_2\!\leq\!c$
where $\bm{A}_i$ is the $i$-th row vector of $\bm{A}$ in eq.~(5).
Therefore, the maxnorm regularization on convolution is 
$||\bm{A}_i||_2\!=\!\textstyle\sqrt{\sum_{l}\!\sum_{m} k^2_{l,m}}\!\leq\!c$ and
is similar to $L_2$ regularization.
In fact, we observed that the effectiveness of maxnorm regularization is similar to weight decay.
\section{Experimental Conditions}
\label{cond}
We had roughly two experimental conditions according to the dataset.
In all experiments, we selected the best regularization weight from among $[10^{1}, 10^{0},\dots, 10^{-6}, 10^{-7}]$ for Absum and standard regularization methods
and selected the best spectral norm $\sigma$ from among $[0.01, 0.1, 0.5, 1.0, 10]$
for spectral norm constraint (SNC) \cite{sedghi2018the}. 
In SNC, we clipped singular values once in 100 iterations due to the large computational cost.
Our experiments ran once for each hyperparameter.
We assumed that
all images were divided by 255 and pixels had the values between 0 and 1.
In addition, MNIST, CIFAR10 and CIFAR100 were standardized as (mean, standard deviation)=(0,1)
before the images were applied to the models as preprocessing.
In the evaluation of robustness, we standardized input images 
by using the means and standard deviations of clean data after adversarial perturbation. 
The computation graph of the standardization was preserved in gradient-based attacks;
thus, perturbations of PGD were optimized while considering this preprocess.
\subsection{MNIST and Fashion-MNIST}
\label{m cond}
The training set of each dataset contains 60,000 examples, and the test set contains 10,000 examples \cite{mnist,fmnist}.
For MNIST and Fashion-MNIST (FMNIST), we stacked two convolutional layers 
and two fully connected layers,
the first convolutional layer had the 10 output channels and the second convolutional
layer had 20 output channels.
The kernel sizes of the convolutional layers were 5, 
their strides were 1, and we did not use zero-padding in these layers.
After each convolutional layer, we applied max pooling (the stride was 2) and ReLU activation.
The output of the second convolutional layer was applied to the first fully connected layer
 (the size was $320\times 50$), and we used the ReLU activation after the first fully connected layer.
The size of the second fully connected layer was $50\times 10$, and we used softmax as the output function.
After the second convolution layer and before the second fully connected layer,
we applied 50 \% dropout.
We trained the model for 100 epochs by using Momentum SGD (the learning rate of 0.01
and momentum of 0.5).
We set the minibatch size to 64.

We changed $(l,m)$ to $\{(0,0)$, $(0,1)$, $\dots$, $(27,26)$, $(27,27)\}$ in SFA
since the size of the images was $28\times 28$
and evaluated the accuracy of the model on the test data perturbed by SFA.
The $L_\infty$ norm of the perturbation of SFA was set to $80/255$.
The perturbed inputs were clipped so that each element would be included in $[0, 1]$.
For fair comparison, all regularization methods were applied to only convolution filter parameters.
\subsection{CIFAR10, CIFAR100, and SVHN}
\label{c cond}
CIFAR10 and CIFAR100 contain 50,000 training images and 10,000 test images \cite{cifar}. 
SVHN contains 73,257 images for training and 26,032 images for testing \cite{svhn}.
For SVHN, we used cropped digits, which were cropped as $32\times 32$.
The model architecture was ResNet-18 for CIFAR10, CIFAR100, and SVHN \cite{resnet}.\footnote{Our
 training settings are based on the open source of https://github.com/kuangliu/pytorch-cifar.}
As the preprocessing for training, given images were randomly cropped as $32\times 32$ after
padding a sequence of four on each border of the images.
Horizontal flip was randomly applied to images with a probability of 0.5.
We trained the model for 350 epochs with Momentum SGD (momentum 0.9).
The initial learning rate was set to 0.1, and after 150 and 250 epochs, we divided
the learning rate by 10. We set the minibatch size to 128.

We changed $(l,m)$ in SFA to $\{(0,0), (0,1), \dots, (31,30), (31,31)\}$ 
since the size of the images was $32\times 32$ and 
evaluated the accuracy of the model on the test data perturbed by SFA.
The $L_\infty$ norm of the perturbation of SFA was set to $10/255$.
The perturbed inputs were clipped so that each element would be included in $[0, 1]$.
For fair comparison, all regularization methods were applied to only convolution filter parameters.

Note that about 20~\%
of SVHN test and train datasets have the class label of `1'. 
Due to the class imbalance, models output class `1' regardless of input images in some hyperparameter settings.
In this case, the robust accuracies are always about 20\%; thus, these models
sometimes outperform properly trained models with naive training in terms of robust accuracy.
However, these results are not meaningful, and we do not list them in the tables.
For the other datasets, we also do not list the results of the models that
output one class regardless of input images.
\subsection{High-Frequency Attack}
To evaluate robustness in the frequency domain, we used High-Frequency attacks.
High-Frequency attacks can be regarded as low-pass filteres, which remove high-frequency components.
In High-Frequency attacks \cite{highF}, we first apply discrete Fourier transform (DFT) $\mathcal{F}$ to data $\bm{X}$ as
\begin{align}
\bm{Z}=\mathcal{F}(\bm{X}).
\end{align}
Next, we decompose the low- and high-frequency components as
\begin{align}
Z^l_{i,j}=\begin{cases}
Z_{i,j}&\mathrm{if}~d((i,j),(c_i,c_j))\leq r\\
0&\mathrm{otherwise}
\end{cases},\\
Z^h_{i,j}=\begin{cases}
0&\mathrm{if}~d((i,j),(c_i,c_j))\leq r\\
Z_{i,j}&\mathrm{otherwise}
\end{cases},
\end{align}
where $Z^l_{i,j}$ and $Z^h_{i,j}$ are elements of low- and high-frequency components in
the frequency domain,
respectively, $(c_i,c_j)$ is a centroid of the image, $d(\cdot,\cdot)$ is the Euclidean distance,
and $r$ is a radius that determines the cutoff frequency.
Finally, we apply the inverse DFT to $\bm{Z}^l$ as
\begin{align}
\bm{X}^l=\mathcal{F}^{-1}(\bm{Z}^l),
\end{align}
and $\bm{X}^l$ is an input image attacked by High-Frequency attacks.
While $r$ is gradually reduced and accuracies are iteratively evaluated for each $r$
in \cite{highF}, we used fixed $r$ as half of the image width
since we just focus on comparing Absum with other methods.

In addition to High-Frequency attacks, we evaluated robust accuracies against high-pass filter
$\bm{X}^h\!=\!\mathcal{F}^{-1}(\bm{Z}^h)$.
Note that images processed by the high-pass filter are not adversarial examples
since it is difficult for humans to accurately classify these images.
Even so, this experiment reveals how the model trained using each method is
biased towards the low-frequency components.
\subsection{Computational Cost}
We evaluated the computation time of Absum.
We used one NVIDIA Tesla V100 GPU and 32 Intel(R) Xeon(R) Silver 4110 CPUs,
and our implementation used Python 3.6.8, pytorch 0.4.1, CUDA 9.0, and numpy 1.11.3
in this experiment.
Note that we used numpy to compute the FFT and singular value decomposition,
which is difficult to parallelize, in SNC. 
We clipped singular values once in 100 iterations due to the large computational cost.
The model architectures and training process were the same as those of the experiments involving SFA.
We used $\lambda=10^{-4}$ and $\sigma=1.0$.
We also conducted an experiment to evaluate the computational time when the input size increases.
We generated random images whose sizes
were $3\times32\times 32$, $3\times64\times64$, $3\times128\times128$, $3\times256\times256$, $3\times512\times512$, and $3\times1024\times1024$ with ten random labels,
and evaluated the computation time of the forward and backward processes of ResNet18 for one image.
\subsection{Robustness against PGD}
We also evaluated the effectiveness of Absum against PGD.
We evaluated Absum with adversarial training \cite{pgd2} in addition to naive training
because Absum and other regularization methods can be used with adversarial training.
In these experiments, we used advertorch \cite{ding2018advertorch} to generate adversarial examples
of PGD.

Model architectures and training conditions were almost the same as the experiments of SFA.
The number of epochs for MNIST and FMNIST was set to 100.
On the other hand, we observed overfitting in the adversarial training on CIFAR10, CIFAR100, and SVHN.
Therefore, we trained the model for 150 epochs with Momentum SGD (momentum 0.9).
The initial learning rate was set to 0.1, and after 50 and 100 epochs, we divided
the learning rate by 10.
We also applied weight decay of $10^{-4}$ to all parameters on CIFAR10 and CIFAR100 in the adversarial training of PGD since overfitting easily occurred in adversarial training on these datasets.

In PGD, the $L_\infty$ norm of the perturbation $\varepsilon$ was set to
 $\varepsilon=[0.05, 0.1, 0.15, 0.2, 0.25, 0.4]$ for MNIST and FMNIST, and $\varepsilon=[4/255, 8/255, 12/255, 
 16/255, 20/255]$ for CIFAR10 at evaluation time. 
For PGD, we updated the perturbation for 40 iterations with a step size of 0.01 on MNIST and FMNIST at training and evaluation times.
On CIFAR10, CIFAR100, and SVHN, we updated the perturbation for 7 
iterations with a step size of 2/255 at training time and 
100 iterations at evaluation time. 
The starting points of PGD were randomly initialized from a uniform distribution of [-2/255, 2/255].
For adversarial training, we used training data perturbed by
PGD with $\varepsilon=0.3$ on MNIST and $\varepsilon=8/255$ on CIFAR10, CIFAR100, and SVHN.
In adversarial training, we only used adversarial examples of training data.
The above conditions are based on \cite{pgd2}. 
\section{Additional Experimental Results}
\subsection{Robustness against SFA}
\begin{figure*}[tbp]
\centering
\scalebox{0.88}{
\captionsetup[subfloat]{farskip=0.8pt,captionskip=0.8pt}
\subfloat[Absum]{\includegraphics[width=0.25\linewidth ,clip]{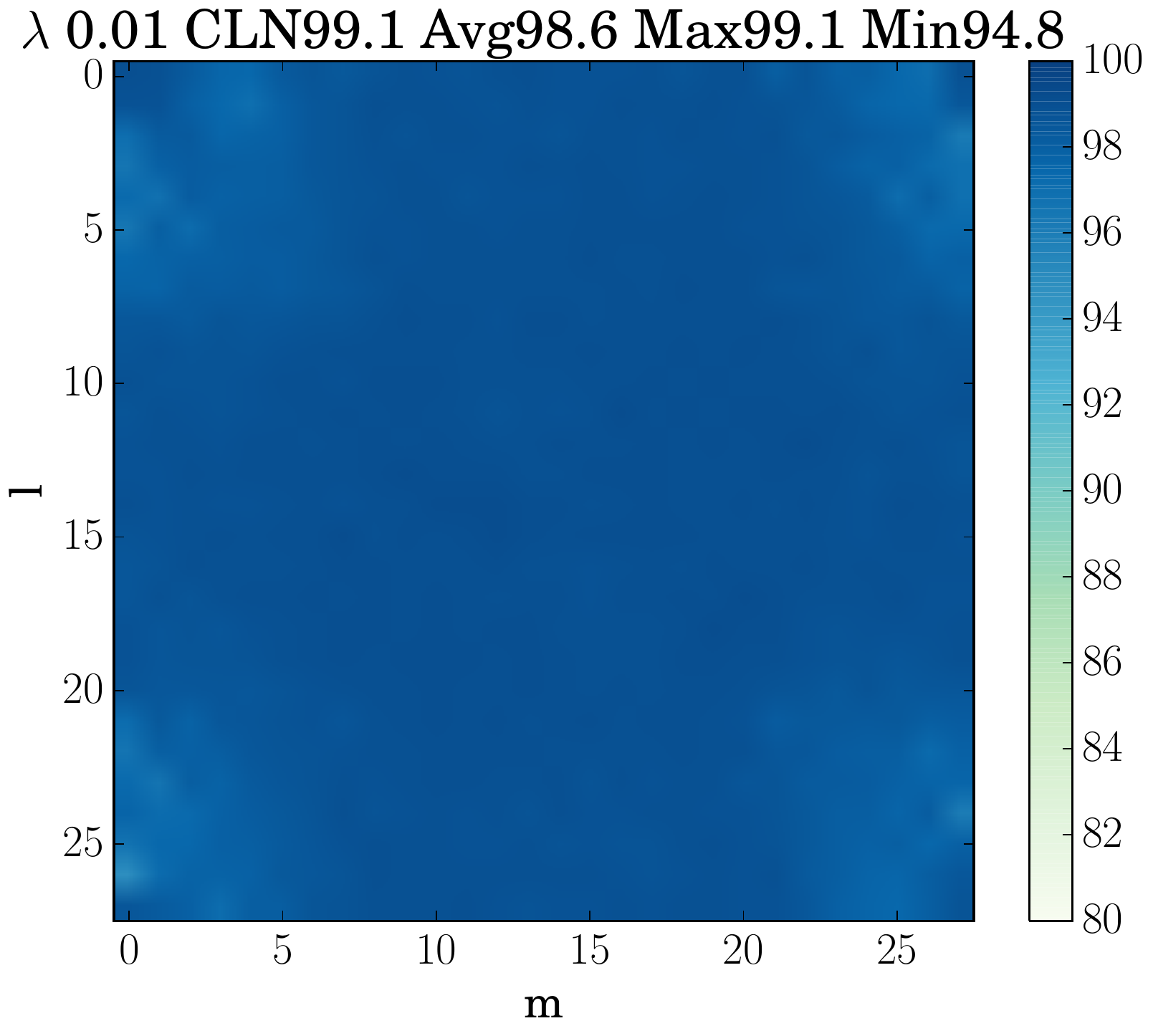}}
\subfloat[WD]{\includegraphics[width=0.25\linewidth ,clip]{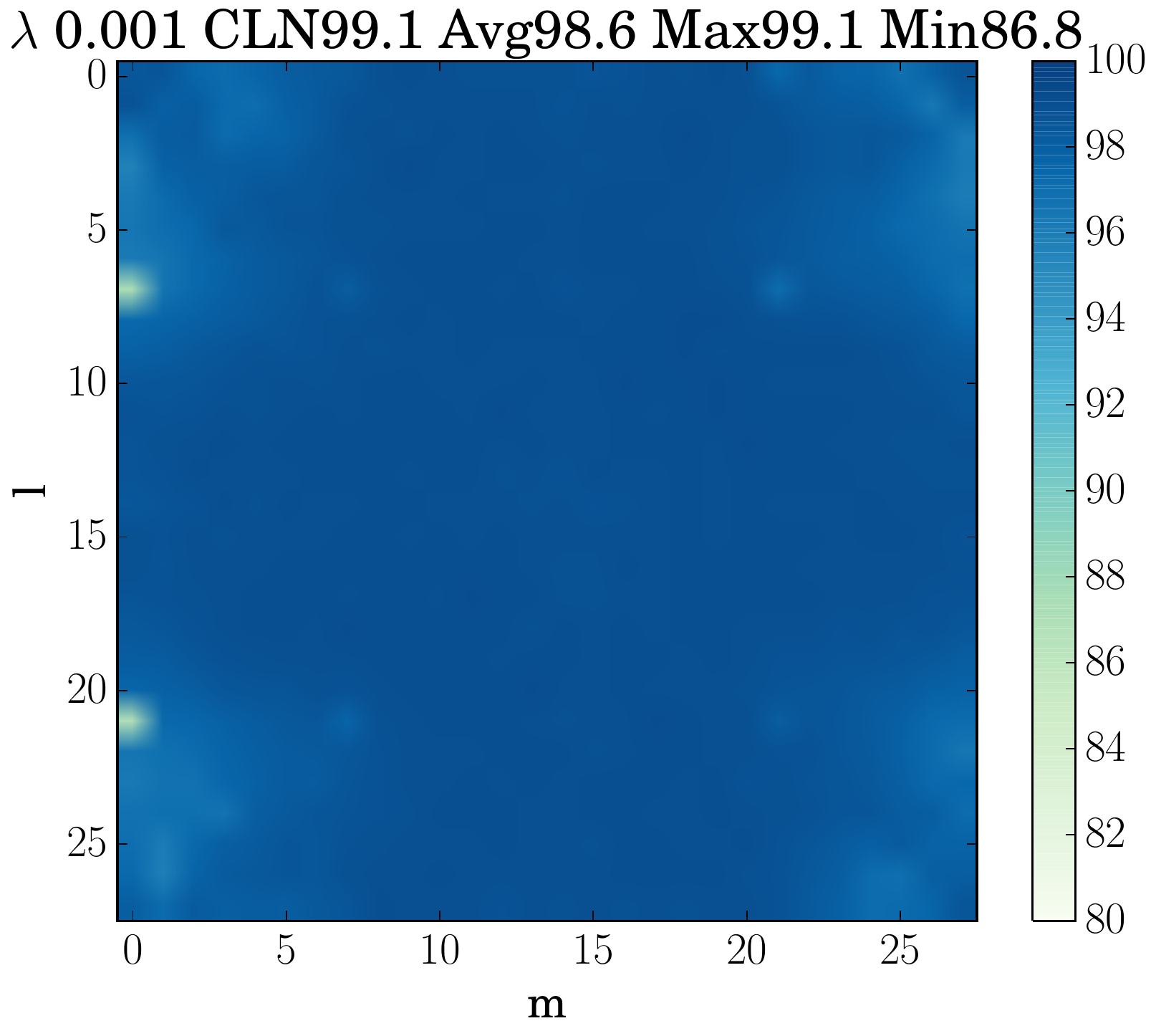}}
\subfloat[$L_1$]{\includegraphics[width=0.25\linewidth ,clip]{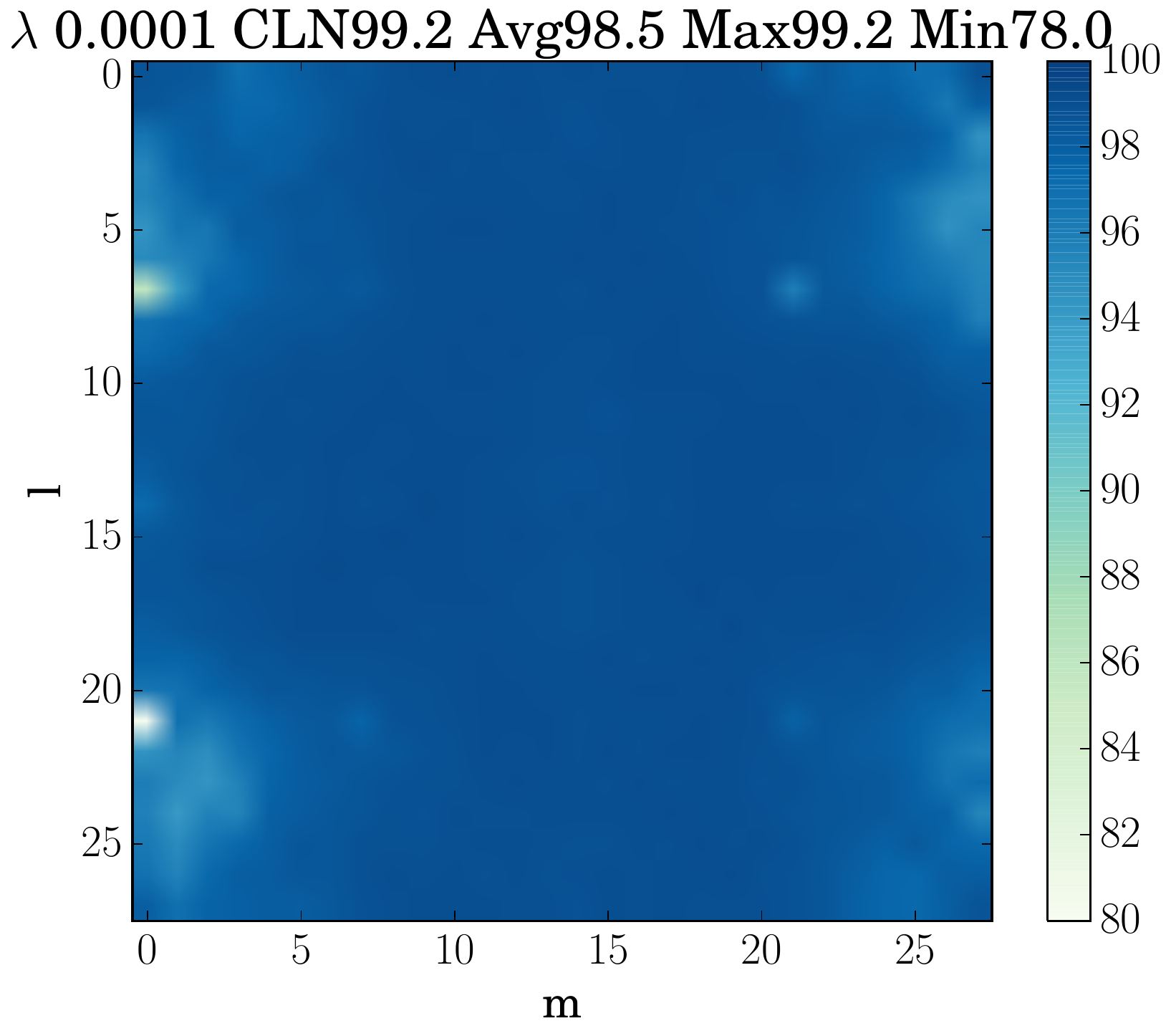}}
\subfloat[SNC]{\includegraphics[width=0.25\linewidth ,clip]{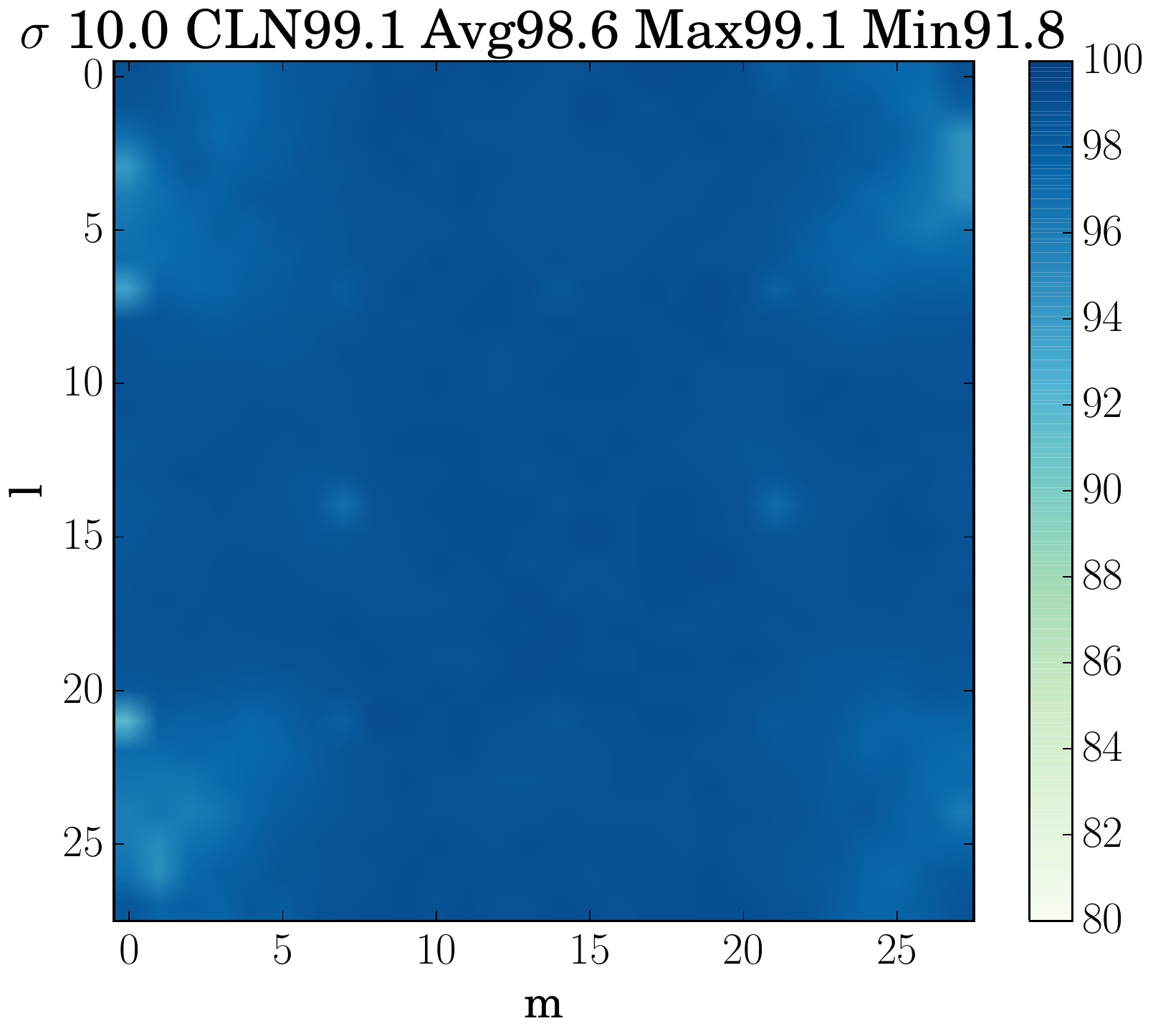}}
}\\
\centering
\scalebox{0.88}{
\subfloat[Absum]{\includegraphics[width=0.25\linewidth ,clip]{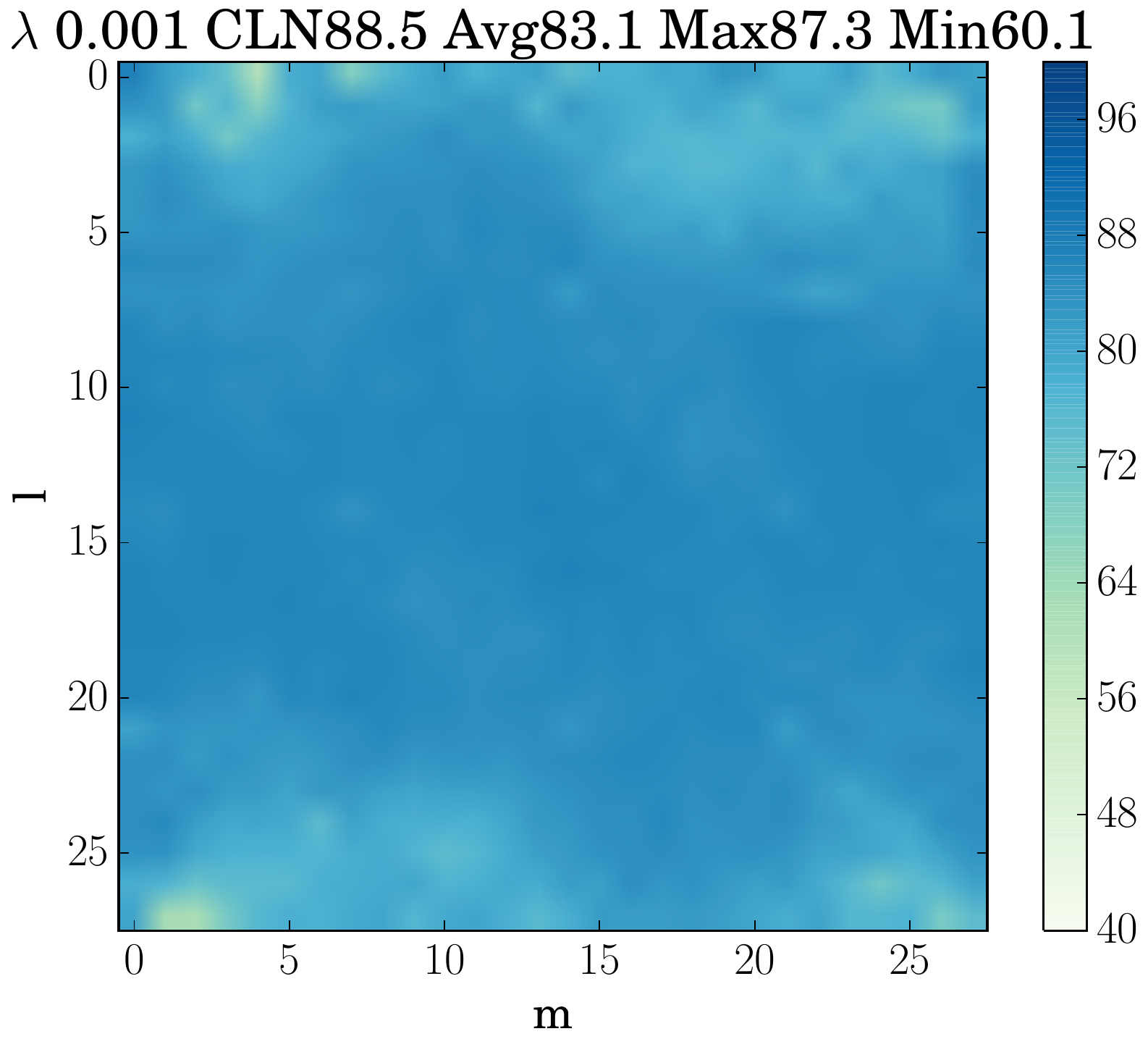}}
\subfloat[WD]{\includegraphics[width=0.25\linewidth ,clip]{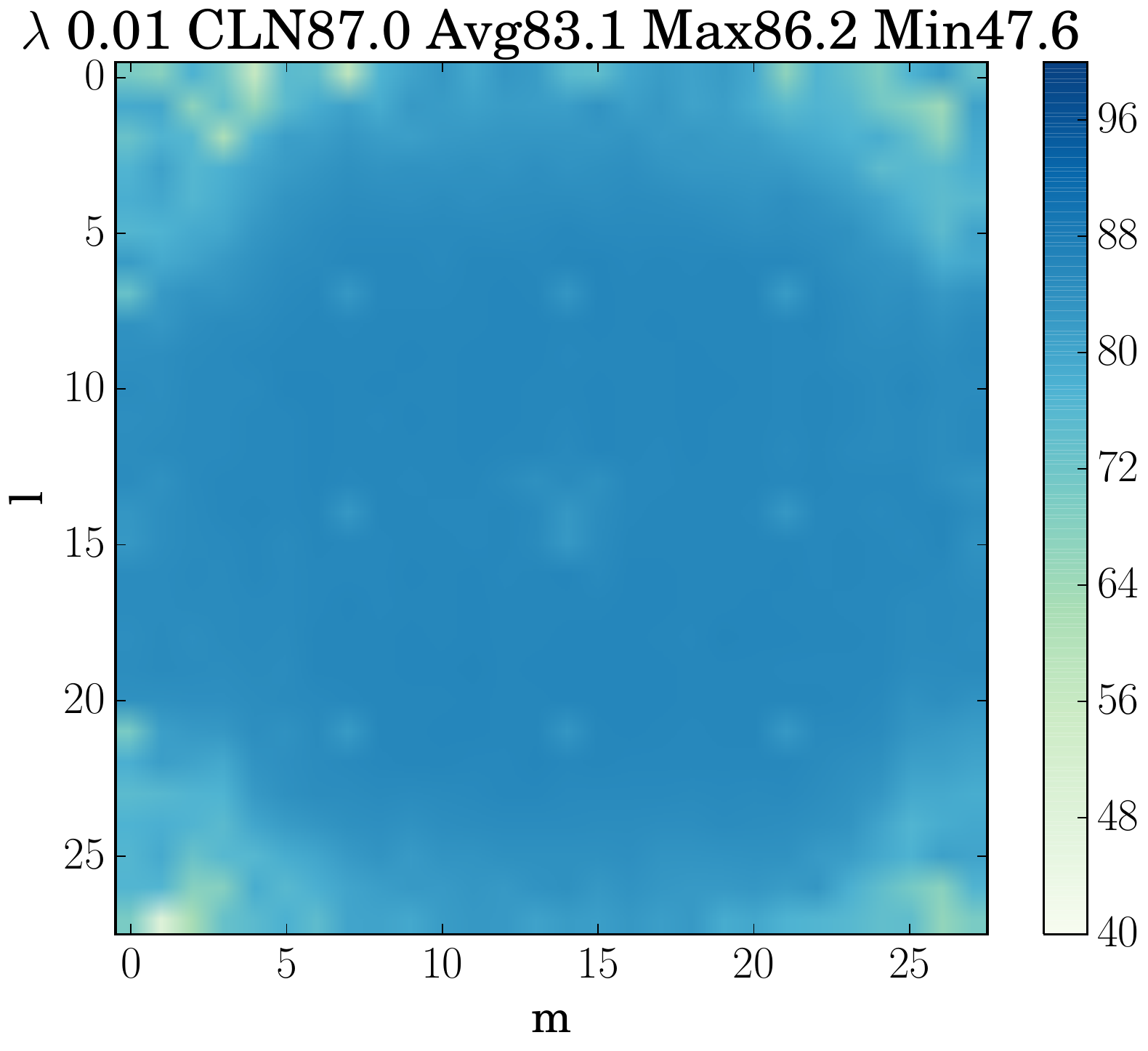}}
\subfloat[$L_1$]{\includegraphics[width=0.25\linewidth ,clip]{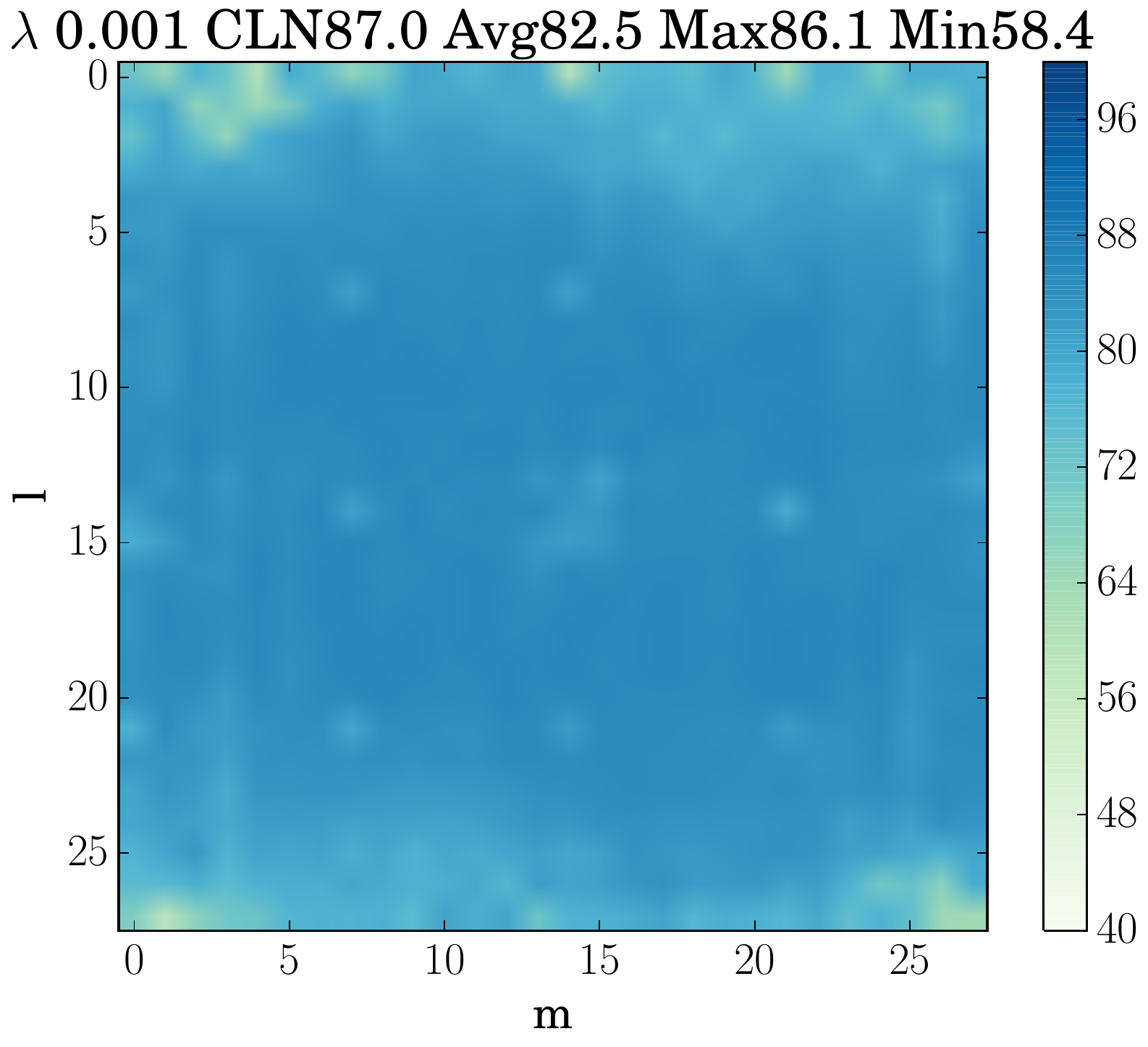}}
\subfloat[SNC]{\includegraphics[width=0.25\linewidth ,clip]{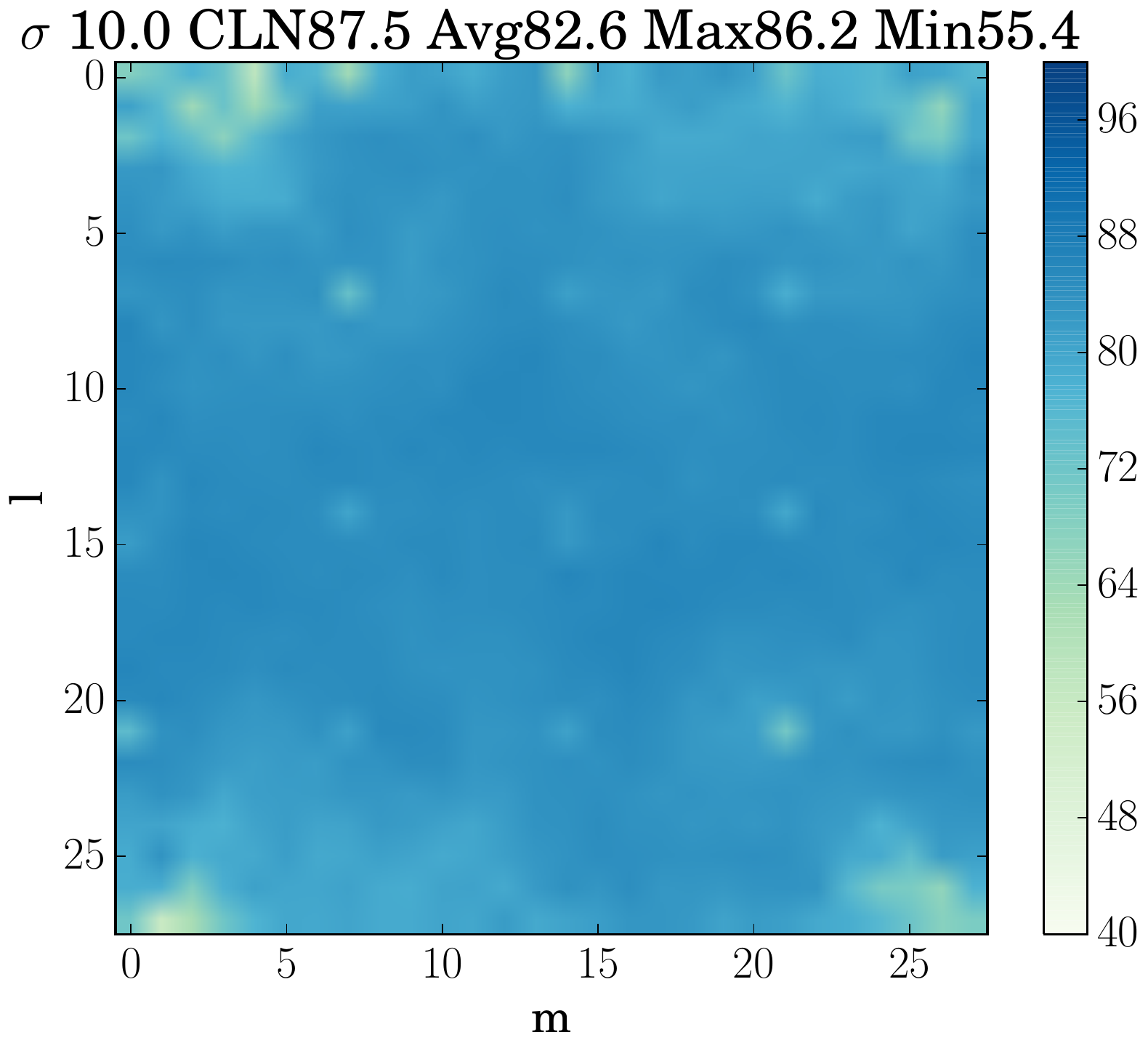}}
}
\\
\centering
\scalebox{0.88}{
\subfloat[Absum]{\includegraphics[width=0.25\linewidth ,clip]{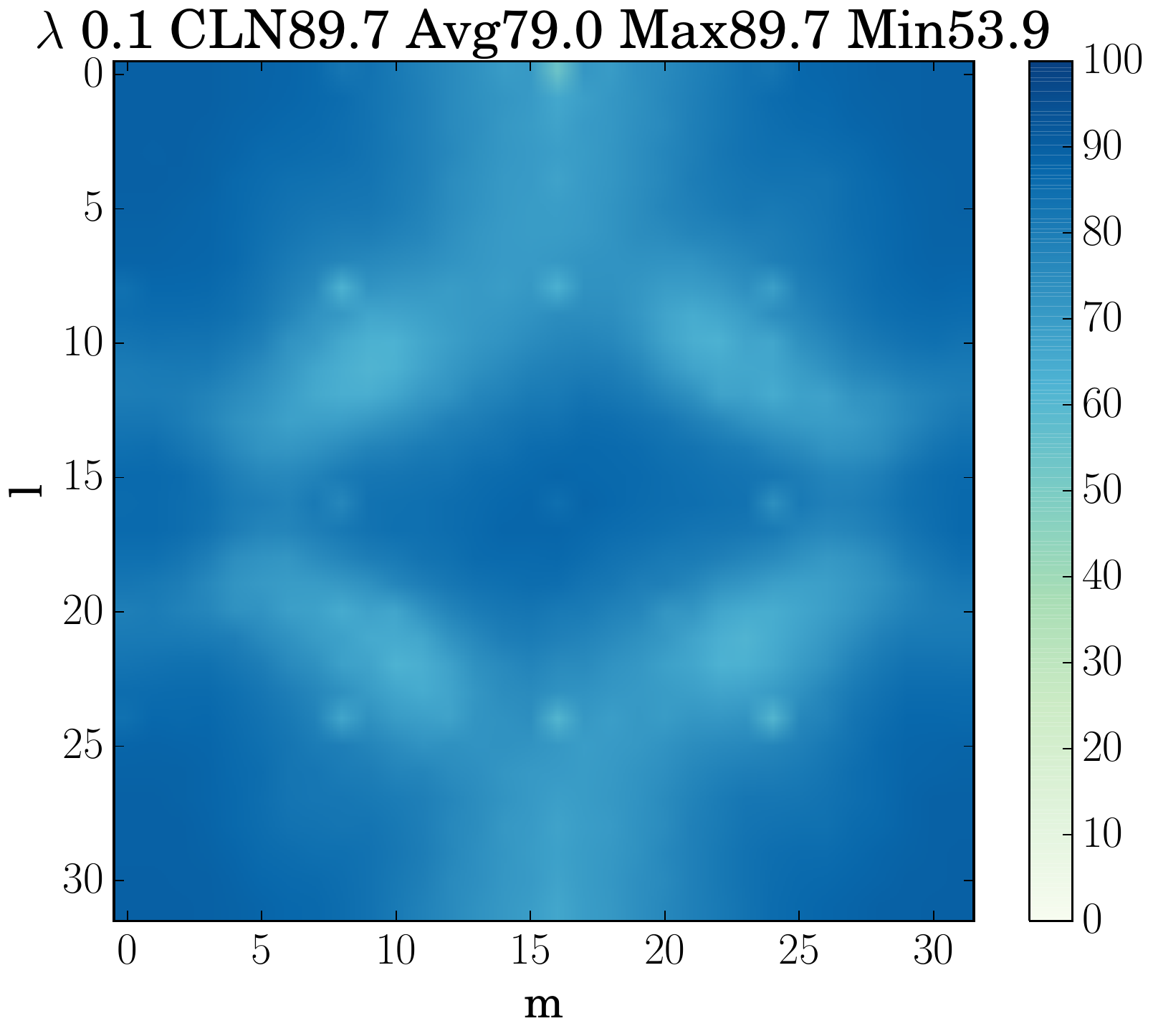}}
\subfloat[WD]{\includegraphics[width=0.25\linewidth ,clip]{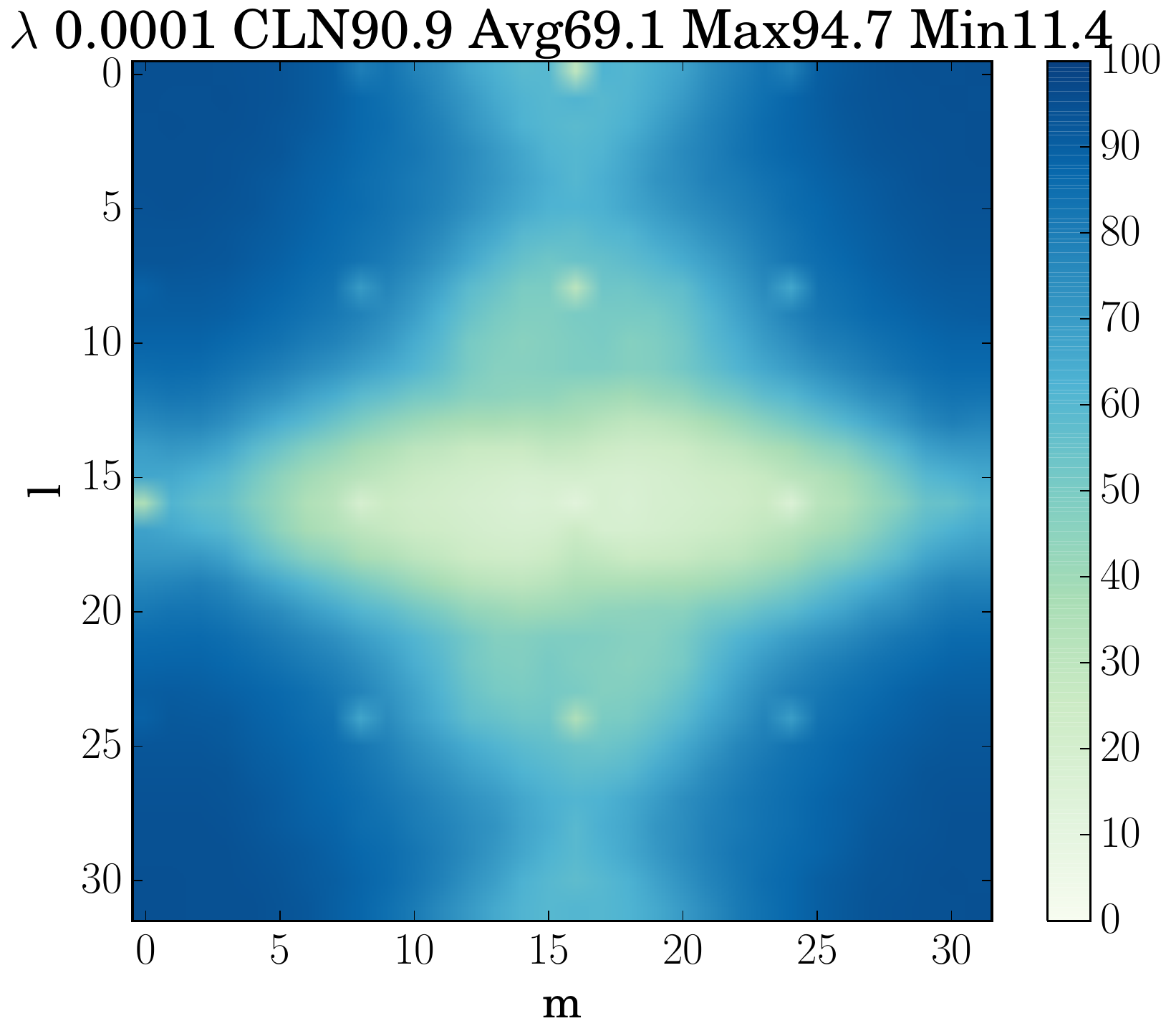}}
\subfloat[$L_1$]{\includegraphics[width=0.25\linewidth ,clip]{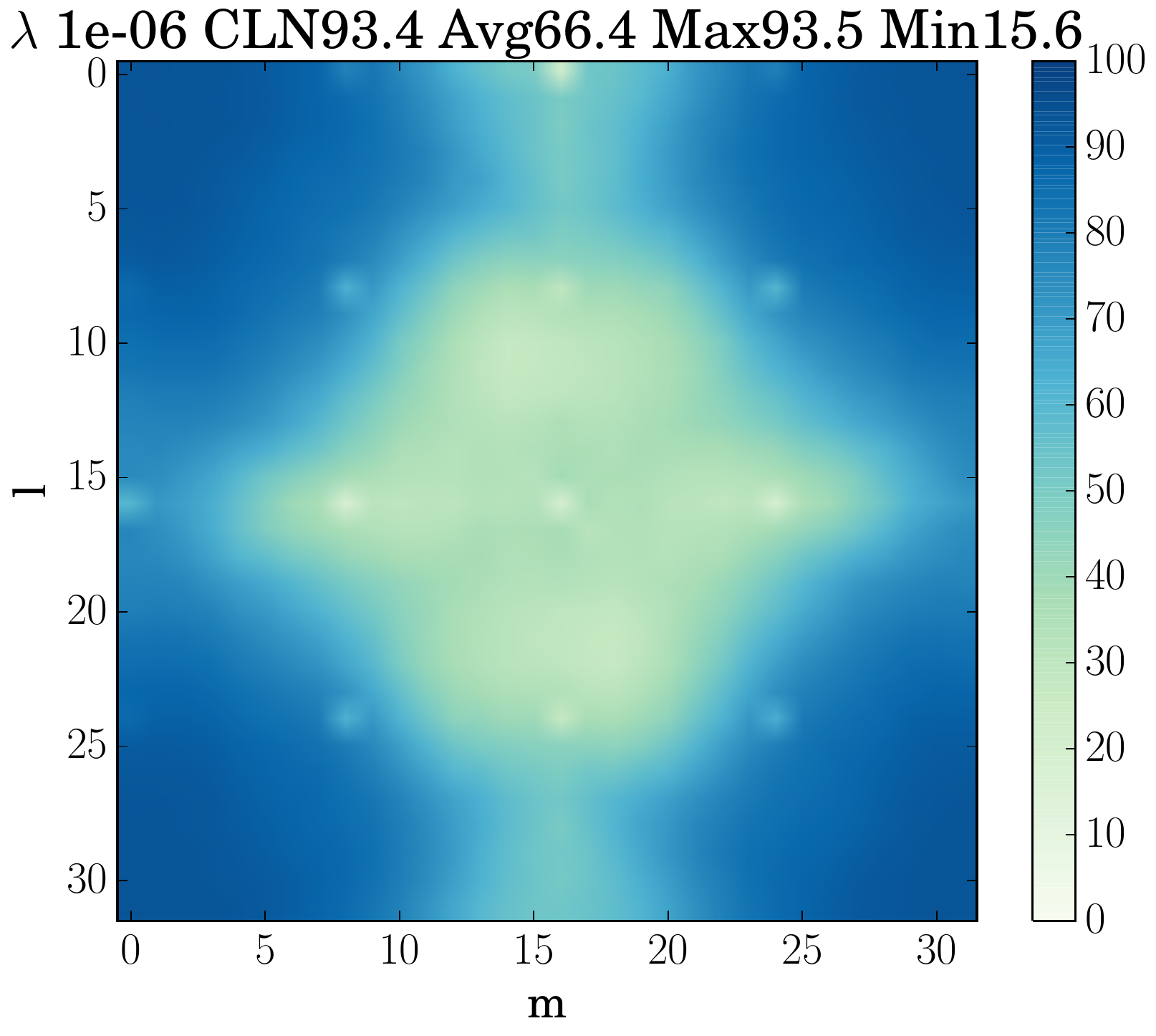}}
\subfloat[SNC]{\includegraphics[width=0.25\linewidth ,clip]{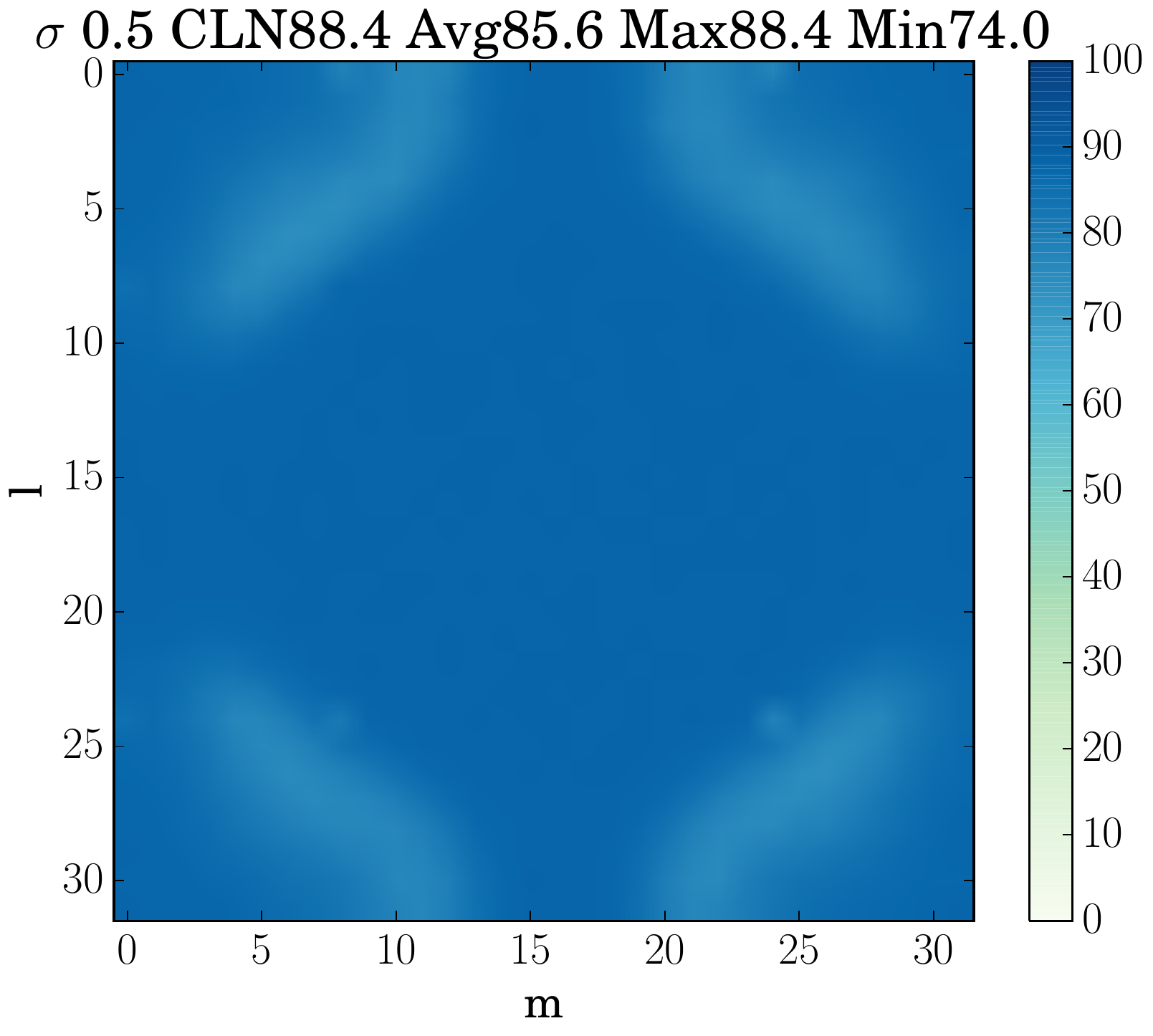}}
}
\\
\centering
\scalebox{0.88}{
\subfloat[Absum]{\includegraphics[width=0.25\linewidth ,clip]{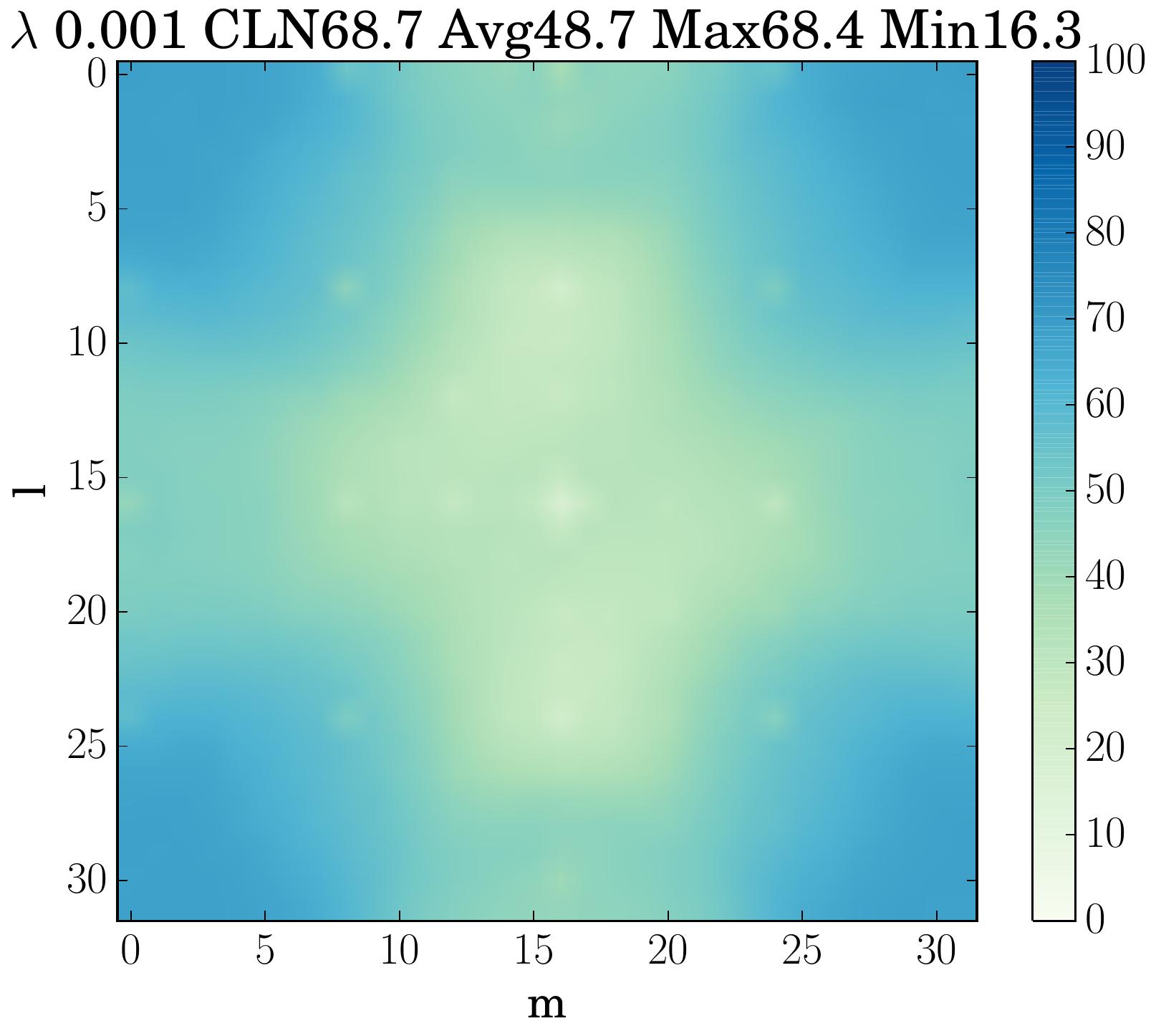}}
\subfloat[WD]{\includegraphics[width=0.25\linewidth ,clip]{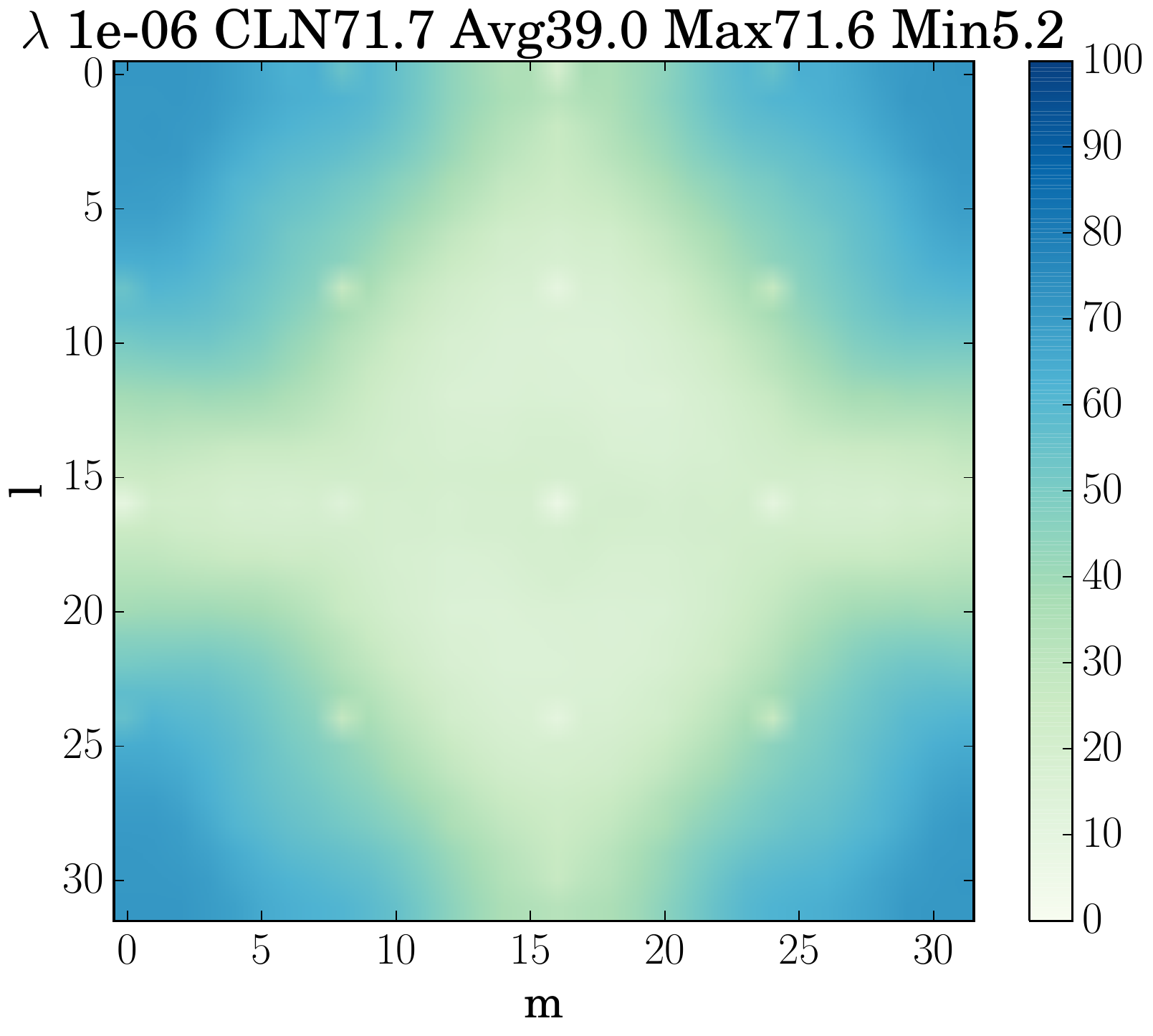}}
\subfloat[$L_1$]{\includegraphics[width=0.25\linewidth ,clip]{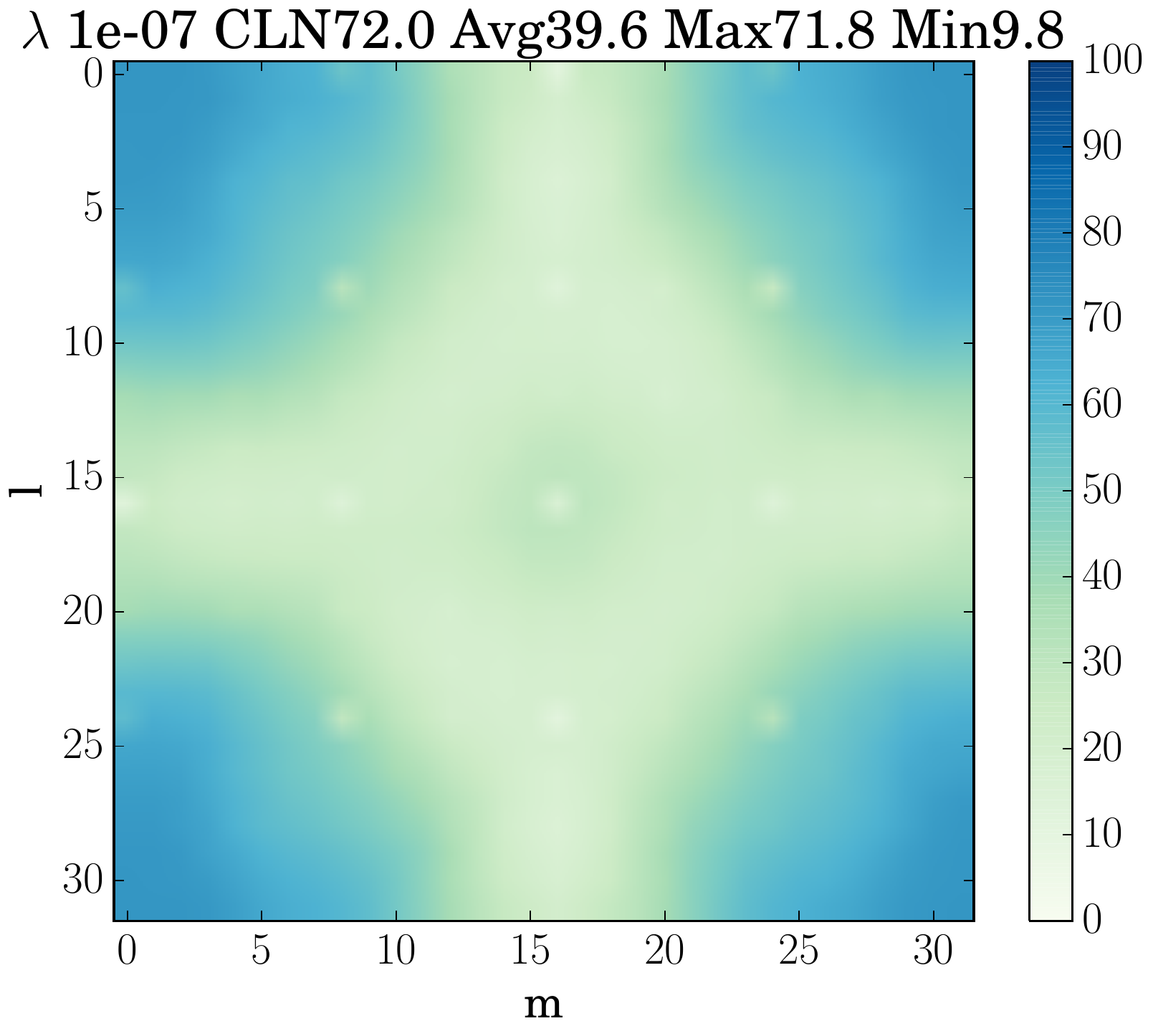}}
\subfloat[SNC]{\includegraphics[width=0.25\linewidth ,clip]{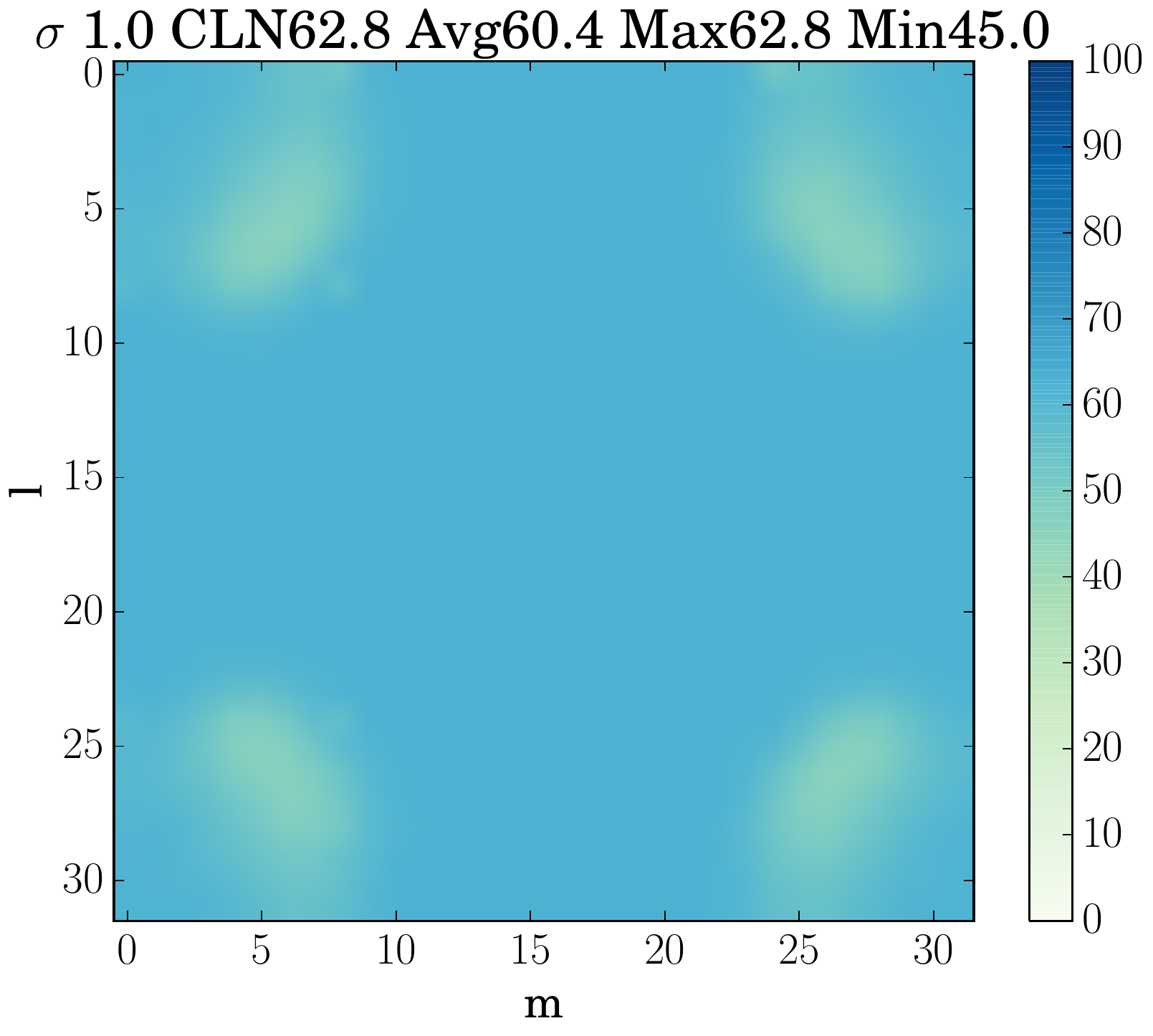}}
}
\\
\centering
\scalebox{0.88}{
\subfloat[Absum]{\includegraphics[width=0.25\linewidth ,clip]{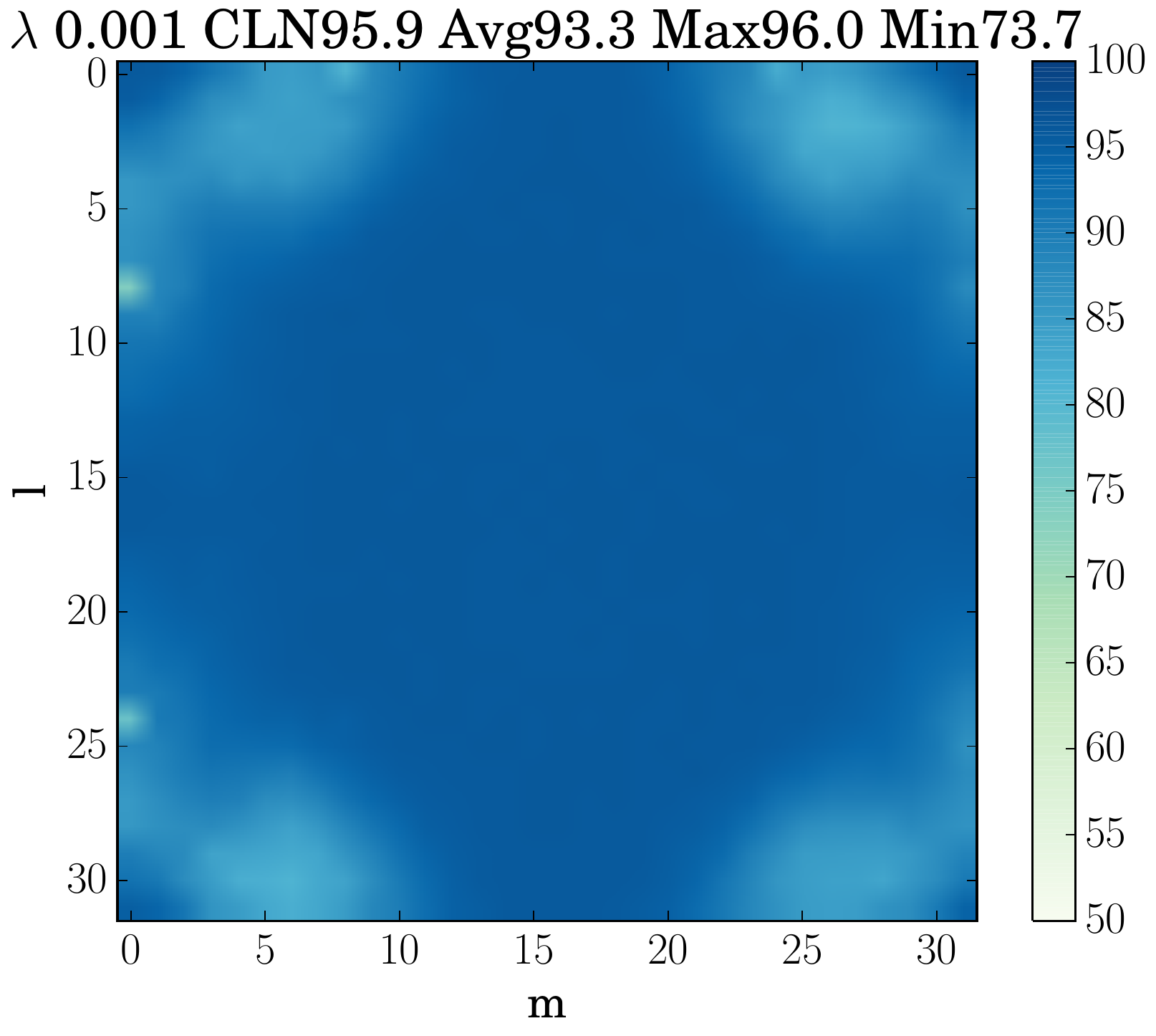}}
\subfloat[WD]{\includegraphics[width=0.25\linewidth ,clip]{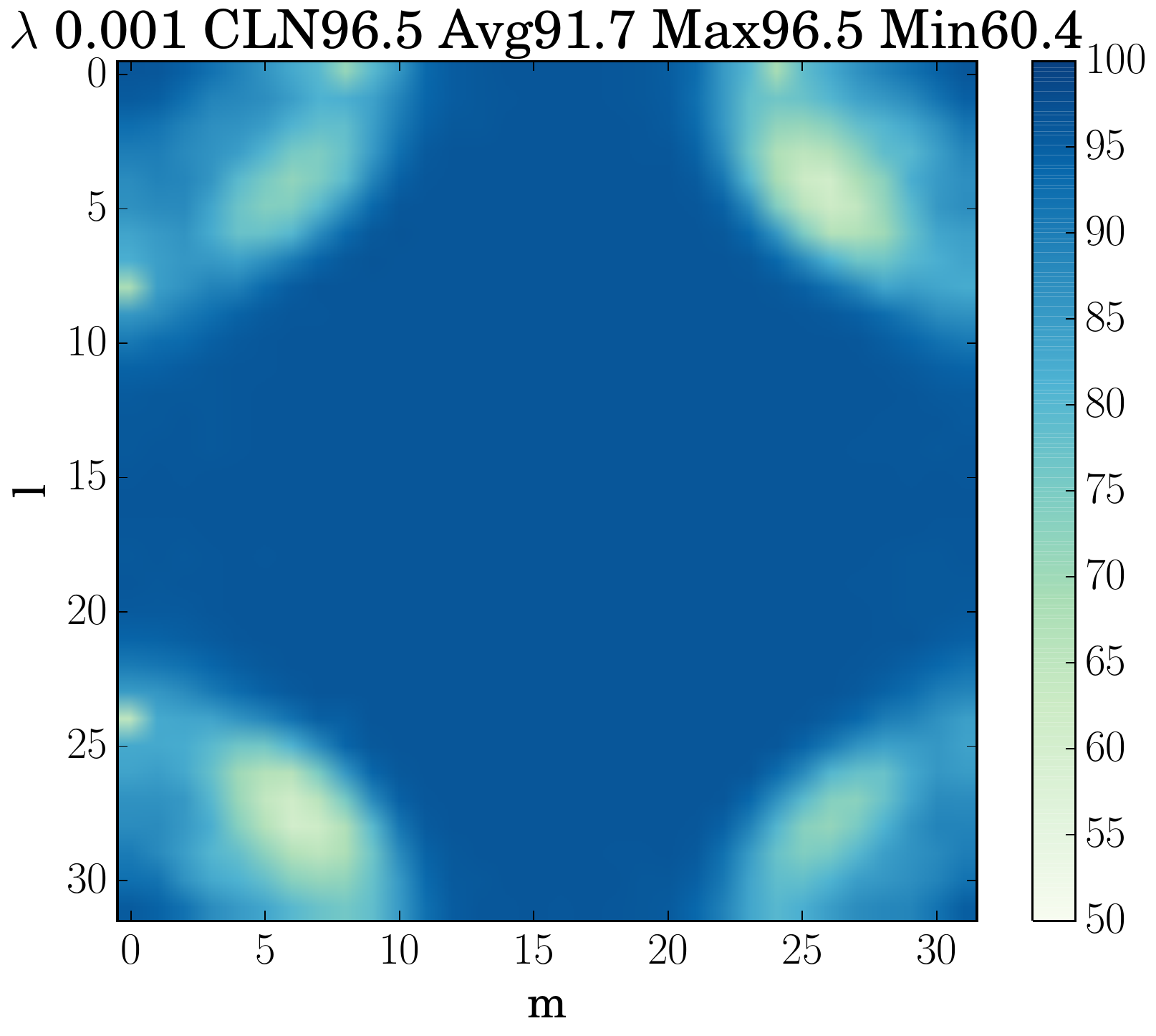}}
\subfloat[$L_1$]{\includegraphics[width=0.25\linewidth ,clip]{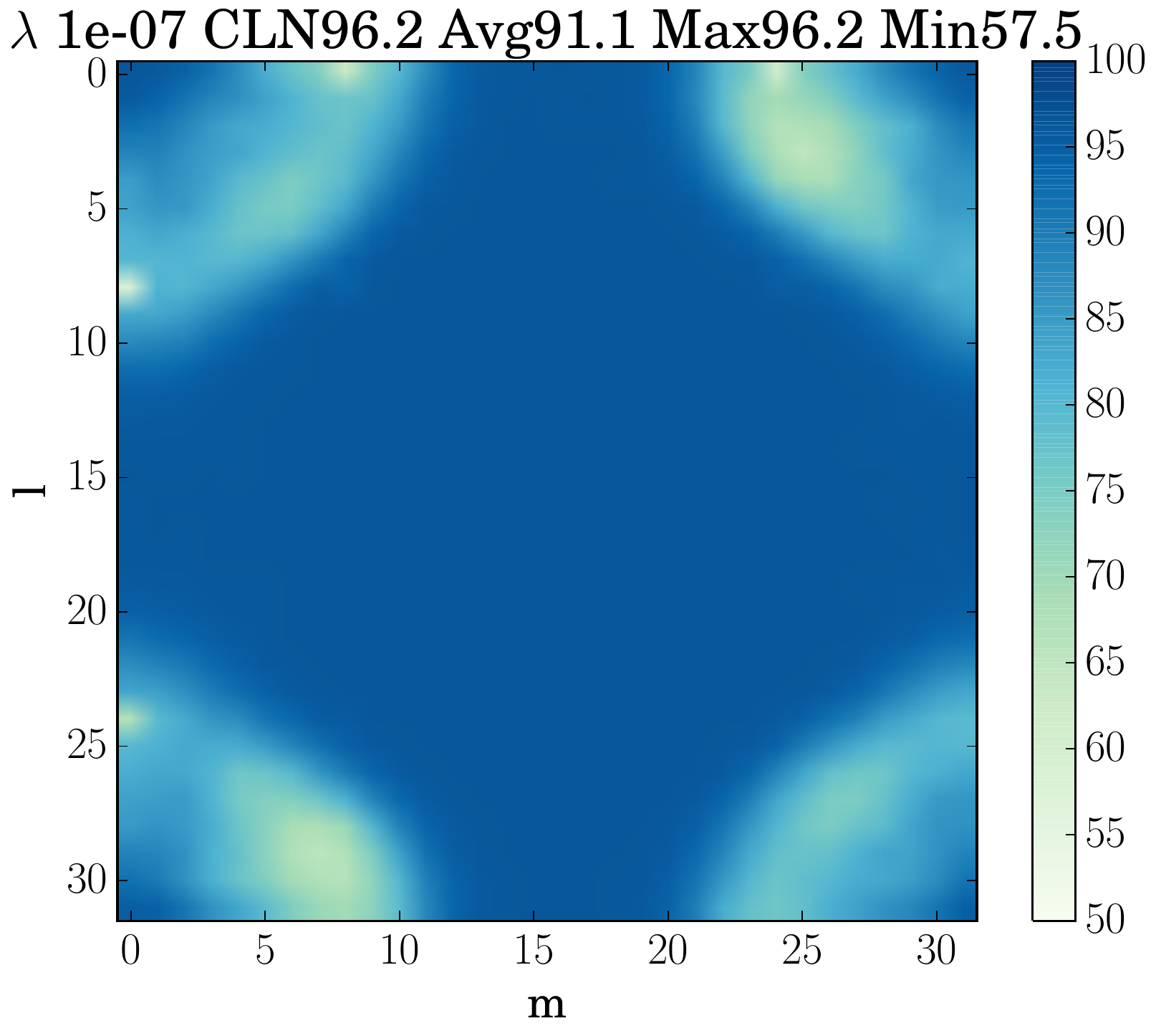}}
\subfloat[SNC]{\includegraphics[width=0.25\linewidth ,clip]{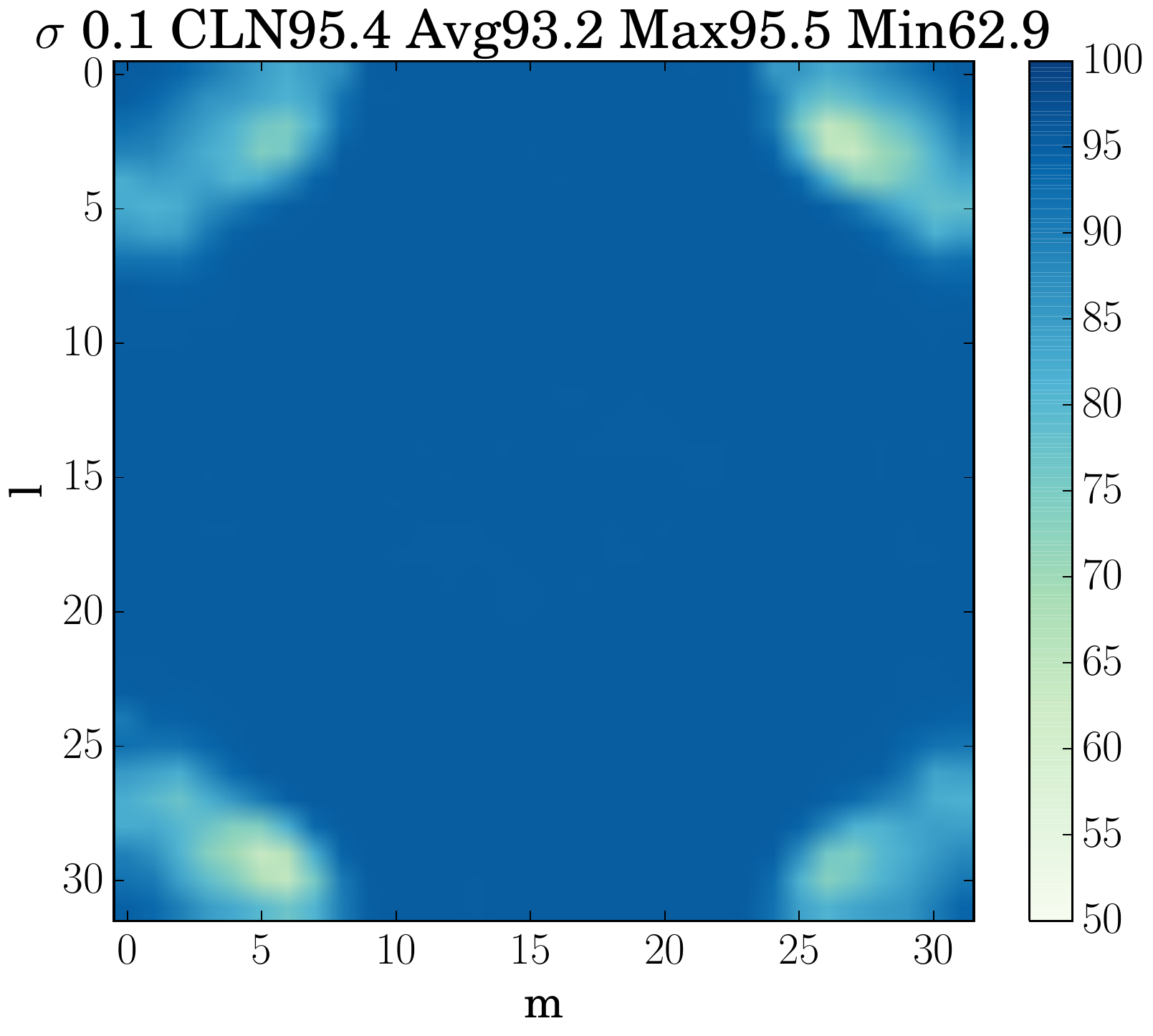}}
}
\caption{Accuracies of methods on MNIST (a)-(d), FMNIST (e)-(h), CIFAR10 (i)-(l), CIFAR100 (m)-(p), and 
SVHN (q)-(t) perturbed by SFA $(l,m)$. CLN is 
accuracy on clean data. $\lambda$ and $\sigma$ are selected so that average accuracies (Avg) against SFA would achieve largest values.}\label{accSFA}
\end{figure*}
\begin{figure}[tb]
\centering
\subfloat[FMNIST]{\includegraphics[width=0.9\linewidth ,clip]{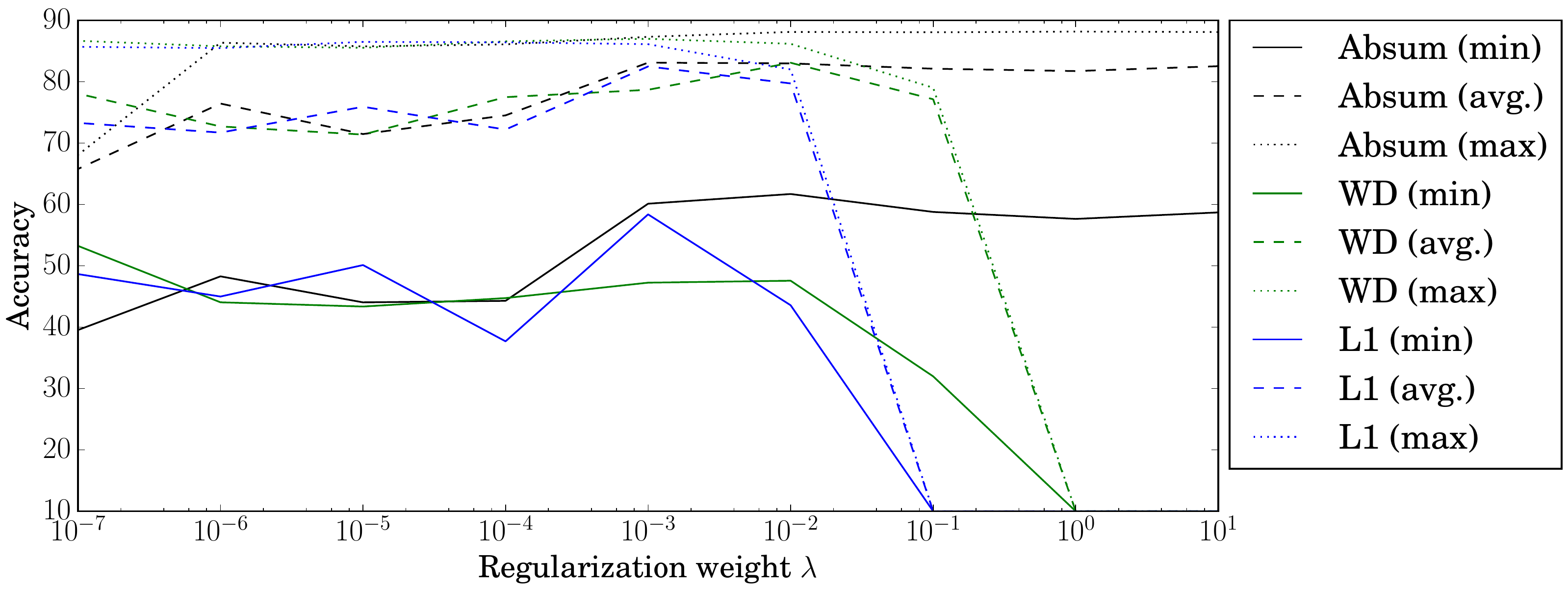}}
\\
\centering
\subfloat[CIFAR100]{\includegraphics[width=0.9\linewidth ,clip]{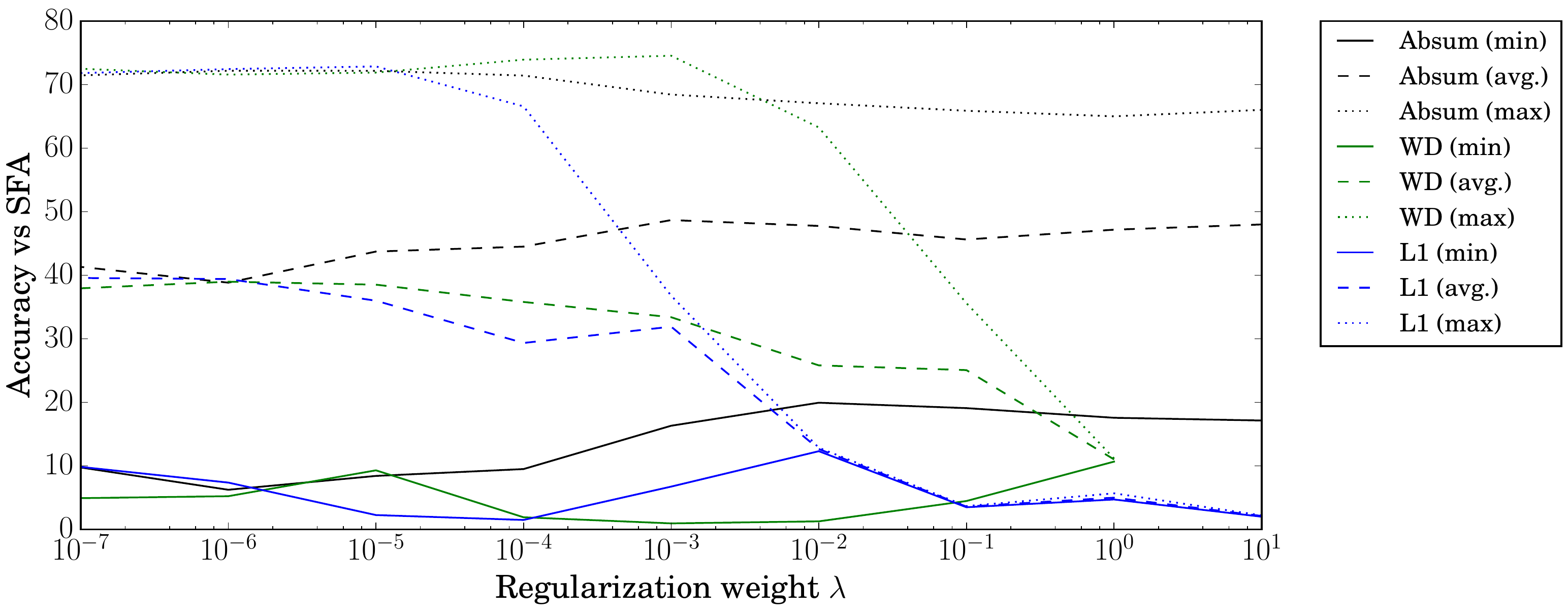}}
\\
\centering
\subfloat[SVHN]{\includegraphics[width=0.9\linewidth ,clip]{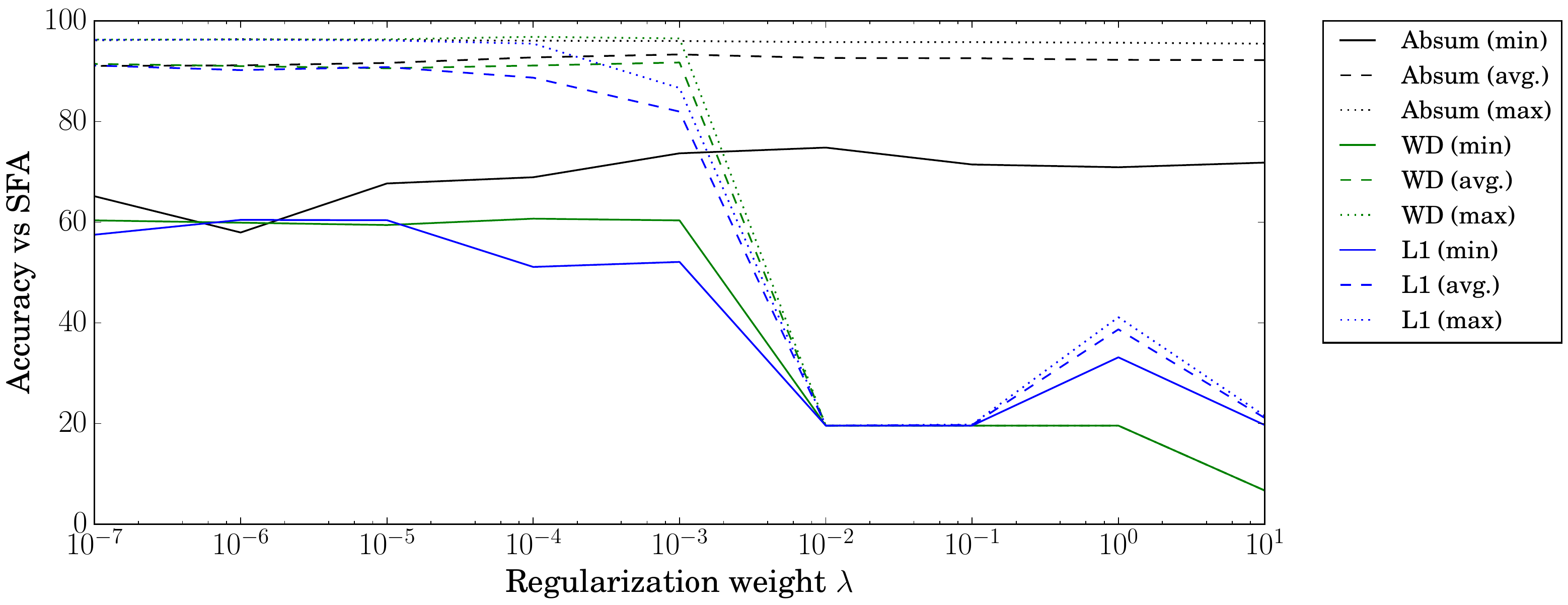}}
\caption{Accuracies of methods on dataset perturbed by SFA vs regularization weight}
\label{AccVsLam2}
\end{figure}
\begin{table*}[tbp]
\centering
\caption{Accuracies of each method on datasets perturbed by SFA $(l,m)$. $\lambda$ and $\sigma$ are selected so that each accuracy would be highest.}
\label{sfatab2}
\scalebox{0.8}{
\begin{tabular}{lrrrrcrrrrcrrrrcrrrrc}
\toprule
& \multicolumn{5}{c}{Avg.}&\multicolumn{5}{c}{Min.} &\multicolumn{5}{c}{CLN}&\\
\cmidrule(l){2-6}\cmidrule(l){7-11}\cmidrule(l){12-16}\cmidrule(l){17-20}
{} &      Absum &         WD &        L1 &        SNC&{\small w/o Reg.} &      Absum &         WD &         L1 &       SNC &{\small w/o Reg.}&      Absum &         WD &         L1 &        SNC&{\small w/o Reg.} \\
\midrule
MNIST &  \bf{98.64} &  98.59 &  98.48 &  98.55&98.44 &  \bf{95.25} &  94.65 &  91.61 &  91.79&80.53 &  \bf{99.23} &  99.21 &  99.19 &  99.10&99.18\\
FMNIST &  \bf{83.11} &  83.09 &  82.49 &  82.60 &72.75&  \bf{60.12} &  53.25 &  58.38 &  55.36 &42.92&  89.25 &  89.20 &  89.27 &  87.50&\bf{89.37} \\
CIFAR10 &  79.05 &  69.09 &  66.44 &  \bf{85.57}&66.64 &   53.90 &  20.80 &  29.82 &  \bf{73.99}&11.74 &  93.87 &  \bf{94.73} &  93.78 &  93.51&93.53\\
CIFAR100 &  48.69 &  42.97 &  38.99 &  \bf{60.42}&39.52 &  19.94 &   10.67 &  12.31 &  \bf{45.58}&8.89 &  72.38 &  \bf{74.63}  &  73.02 &  71.51&71.93\\
SVHN &  \bf{93.34} &  91.74 &  91.14&  93.20&90.72 &  \bf{74.83} &  70.77 &  60.70&  70.66&58.48&  96.27 &  \bf{96.72} &  96.20 &  96.15&96.17\\
\bottomrule
\end{tabular}
}
\end{table*}
\begin{figure}[tb]
\centering
\subfloat[MNIST]{\includegraphics[width=0.8\linewidth ,clip]{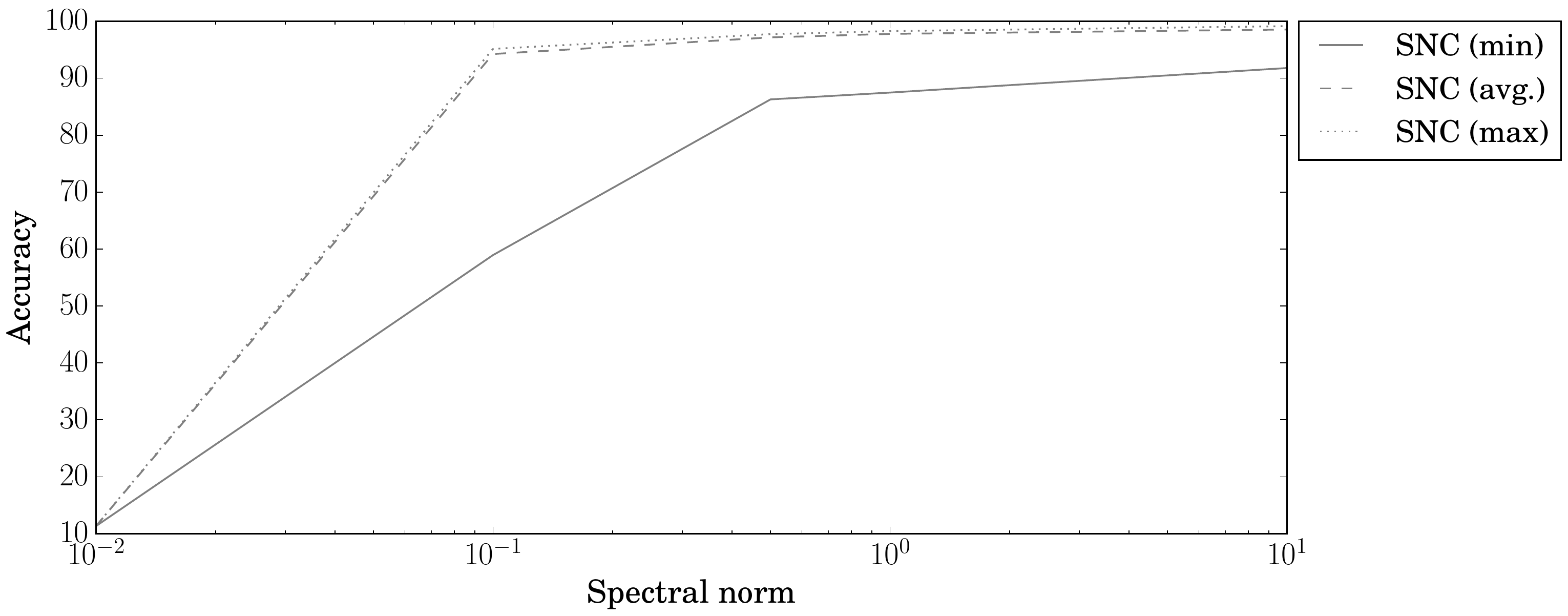}}\\
\centering
\subfloat[FMNIST]{\includegraphics[width=0.8\linewidth ,clip]{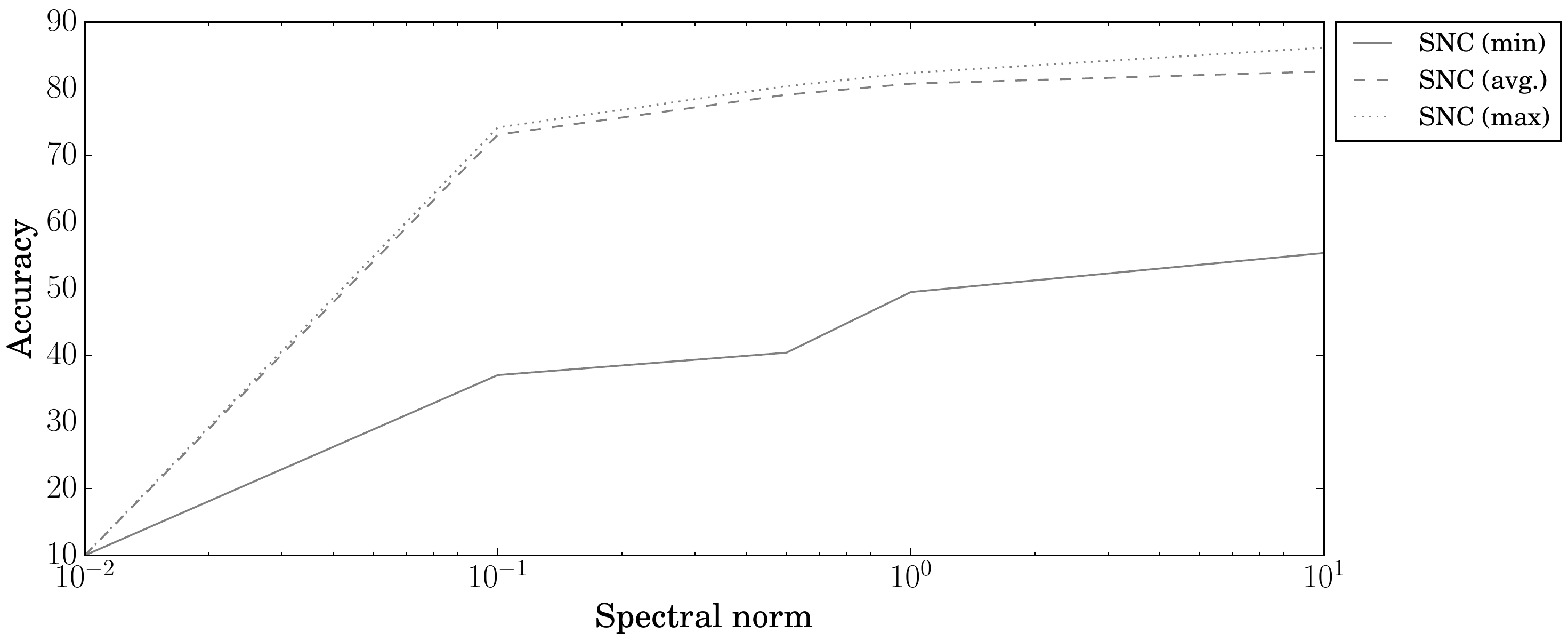}}
\\
\centering
\subfloat[CIFAR10]{\includegraphics[width=0.8\linewidth ,clip]{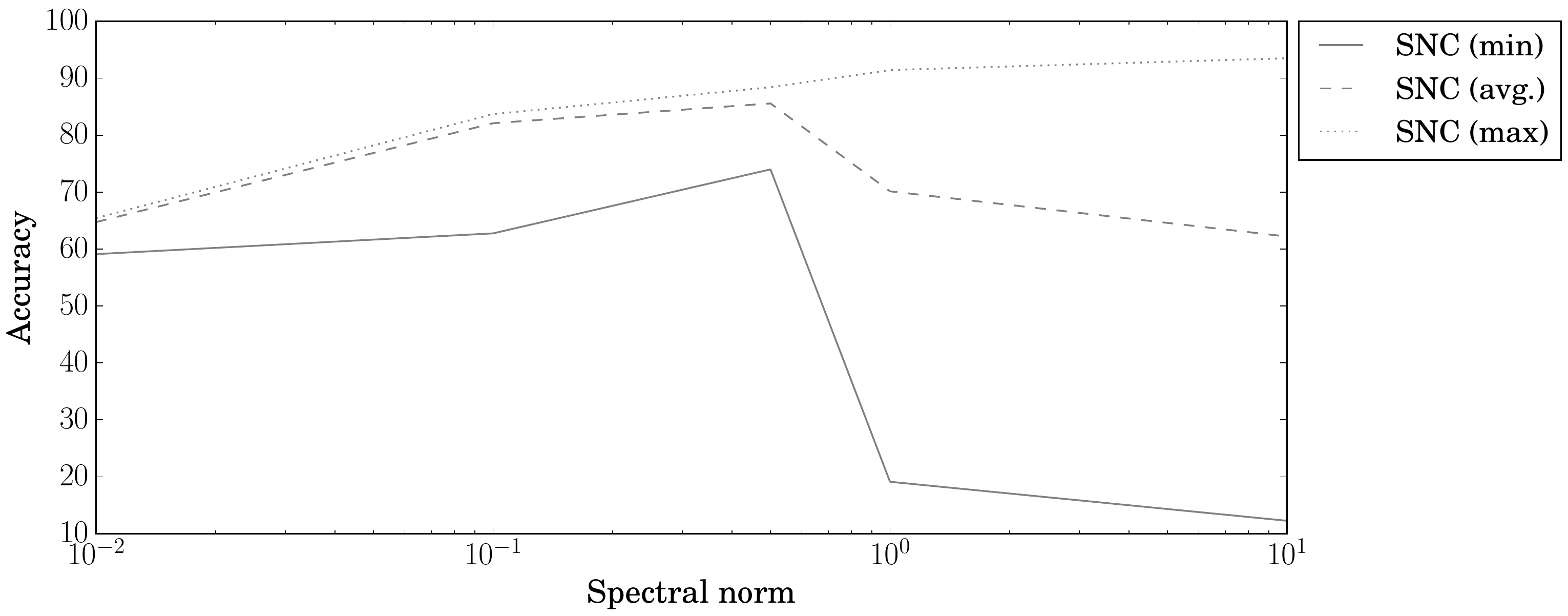}}
\\
\centering
\subfloat[CIFAR100]{\includegraphics[width=0.8\linewidth ,clip]{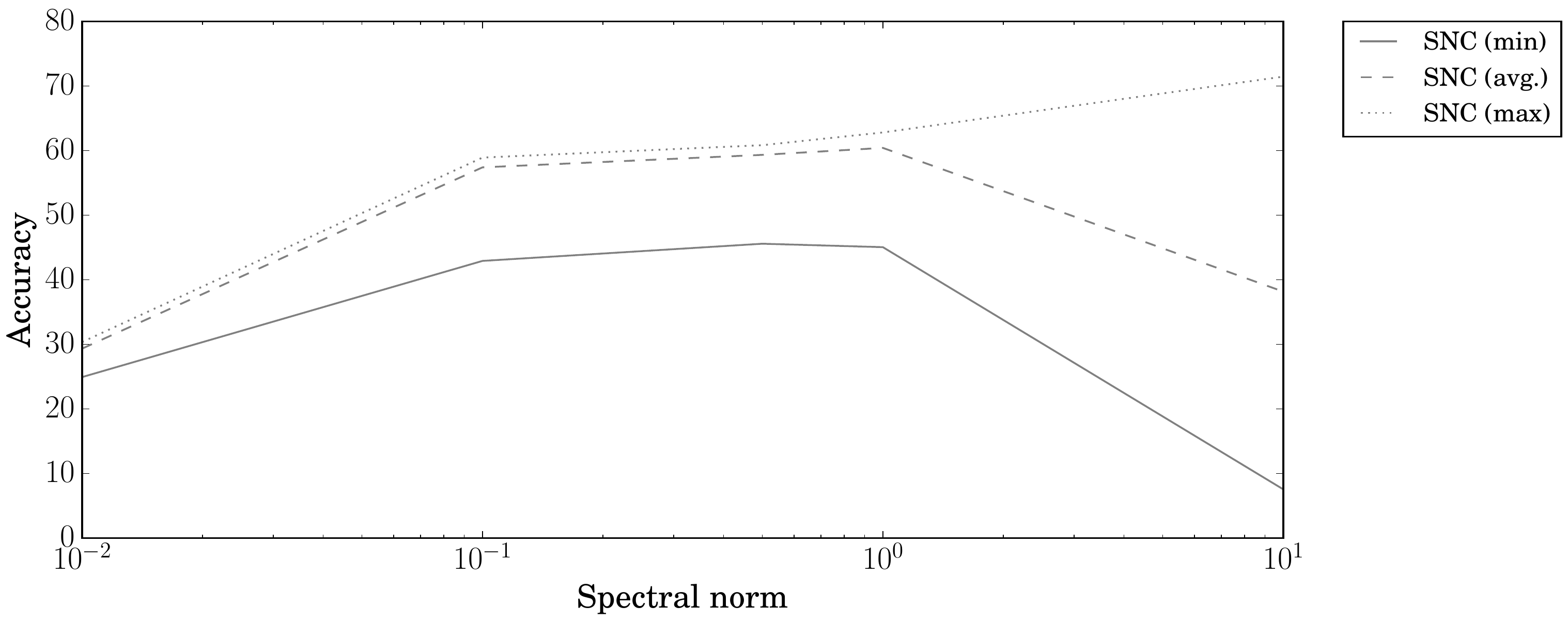}}
\centering
\\
\subfloat[SVHN]{\includegraphics[width=0.8\linewidth ,clip]{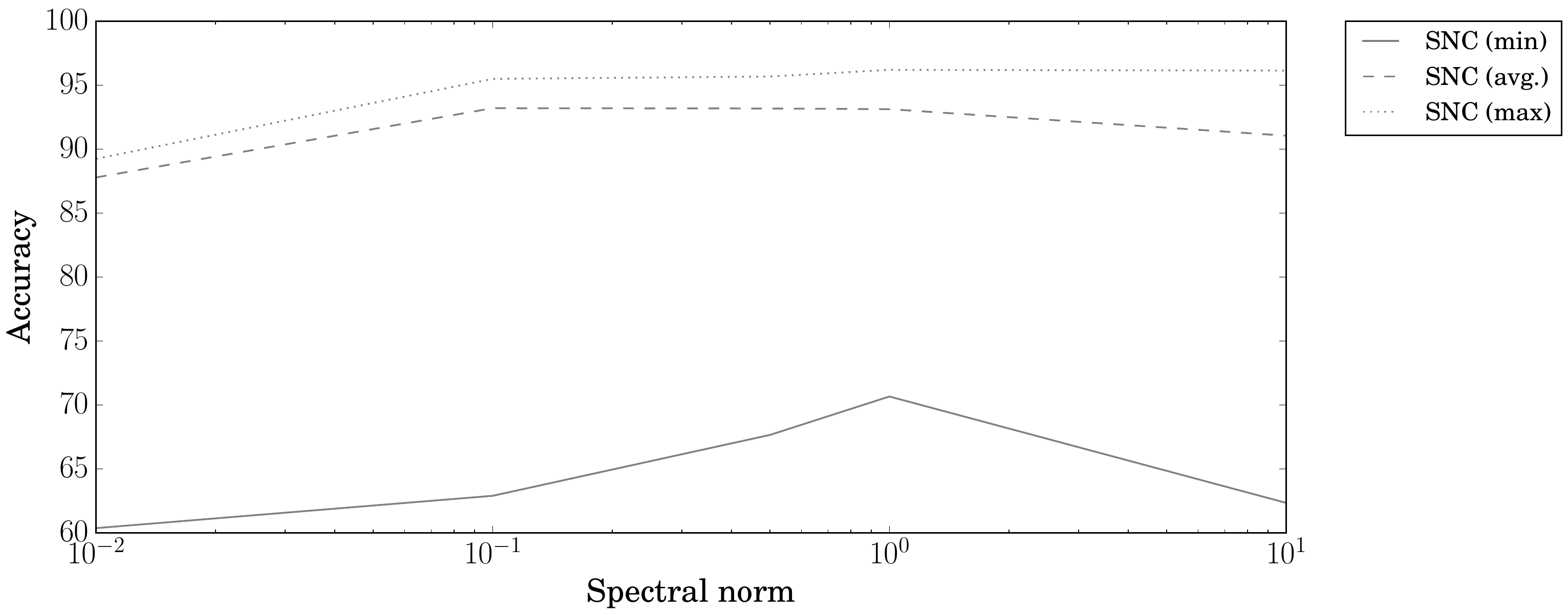}}
\caption{Accuracies of methods on dataset perturbed by SFA vs spectral norm $\sigma$ in SNC}
\label{AccVsSN}
\end{figure}
\begin{figure}[tb]
\centering
\includegraphics[width=0.8\linewidth ,clip]{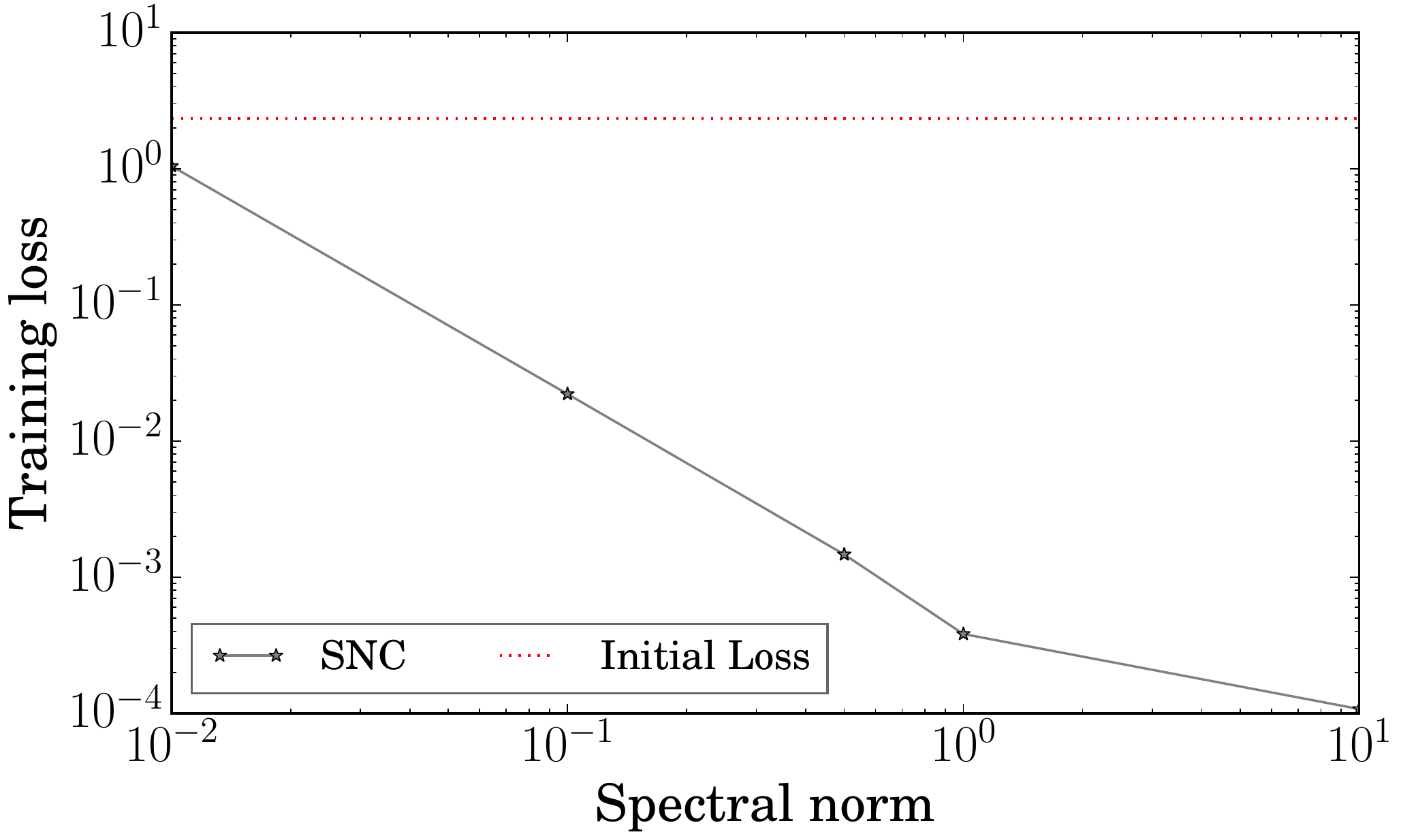}
\caption{Training loss vs. $\sigma$ on CIFAR10}
\label{SVtrainLoss}
\end{figure}
Figure~\ref{accSFA} shows the accuracies on datasets 
perturbed by SFA against hyperparameters $(l,m)$ of SFA.
As shown in \rfig{accSFA}, the models trained with WD and $L_1$ regularization are sensitive to certain frequency noise (e.g., $(17,17)$ in Figs. \ref{accSFA}~(j) and (k)).
Table~\ref{sfatab2} lists the average, minimum, and clean
accuracies on datasets perturbed by SFA.
In this table, Avg. denotes accuracies on data perturbed by SFA averaged over hyperparameters
$(l,m)$, and Min. denotes minimum accuracies on data perturbed by SFA
among $(l, m)$.
CLN denotes accuracies on clean data.
The $\lambda$ and $\sigma$ are selected for each
of Avg., Min., and CLN so that each of them would become the highest.

Figure~\ref{AccVsLam2} shows the accuracies of the methods on FMNIST, CIFAR100, and SVHN perturbed by SFA against regularization weights.
These results are almost the same as those of
MNIST and CIFAR 10.
On all the datasets, Absum improves the avg. and min. accuracies according to the regularization weights,
while the other methods decrease the accuracies according to them.

Figure~\ref{AccVsSN} shows the accuracies of SNC on all datasets perturbed by SFA against
the threshold of the spectral norm $\sigma$.
We can see that, on MNIST and FMNIST, accuracies increase along with $\sigma$.
This is because, when the spectral norm is small, gradient vanishing occurs in the stacked convolutional layers.
On the other hand, on CIFAR10, CIFAR100, and SVHN, where we used ResNet,
minimum accuracy decreases, when the spectral norm becomes larger than a certain point,
while max accuracy increases along with the spectral norm.
Figure~\ref{SVtrainLoss} shows the lowest training loss $\frac{1}{N}f$ in training with SNC on CIFAR10 against $\sigma$.
We can see that SNC with low $\sigma$ prevents the minimizing of the loss function.
\subsection{Robustness against Transferred PGD}
Table~\ref{TransferP} lists robust accuracies against transferred PGD for various $\varepsilon$.
We can see that Absum and SNC can improve robustness against transferred PGD
better than WD and $L_1$.
\begin{table}[tb]
\centering
\caption{Accuracies (\%) on test datasets perturbed by transferred PGD}
\label{TransferP}
\scalebox{0.9}{
\begin{tabular}{lrrr}
\toprule
MNIST\\
\midrule
$\varepsilon$&   0.10 &  0.20 &   0.30 \\
\midrule
Absum $\lambda=10^{-2}$& \bf{95.71}&  \bf{76.34} & \bf{34.39}\\
WD $\lambda=10^{-3}$&  92.67 &  48.94 &  6.63\\
L1 $\lambda=10^{-4}$&  94.63 &  66.48&23.30 \\
SNC $\sigma=1.0$&  94.99 &  71.30 & 25.24 \\
w/o Reg. &94.68&65.87&25.54\\
\bottomrule
\toprule
FMNIST\\
\midrule
%
Absum $\lambda=10^{-3}$&  \bf{58.32} &  \bf{30.08} &  \bf{17.02}\\
WD $\lambda=10^{-2}$&  35.15 &  3.46 &  0.02 \\
L1 $\lambda=10^{-3}$&  48.47 &  18.35 &  7.08\\
SNC $\sigma=10$&  51.45 &  21.31 &  9.76 \\
w/o Reg. &39.92&19.74&12.24\\
\bottomrule
\end{tabular}
}
\centering
\scalebox{0.9}{
\begin{tabular}{lrrrrr}
\toprule
CIFAR10\\
\midrule
$\varepsilon$ &   2/255  &   4/255  &   6/255 &\\
\midrule
Absum $\lambda=10^{-1}$ &  63.33 &  26.29 &  8.58 \\
WD $\lambda=10^{-4}$&   57.45&  18.48 &  5.01 \\
L1 $\lambda=10^{-6}$&57.52 &15.66 &  2.88 \\
SNC $\sigma=0.5$&  \bf{74.14} &  \bf{48.85} & \bf{24.31} \\
w/o Reg. &57.64&15.85&3.15\\
\bottomrule
\toprule
CIFAR100\\
\midrule
Absum $\lambda=10^{-3}$&  41.64 &  18.57 &  8.20 \\
WD $\lambda=10^{-6}$&  42.91 &  17.40 &  6.86 \\
L1 $\lambda=10^{-7}$&  41.35 &  16.68 &  6.60 \\
SNC $\sigma=1.0$&  \bf{50.96} &  \bf{36.57} &  \bf{23.90} \\
w/o Reg. &41.59&16.68&7.28\\
\bottomrule
\toprule
SVHN\\
\midrule
Absum $\lambda=10^{-3}$&  78.37 &  49.11 &  30.27\\
WD $\lambda=10^{-3}$&  76.20 &  40.49 &  20.19 \\
L1 $\lambda=10^{-7}$&  79.99 &  52.79 &  32.95 \\
SNC $\sigma=1.0$&  77.70 &  46.36 &  25.78 \\
w/o Reg. &\bf{80.75}&\bf{54.39}&\bf{34.17}\\
\bottomrule
\end{tabular}
}
\end{table}
\subsection{Accuracy on Data Filtered using High-pass Filter}
Table~\ref{hpf} lists accuracies on test data processed using the high-pass filter.
As shown in this table, the accuracies of Absum tend to be higher than the other methods.
This table and results against High-Frequency attacks imply that Absum does not bias 
towards a specific frequency domain.
On the other hand, the models trained using SNC are not more robust against the high-pass
filter than WD and L1 while they are more robust against High-Frequency attacks.
Therefore, SNC biases CNNs towards low-frequency components of image data.
\begin{table}[t]
\centering
\caption{Robust accuracy against high-pass filter} 
\label{hpf}
\scalebox{0.9}{
\begin{tabular}{ccccccc}
\toprule
&Absum&WD&L1&SNC&w/o Reg.\\
\midrule
MNIST&13.21&40.07&\bf{46.96}&32.96&27.30\\
FMNIST&\bf{29.75}&10.08&10.00&10.05&10.04\\
CIFAR10&\bf{28.75}&19.88&20.35&12.93&28.19\\
CIFAR100&\bf{4.03}&2.17&1.80& 1.1&2.66\\
SVHN&15.53&\bf{19.59}&7.50&6.12&6.38
\\
\bottomrule
\end{tabular}
}
\end{table}
\subsection{Computational Cost}
We evaluated the computation time for convergence on CIFAR10.
Figure~\ref{CompTime2} shows the
training loss $\frac{1}{N}f$ against computation time when
$\lambda=10^{-4}$ and $\sigma=1.0$.
In this figure, Absum converges as fast as L1 regularization.
\begin{figure}[t]
\centering
\includegraphics[width=0.8\linewidth]{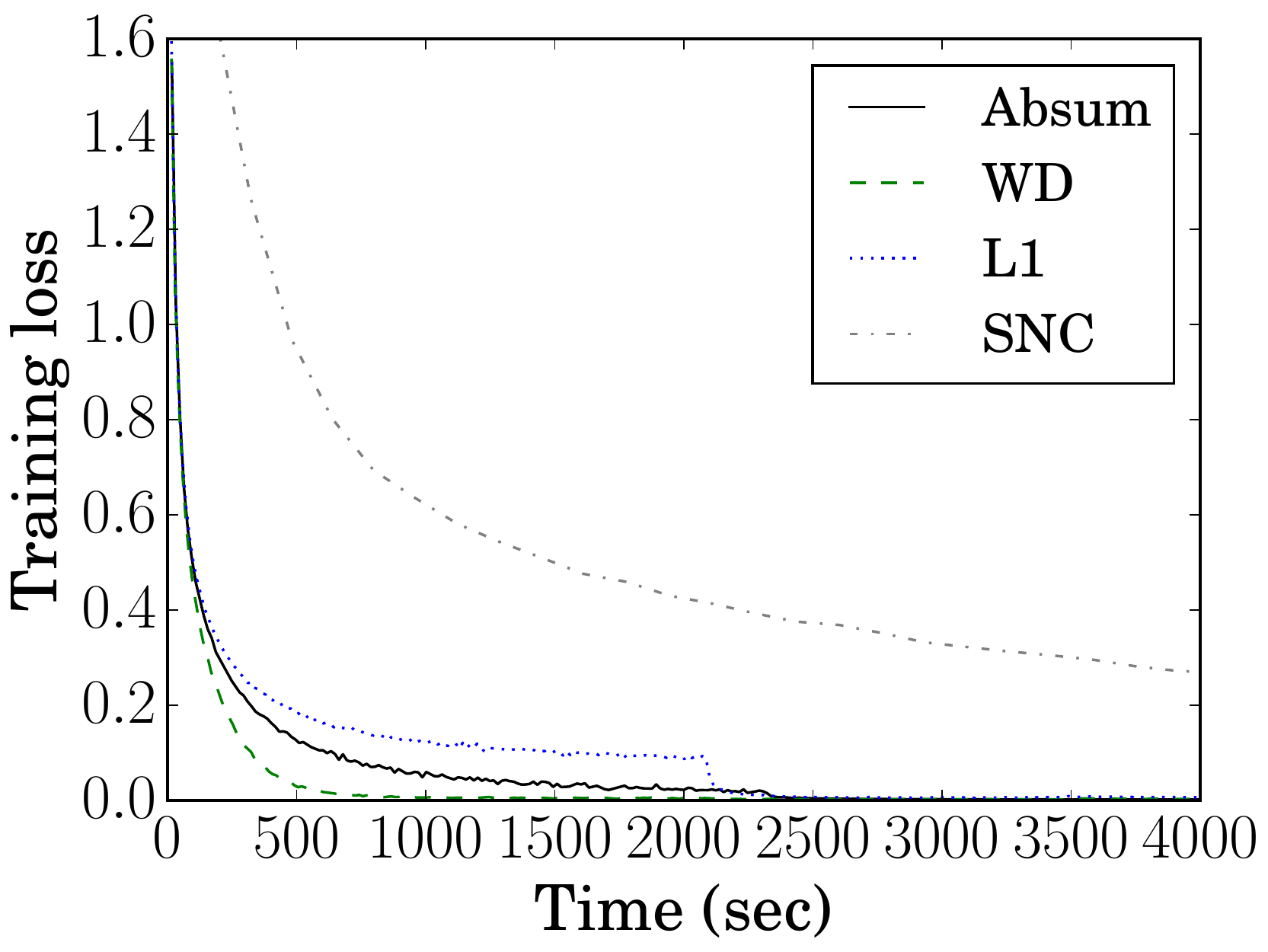}
\caption{Training loss vs. computation time}
\label{CompTime2}
\end{figure}
\subsection{Robustness against PGD}
Table~\ref{data adPGD} 
lists the test accuracies of the models trained by naive training and
adversarial training on the data perturbed by PGD.
We can see that when we train the models without using adversarial training,
Absum does not improve robustness against PGD.
This implies that the structural sensitivity of CNNs does not necessarily cause
all vulnerabilities of CNN-based models.
However, when we use adversarial training, Absum improves robustness against PGD, the highest among regularization methods on almost all datasets.
Therefore, robustness against Fourier basis functions 
can contribute to robustness against other adversarial attacks.
We can see that SNC can slightly improve the robustness against PGD in naive training.
However, when using adversarial training, it does not improve robustness more than Absum.
The best regularization weights for WD and L1 regularization in adversarial training tend to be
lower, and $\sigma$ in adversarial training is higher compared with naive training.
These results indicate that these methods impose too tight of constraints to achieve 
high accuracy and robustness at the same time.
On the other hand, the best regularization weights of Absum in adversarial training
tend to be higher than those in naive training.
Thus, Absum can improve robustness without deteriorating classification performance due to its looseness.
\begin{table*}[t]
\centering
\caption{Accuracies (\%) on test datasets perturbed by PGD. Reg. denotes regularization.}
\label{data adPGD}
\scalebox{0.85}{
\begin{tabular}{lrrrrrrlrrrrrr}
\toprule
MNIST&\multicolumn{6}{c}{Naive training}&&\multicolumn{6}{c}{Adversarial training}\\
\cmidrule(r){1-7}\cmidrule(l){8-14}
$\varepsilon$ &   0.05 &   0.10 &   0.15 &   0.20 &   0.25 &   0.30 &&   0.05 &   0.10 &   0.15 &   0.20 &   0.25 &   0.30 \\
\cmidrule(r){1-7}\cmidrule(l){8-14}
Absum $\lambda=10^{-2}$&  91.70 &  58.97 &  13.30 &   0.90 &  0.09 &  0.00&Absum $\lambda=10^{-3}$& \bf{96.01}&  \bf{94.92} & \bf{93.75}&  \bf{92.73}& \bf{91.59}&  \bf{90.78}\\
WD $\lambda=10^{-5}$ &\bf{93.38}&  73.97 &  34.34 &   6.55 &  0.71 &  0.04&WD $\lambda=10^{-4}$&  92.97 &  91.34 &  89.69 &  88.02 &  87.05 &  85.96\\
L1 $\lambda=10^{-2}$&  92.23 &  76.02 &  43.92 &  13.14 &  2.29 &  0.21 &L1 $\lambda=10^{-4}$&  93.12 &  91.86 &  90.60 &  89.28 &  88.25 &  87.06 \\
SNC $\sigma=0.5$&  92.16 &  \bf{79.30}&  \bf{49.12}&  \bf{16.21}&  \bf{3.64}& \bf{0.63}&SNC $\sigma=10$&  91.92 &  89.43 &  86.77 &  83.89 &  80.24 &  76.92 \\
w/o Reg. &  93.21 &  73.81 &  33.26 &  6.07 &   0.70 &  0.04 &w/o Reg. &  91.57 &  89.85 &  88.43 &  86.87 &  85.76 &  84.86 \\
\bottomrule
\toprule
FMNIST&\multicolumn{6}{c}{Naive training}&&\multicolumn{6}{c}{Adversarial training}\\
\cmidrule(r){1-7}\cmidrule(l){8-14}
%
Absum $\lambda=10^{-7}$&  53.40 &  21.92 &  6.19 &  1.06 &  0.03 &  0.00 &Absum $\lambda=10^{-3}$& \bf{66.94} &  \bf{65.92} &  \bf{65.77} &  \bf{65.52} &  \bf{65.24} &  \bf{64.95} \\
WD $\lambda=10^{-2}$&  \bf{54.45} &  \bf{26.12} &  8.93 &  2.19 &  0.38 &  0.01 &WD $\lambda=10^{-5}$&  65.38 &  63.64 &  62.91 &  62.60 &  62.11 &  61.96 \\
L1 $\lambda=10^{-3}$&  52.80 &  22.43 &  6.80 &  1.59 &  0.18 &  0.00 &L1 $\lambda=10^{-6}$&  66.13 &  64.16 &  62.95 &  62.23 &  61.64 &  61.66 \\
SNC $\sigma=1.0$&  49.08 &  24.03 &  \bf{9.36} &  \bf{3.06} &  \bf{0.68} &  \bf{0.07} &SNC $\sigma=10$&  51.58 &  49.33 &  47.31 &  45.85 &  44.86 &  44.04 \\
w/o Reg. &  52.75 &  20.97 &  5.79 &  0.98 &  0.05 &   0.0 &w/o Reg. &  63.36 &  61.66 &  61.15 &  60.97 &  60.46 &  60.26 \\
\bottomrule
\end{tabular}
}
\centering
\scalebox{0.89}{
\begin{tabular}{lrrrrrlrrrrr}
\toprule
CIFAR10&\multicolumn{5}{c}{Naive training}&&\multicolumn{5}{c}{Adversarial training}\\
\cmidrule(r){1-6}\cmidrule(l){7-12}
$\varepsilon$ &   4/255  &   8/255  &   12/255 &   16/255 &   20/255 &&   4/255  &   8/255  &   12/255 &   16/255 &   20/255 \\
\cmidrule(r){1-6}\cmidrule(l){7-12}
Absum $\lambda=10^{-7}$&  0.26 &  0.00 &  0.00 &  0.00 &  0.00 &Absum $\lambda=10^{-5}$&  69.42 &  49.39 &  \bf{30.22} &  \bf{15.03} &  \bf{6.54} \\
WD $\lambda=10^{-2}$&   8.42 &  \bf{1.87} &  \bf{0.37} &  0.09 &  0.02 &WD $\lambda=10^{-5}$&  \bf{69.48} &  49.38 &  29.37 &  14.45 &  6.06 \\
L1 $\lambda=10^{-1}$&   5.61 &  1.43 &  0.32 &  \bf{0.12} &  \bf{0.08} &L1 $\lambda=10^{-5}$& 68.99 &  \bf{49.45} &  29.51 &  14.68 &  6.31 \\
SNC  $\sigma=0.5$&  \bf{13.39} &  0.41 &  0.01 &  0.00 &  0.00 &SNC $\sigma=10$&  68.47 &  48.74 &  29.07 &  14.32 &  6.04 \\
w/o Reg. &  0.17 &  0.00&  0.00&  0.00&  0.00&w/o Reg. &  68.46 &  48.77 &  29.20 &  14.50 &  6.08 \\
\bottomrule
\toprule
CIFAR100&\multicolumn{5}{c}{Naive training}&&\multicolumn{5}{c}{Adversarial training}\\
\cmidrule(r){1-6}\cmidrule(l){7-12}
Absum $\lambda=10^{-7}$&  2.85 &  0.35 &  0.08 &  0.04 &  0.02 &Absum $\lambda=10^{-4}$&  \bf{42.19} &  \bf{27.25} &  15.89 &  \bf{8.47} &  4.14 \\
WD $\lambda=10^{-4}$&  4.61 &  1.08 &  0.37 &  0.19 &  0.14 &WD $\lambda=10^{-7}$&  41.14 &  27.05 &  \bf{15.90} &  8.26 &  \bf{4.28} \\
L1 $\lambda=10^{-2}$&  4.26 &  1.53 &  0.68 &  \bf{0.37} &  \bf{0.17} &L1 $\lambda=10^{-4}$&  40.75 &  26.14 &  14.45 &  7.61 &  3.67 \\
SNC $\sigma=1.0$&  \bf{7.03} &  \bf{1.88} &  \bf{0.70} &  0.25 &  0.15 &SNC $\sigma=10$&  40.90 &  26.61 &  15.53 &  8.32 &  4.13 \\
w/o Reg. &  2.02 &  0.14 &  0.03 &  0.03 &  0.01 &w/o Reg. &  40.03 &  25.42 &  13.94 &  7.34 &  3.68 \\
\bottomrule
\toprule
SVHN&\multicolumn{5}{c}{Naive training}&&\multicolumn{5}{c}{Adversarial training}\\
\cmidrule(r){1-6}\cmidrule(l){7-12}
Absum $\lambda=10^{-7}$&   9.36 &  0.33 &  0.02 &  0.00 &  0.00 &Absum $\lambda=10^{-5}$&  \bf{77.78} &  \bf{52.74} &  \bf{27.39} &  11.97 &  5.50 \\
WD $\lambda=10^{-4}$&  10.46 &  0.41 &  0.02 &  0.00 &  0.00 &WD $\lambda=10^{-7}$&  76.66 &  50.40 &  25.05 &  10.86 &  5.04 \\
L1 $\lambda=10^{-5}$&  11.78 &  0.565 &  0.03&  0.00 &  0.00&L1 $\lambda=10^{-6}$&  76.50 &  51.49 &  27.10 &  \bf{12.12} &  \bf{5.63} \\
SNC $\sigma=0.5$&  \bf{22.34} &  \bf{2.32} &  \bf{0.21}&  \bf{0.02} &  \bf{0.00} &SNC $\sigma=1.0$&  77.23 &  50.80 &  25.24&  11.04 &  5.03 \\
w/o Reg. &  8.44 &  0.28 &  0.02 &  0.00 &  0.00 &w/o Reg. & N/A& N/A& N/A& N/A& N/A\\
\bottomrule
\end{tabular}
}
\end{table*}
\end{document}